\documentclass{article}

    \PassOptionsToPackage{numbers, compress}{natbib}
 \usepackage[preprint]{neurips_2026}

\usepackage[utf8]{inputenc} 
\usepackage[T1]{fontenc}    
\usepackage{hyperref}       
\usepackage{url}            
\usepackage{booktabs}       
\usepackage{amsfonts}       
\usepackage{nicefrac}       
\usepackage{microtype}      
\usepackage{xcolor}         

\usepackage{amsmath}
    \newcommand\numberthis{\addtocounter{equation}{1}\tag{\theequation}}
\usepackage{amssymb}
\usepackage{amsthm} 
\usepackage{abbreviations}
\usepackage{algorithm}    
\usepackage{algpseudocode}
\usepackage{graphicx}
\usepackage{subcaption}

\usepackage{thmtools} 
\usepackage{multirow} 

\usepackage{hyperref} 
\addtocontents{toc}{\protect\setcounter{tocdepth}{0}} 

\title{
Linear Ensemble Sampling with Smaller Ensembles
}

\author{%
  Taehyun Hwang
  \\
  Seoul National University
  \\
  \texttt{th.hwang@snu.ac.kr}
  \And
  Min-hwan Oh
  \\
  Seoul National University
  \\
  \texttt{minoh@snu.ac.kr}
}

\begin{document}

\maketitle

\begin{abstract}
    Ensemble sampling offers a practical approach to randomized exploration by maintaining a collection of models, but how small an ensemble can be while retaining strong regret guarantees remains unresolved.
    In particular, the existing guarantees use an ensemble size of $\Theta(d\log T)$, leaving a logarithmic gap in the horizon $T$ relative to the intrinsic $\Omega(d)$ ensemble-size barrier.
    We aim to narrow this gap by proposing an ensemble sampling algorithm that refreshes the ensemble only when the regularized Gram matrix changes substantially. This mechanism localizes the perturbation analysis to epochs with controlled Gram-matrix drift and reduces the sufficient ensemble size to $\Theta(d\log d+d\log\log T)$, while preserving the state-of-the-art $\tilde{\mathcal O}(d^{3/2}\sqrt T)$ regret for ensemble sampling with arbitrary bounded arm sets. We further show that, when the arm set is finite of cardinality $K$, the proposed algorithm achieves the sharper regret bound $\tilde{\mathcal{O}}(d\sqrt{T\log K})$. To the best of our knowledge, this is the first ensemble-sampling guarantee that simultaneously recovers both canonical regret scalings known for randomized linear bandit algorithms: the $\tilde{\mathcal{O}}(d^{3/2}\sqrt{T})$ rate for arbitrary bounded arm sets and the $\tilde{\mathcal{O}}(d\sqrt{T\log K})$ rate for finite arm sets. The algorithm also admits an anytime implementation without resetting past data, and experiments show that it remains competitive with baselines while using substantially smaller ensembles.
\end{abstract}

\section{Introduction}
Randomized exploration is a key principle in sequential decision making. In linear contextual
bandits, algorithms such as \textit{Thompson sampling} (TS) and \textit{perturbed-history exploration} (PHE)
translate uncertainty about the unknown parameter into randomized actions, leading to strong regret
guarantees and strong empirical performance~\citep{thompson1933likelihood, agrawal2013thompson,
abeille2017Linear, kveton2019perturbed, kveton2020perturbed, kveton2020randomized}. However,
exact posterior sampling can be difficult to implement beyond conjugate models in TS, and repeatedly
applying perturbations to all past observations at every round can be computationally costly in PHE.

To address these challenges, \emph{ensemble sampling} (ES) offers an attractive alternative~\citep{lu2017ensemble}.
Instead of sampling from an exact posterior distribution or resampling perturbations for all past
observations, ES maintains a finite collection of perturbed models and acts greedily with respect to
one randomly selected ensemble member. This makes ES incremental and compatible with how
ensembles are used in practice, including in exploration for deep reinforcement learning and online recommendation
systems~\citep{osband2016deep, lu2018efficient, osband2018randomized, osband2019deep,
zhu2023deep, zhou2025stochastic}. These practical advantages have motivated a growing theoretical
literature~\citep{lu2018efficient, qin2022analysis, janz2024ensemble, lee2024improved,
sun2025provable, janz2026sharp} studying whether ES can match the statistical guarantees of
canonical randomized exploration methods such as TS. Despite recent progress, two open questions remain.

The first open question concerns the \emph{ensemble complexity}. The sharpest existing analysis of
linear ensemble sampling~\citep{janz2026sharp} shows that ES can achieve the canonical regret
bound $\tilde{\mathcal O}(d^{3/2}\sqrt T)$, but this guarantee requires an ensemble size of order
$\Theta(d\log T)$. Its theoretical ensemble requirement can be larger than the dimension: the
multiplicative $\log T$ factor can still significantly increase the ensemble size because $T$ is typically much
larger than $d$ in many linear bandit instances. On the other hand, an $\Omega(d)$ lower bound on
the ensemble size has also been shown~\citep{janz2026sharp}, implying that one cannot generally
hope for ensembles smaller than the dimension. This leaves a natural question:
\textit{Can ES approach the intrinsic $\Omega(d)$ ensemble-size barrier while preserving sharp regret?}

The second open question concerns the regret landscape of ES for finite arms. For randomized linear bandit
algorithms such as TS and PHE, the regret benchmark is not a single regret rate. For arbitrary bounded arm
sets, including infinite ones, the  \textit{dimension-dominated} rate is
$\tilde{\mathcal O}(d^{3/2}\sqrt T)$. When the arm set is finite with cardinality \(K\) and
\(K \lesssim e^d\), which we refer to as the \textit{finite-arm-dominated} regime, a regret guarantee of
$\tilde{\mathcal O}(d\sqrt{T\log K})$ is available for both TS and PHE.%
\footnote{The relevant comparison of regret bounds depends on the relation between the arm
cardinality \(K\) and the dimension \(d\). We refer to the regime \(K \lesssim e^d\) as the
\emph{finite-arm-dominated regime}, since in this regime the finite-arm rate is sharper than the dimension-dominated rate. We refer to the complementary regime \(K \gtrsim e^d\) as the
\emph{dimension-dominated regime}; infinite arm sets also fall under the dimension-dominated
regimes. Note that this comparison corresponds to whether \(K\) is smaller or larger than \(e^d\),
up to constants and logarithmic factors.}
This \(\log K\)-dependent finite-arm rate is meaningful for two reasons: (i) in finite-arm-dominated regimes (i.e., instances with
\(K \lesssim e^d\)), it is sharper than $\tilde{\mathcal O}(d^{3/2}\sqrt T)$; and (ii) the analysis only
needs to control the perturbed reward estimates on the \(K\) candidate arms, rather than uniformly
over all directions in~\(\mathbb R^d\).

However, existing ES analyses do not fully recover the bounds in these two regimes. They obtain the
dimension-dominated regret bound of \(\tilde{\mathcal O}(d^{3/2}\sqrt T)\), but they do not yield the
finite-arm-dominated regret bound of \(\tilde{\mathcal O}(d\sqrt{T\log K})\).
The main analytical obstacle stems from the persistent nature of ensemble perturbations.
In TS and PHE, the random perturbation is independently resampled at each round, or can be analyzed as conditionally independent randomness given the past. Hence, for each fixed arm, the perturbed score has a
clean conditional tail bound, and a union bound over the \(K\) arms yields the desired
\(\sqrt{\log K}\) dependence. In ES, by contrast, perturbations are persistent: each ensemble member
accumulates perturbations over time, and the actions used to generate those perturbations are
themselves selected using earlier ensemble values. 
Consequently, the perturbations and the subsequent action sequence are adaptively coupled, precluding the standard conditional tail-bound argument.    
    
\begin{table*}[t!]
    \begin{center}
    \caption{
        Comparison of regret bounds and ensemble complexities for linear ensemble sampling.
        Here $d$ is the feature dimension, $K$ is the number of arms, $T$ is the horizon, and $\delta$ is the failure probability in high-probability regret guarantees.
        The dependence on $\delta$ is omitted for results stated in Bayesian regret, where no high-probability failure parameter is used.
        Logarithmic factors other than those displayed are suppressed in the regret bounds.
    }
    \label{tab:my-table}
    \resizebox{\columnwidth}{!}{%
    \begin{tabular}{llll}
        \toprule
        Paper    & Ensemble size & Regret & Setting
        \\
        \midrule
        \citet{lu2017ensemble}  & $\sqrt{K}T$ & N/A  & Finite fixed arms
        \\
        \citet{qin2022analysis}  & $dT$ & $\sqrt{d T}$\textsuperscript{\dag} & Finite fixed arms
        \\
        \citet{janz2024ensemble}  & $d \log T + d \log \frac{1}{\delta}$ & $d^{5/2} \sqrt{T}$ & Arbitrary bounded arms
        \\
        \citet{lee2024improved}  & $K \log T + \log \frac{1}{\delta}$ & $d^{3/2} \sqrt{T}$ & Finite fixed arms
        \\
        \citet{sun2025provable} & $K \log T + \log\frac{1}{\delta}$ & $d^{3/2} \sqrt{T} + d^{9/2}$ & Finite fixed arms, GLM
        \\        
        \citet{janz2026sharp}  & $d \log T + d \log{\frac{1}{\delta}} $ & $d^{3/2} \sqrt{T}$ & Arbitrary bounded arms
        \\
        \textbf{This work} (Theorem~\ref{thm:regret bound for infinite arm set})  & $d \log d + d \log \log T + \log{\frac{1}{\delta}}$ & $d^{3/2} \sqrt{T}$ & Arbitrary bounded arms
        \\
        \textbf{This work} (Theorem~\ref{thm:finite-arm-regret}) & $d \log d + d \log \log T  + \log{\frac{1}{\delta}}$ & $d\sqrt{T \log K}$ & Finite arms
        \\
        \midrule
        \multirow{2}{*}{TS~\citep{agrawal2013thompson, abeille2017Linear} / PHE~\citep{kveton2020perturbed,lee2024improved}\textsuperscript{$\ddagger$}}
        & -- & $d^{3/2}\sqrt{T}$ & Arbitrary bounded arms
        \\
        & -- & $d\sqrt{T\log K}$ & Finite arms
        \\
        \bottomrule         
        \multicolumn{4}{l}{\textsuperscript{$\dag$}\,\footnotesize{The displayed bound is a Bayesian regret bound; all other regret bounds are frequentist.
        }}\\ 
        \multicolumn{4}{l}{%
\parbox{\columnwidth}{%
\footnotesize
\hangindent=0.8em
\hangafter=1
\textsuperscript{$\ddagger$}
For linear PHE, \citet{kveton2020perturbed} established a
$\tilde{\mathcal O}(d\sqrt{T\log K})$ regret bound for finite arms,
while \citet{lee2024improved} later proved a
$\tilde{\mathcal O}(d^{3/2}\sqrt{T})$ regret bound for arbitrary bounded arm sets.
}%
}
    \end{tabular}%
    }
    \end{center}
\end{table*}

In this work, we address both open problems through an epoch-refresh ensemble sampling algorithm.
The proposed algorithm refreshes the ensemble only when the regularized Gram matrix changes by a prescribed multiplicative factor. This localizes the perturbation analysis to epochs over which the regularized Gram matrices remain uniformly comparable. As a result, in the dimension-dominated regime, which
includes infinite arm sets, we recover the \(\tilde{\mathcal O}(d^{3/2}\sqrt T)\) regret bound with only
\(\Theta(d\log d+d\log\log T)\) ensemble members, narrowing the gap between the previously known
sufficient ensemble size \(\Theta(d\log T)\)~\citep{janz2024ensemble, janz2026sharp} and the known
\(\Omega(d)\) lower bound whenever \(d < T/\log T\), a common parameter regime in linear bandit
problems. 
Moreover, in the finite-arm-dominated regime, the same localized analysis
enables an arm-wise concentration argument and yields the sharper regret bound
\(\widetilde{\mathcal O}(d\sqrt{T\log K})\) with the same reduced ensemble complexity.
Our main contributions are summarized as follows:
\begin{itemize}
    \item 
    We propose an ES algorithm for stochastic linear contextual bandits. For arbitrary bounded arm sets, including infinite ones, 
    we prove that the proposed algorithm attains the sharpest known regret order for ensemble sampling~\citep{janz2026sharp, lee2024improved},
    \(\tilde{\mathcal O}(d^{3/2}\sqrt{T})\), where \(d\) is the feature dimension and \(T\) is the horizon. 
    While matching the previously sharpest regret bound, we prove the ensemble complexity of
    \(\Theta(d \log d + d \log \log T)\). 
    This narrows the gap between the previously known sufficient ensemble size,
    \(\Theta(d \log T)\)~\citep{janz2024ensemble, janz2026sharp}, and the lower bound of \(\Omega(d)\) on the ensemble complexity whenever \(d < T/\log T\).

    \item 
    We further show that, when the arm sets are chosen by an oblivious adversary and each has cardinality \(K\), the proposed algorithm achieves a regret bound of order \(\tilde{\mathcal O}(d\sqrt{T\log K})\) with the same reduced ensemble complexity
    \(\Theta(d \log d + d \log \log T)\).
    To the best of our knowledge, this is the first \(\log K\)-dependent regret bound for ES that scales linearly in \(d\).
    For instances with \(K \lesssim e^d\), i.e., the finite-arm-dominated regime,
    \(\tilde{\mathcal O}(d\sqrt{T\log K})\) can be significantly sharper than
    \(\tilde{\mathcal O}(d^{3/2}\sqrt{T})\), depending on the number of arms.
    
    \item
    Importantly, the same ES algorithm simultaneously recovers the two canonical regret bounds known for randomized linear bandit algorithms such as TS and PHE:
    the dimension-dominated bound \(\tilde{\mathcal O}(d^{3/2}\sqrt{T})\) and the finite-arm-dominated bound
    \(\tilde{\mathcal O}(d\sqrt{T\log K})\).
  
    \item 
    We empirically show that the proposed algorithm achieves regret comparable to existing randomized baselines while using substantially fewer ensemble members. This reduction in ensemble size also leads to lower total runtime, with the improvement becoming more pronounced as the feature dimension increases. In addition, to the best of our knowledge, our results provide the first empirical evidence that ensemble methods can offer a computationally efficient alternative to TS and PHE.

    \item 
    Our algorithm admits an anytime implementation without resetting past data. When the horizon \(T\) is unknown, prior ensemble sampling methods typically rely on the doubling trick, increasing the ensemble size at each doubling epoch and restarting the algorithm. In contrast, our epoch-refresh design allows the ensemble size to be increased online only as needed, while retaining all previously collected observations. This makes the procedure substantially more data-efficient in practical scenarios.    
\end{itemize}

\subsection{Related Work}
\paragraph{Randomized exploration in linear bandits.}
Linear bandits have been studied extensively through both optimism-based methods, such as \texttt{OFUL}~\citep{abbasi2011improved} and \texttt{LinUCB}~\citep{chu2011contextual}, and randomized exploration methods. Among the latter, TS~\citep{thompson1933likelihood, agrawal2013thompson, abeille2017Linear, hamidi2020frequentist} is a classical posterior-sampling method that achieves the canonical regret rate
$\tilde{\Ocal}(d^{3/2}\sqrt{T})$ for general arm sets, and sharper $\tilde{\Ocal}(d\sqrt{T\log K})$ bounds are available when $|\Xcal|=K$. However, TS requires sampling from, or accurately approximating, the
posterior distribution, which can be computationally demanding beyond simple conjugate models.
Another representative randomized approach is PHE~\citep{kveton2019perturbed, kveton2020perturbed, kveton2020randomized}, which avoids explicit posterior sampling by injecting
random perturbations into the observed history. PHE can also attain the finite-arm rate $\tilde{\Ocal}(d\sqrt{T\log K})$. These two regret scalings serve as the benchmark for randomized linear bandit algorithms and motivate our goal
of recovering them for ensemble sampling.

\paragraph{Ensemble sampling.}
ES~\citep{lu2017ensemble} was introduced as a computationally attractive alternative to posterior sampling. Instead of maintaining an exact posterior distribution, ES keeps a finite collection of perturbed models and acts greedily with respect to one randomly selected ensemble member. This structure is appealing in settings where incremental model updates are cheap, and it has motivated algorithms in deep reinforcement learning, online recommendation, and other sequential decision-making problems. However, despite its empirical success, ES has been difficult to analyze theoretically because the ensemble models are trained on data collected adaptively by the ensemble itself. In contrast to TS or PHE, the randomness used for exploration is not freshly sampled independently of the past at every round; rather, the perturbations persist and interact with the adaptively chosen arms.

Early theoretical guarantees for ES required ensemble sizes that were too large for the method to be practically meaningful.~\citet{qin2022analysis} obtained Bayesian regret guarantees under ensemble sizes that scale at least linearly with the horizon.~\citet{janz2024ensemble} gave the first rigorous frequentist analysis of ES in stochastic linear bandits with an ensemble size logarithmic in the horizon and linear in the dimension.~\citet{janz2024ensemble} showed that a symmetrized version of linear ES with $\Theta(d\log T)$ ensemble members obtains regret $\tilde{\Ocal}(d^{5/2}\sqrt{T})$ and, importantly, also applies to infinite arm sets. This result established that sublinear-in-\(T\) ensembles can suffice, but the regret remained worse than the canonical TS rate by a factor of roughly \(d\).
At the same time,~\citet{lee2024improved} improved the regret analysis of linear ES and showed that ES can achieve the canonical regret rate $\tilde{\Ocal}(d^{3/2}\sqrt{T})$ with ensemble size logarithmic in \(T\). 
However, their ES guarantee requires the ensemble size to scale with the number of arms, namely \(\Omega(K\log T)\). Thus, while their result is sharp for finite and moderately sized arm sets, it becomes unsuitable for large finite arm sets and does not directly address general infinite arm sets.
~\citet{janz2026sharp} recently provided a sharp analysis of linear ES for general arm sets. They showed that Gaussian linear ES with \(\Theta(d\log T)\) ensemble members achieves the regret rate $\tilde{\Ocal}(d^{3/2}\sqrt{T})$, matching the canonical rate of TS for general linear bandits. Their analysis reduces the main difficulty to a time-uniform self-normalized exceedance problem and solves it through a Brownian-motion embedding argument.
They also established a lower-bound result showing that ensembles of size at most \(d/2\) can suffer linear regret on some instances, indicating that the ensemble size cannot be pushed far below the dimension in general.

\section{Preliminaries}
\paragraph{Notation.}
For \(n\in\NN\), let \([n]:=\{1,\ldots,n\}\).
For \(\xb,\yb\in\RR^d\), let \(\ips{\xb}{\yb}:=\xb^\top\yb\).
For a positive semidefinite matrix \(\Ab\in\RR^{d\times d}\), define
\(\|\xb\|_{\Ab}:=\sqrt{\xb^\top\Ab\xb}\), and let \(\|\Ab\|_{\op}\) denote its operator norm.
We use \(\ind\{\cdot\}\) for the indicator function.
The notation \(\Ocal(\cdot)\) denotes the standard big-\(\Ocal\) notation with respect to the problem parameters, such as \(d\) and \(T\), and
\(\tilde{\Ocal}(\cdot)\) hides logarithmic factors in these parameters.
We denote the unit Euclidean ball in \(\RR^d\) by $\Bcal_2^d := \{\xb\in\RR^d:\|\xb\|_2\le 1\}$.

\paragraph{Problem setup.}
We consider a stochastic linear contextual bandit problem.
At each round \(t\in[T]\), the agent observes a compact arm set \(\Xcal_t\subseteq\RR^d\), selects an arm \(\xb_t\in\Xcal_t\), and observes a real-valued reward \(y_t\in\RR\). The reward is generated according to $y_t = \ips{\xb_t}{\thetab^*} + \eta_t$,
where \(\thetab^*\in\RR^d\) is an unknown parameter and \(\eta_t\) is noise.
Let $\Hcal_t := \sigma(\Xcal_1,\xb_1,y_1,\ldots,\Xcal_{t-1},\xb_{t-1},y_{t-1}, \Xcal_t,\xb_t)$ be the \(\sigma\)-algebra generated by the past observations, the current arm set, and the selected arm at round \(t\). We assume that \(\eta_t\) is
conditionally zero-mean and \(1\)-sub-Gaussian with respect to \(\Hcal_t\), that is, for all \(s\in\RR\), $\EE[\eta_t\mid \Hcal_t]=0$, $\EE[\exp(s\eta_t)\mid \Hcal_t]\le \exp(s^2/2)$.
The agent's objective is to maximize the cumulative reward over \(T\) rounds,
or equivalently, to minimize the cumulative regret
\begin{equation*}
    \regret_T
    :=
    \sum_{t=1}^T
    \left(
    \ips{\xb_t^*}{\thetab^*}
    -
    \ips{\xb_t}{\thetab^*}
    \right) \, ,
\end{equation*}
where $\xb_t^* \in \argmax_{\xb\in\Xcal_t} \ips{\xb}{\thetab^*}$ is an optimal arm at round \(t\).

\section{Main Results}
\subsection{Algorithm}
\begin{algorithm}[h!]
    \caption{Linear Ensemble Sampling with Epoch Refresh ($\algname$)}
    \label{alg:main algorithm}
    \begin{algorithmic}[1]
    \State \textbf{Inputs:} horizon $T \in \NN$, ensemble size $m \in \NN$, regularization parameter $\lambda > 0$, inflation parameter $\gamma>0$, confidence radii $\{\beta_t\}_{t\ge0}$, and epoch refresh parameter $\alpha >0$
    \State Initialize $\Vb_0 = \lambda \Ib_d$, $\hat{\thetab}_0 = \zero_d$, and set the current epoch start time $\tau = 0$
    \State Sample $\gb^j\sim \Ncal (\zero_d, \Ib_d)$ and set $\sbb_0^j = \Vb_0^{1/2} \gb^j, \thetab^j_0 = \Vb_0^{-1} \sbb^j_0$ for each $j \in [m]$
    \For{$t=1,\dots,T$}
        \If{$\Vb_{t-1} \npreceq (1 + \alpha)\Vb_\tau$}
            \Comment{\textit{Refresh the ensemble perturbations}}
            \State Set $\tau \gets t-1$
            \State For each $j \in [m]$, sample $\gb^j \sim \Ncal(\zero_d,\Ib_d)$ and set
            \begin{equation*}
                \sbb_{t-1}^j = \Vb_\tau^{1/2} \gb^j \, ,
                \quad                
                \thetab^j_{t-1} = \hat{\thetab}_{t-1} + \gamma \beta_{t-1} \Vb_{t-1}^{-1} \sbb^j_{t-1}
            \end{equation*}
        \EndIf
        \State Observe $\Xcal_t$, sample $J_t \sim \unif([m])$, choose $\xb_t \in \argmax_{\xb \in \Xcal_t} \ips{\xb}{\thetab_{t-1}^{J_t}}$ and observe $y_t$
        \State Update $\Vb_t = \Vb_{t-1} + \xb_t \xb_t^\top$ and $\hat{\thetab}_t = \Vb_t^{-1} \sum_{s=1}^t \xb_s y_s$
        \State For each $j \in [m]$,
            sample $\xi_t^j \sim \Ncal(0,1)$ and update 
            \begin{equation*}
                \sbb_t^j = \sbb_{t-1}^j + \xb_t \xi_t^j \, ,
                \quad
                \thetab_t^j = \hat{\thetab}_{t} + \gamma \beta_{t} \Vb_{t}^{-1} \sbb_{t}^{j}
            \end{equation*}
    \EndFor
    \end{algorithmic}
\end{algorithm}
We now describe our ensemble sampling algorithm. The algorithm maintains $m$ perturbed least-squares models.
At each round, it samples one ensemble member uniformly at random and acts greedily with respect to the corresponding perturbed parameter. 
The key difference from prior ensemble sampling methods~\citep{janz2024ensemble, lee2024improved, sun2025provable,
janz2026sharp} is that the ensemble is refreshed only when the regularized Gram matrix changes substantially.

Let $\alpha>0$ be the epoch-refresh parameter. Each epoch starts at some round $\tau$. During the epoch, the algorithm monitors the regularized Gram matrix and starts a new epoch whenever
\[
    \Vb_{t-1} \npreceq (1+\alpha)\Vb_\tau .
\]
Thus, throughout an epoch, the regularized Gram matrix remains comparable to its value at the epoch start. This comparability is the key property used in our analysis.

At the beginning of an epoch, each ensemble perturbation is initialized according to the current Gram matrix. Specifically, for each \(j\in[m]\), the
algorithm samples \(\gb^j\sim\Ncal(\zero_d,\Ib_d)\) and sets
\(\sbb_\tau^j=\Vb_\tau^{1/2}\gb^j\) and
\(\thetab_\tau^j=\hat{\thetab}_\tau+\gamma\beta_\tau\Vb_\tau^{-1}\sbb_\tau^j\).
At round \(t\), the algorithm observes \(\Xcal_t\), samples an ensemble index
\(J_t\sim\unif([m])\), and chooses
\[
    \xb_t
    \in
    \argmax_{\xb\in\Xcal_t}
    \ips{\xb}{\thetab_{t-1}^{J_t}} .
\]
After observing reward \(y_t\), the algorithm updates
\(\Vb_t=\Vb_{t-1}+\xb_t\xb_t^\top\) and
\(\hat{\thetab}_t=\Vb_t^{-1}\sum_{s=1}^t\xb_sy_s\).
Then, for each \(j\in[m]\), it samples
\(\xi_t^j\sim\Ncal(0,1)\) and updates
\(\sbb_t^j=\sbb_{t-1}^j+\xb_t\xi_t^j\) and
\(\thetab_t^j=\hat{\thetab}_t+\gamma\beta_t\Vb_t^{-1}\sbb_t^j\).

The epoch-refresh mechanism has two roles. First, it repeatedly adjusts the perturbation distribution to match the current Gram matrix. 
Second, it restricts the uniform control of the ensemble to a sequence of epochs within which the regularized Gram matrices remain comparable.
This localization is what allows us to replace the previous $\Theta(d\log (T/\delta))$ ensemble-size requirement for ensemble sampling~\citep{janz2024ensemble, janz2026sharp} by $\Theta(d\log d+d\log\log T + \log(1/\delta))$.

\subsection{Regret Bound for Dimension-Dominated Regime}
We present our main theoretical guarantees. 
The first result considers arbitrary bounded arm sets, including infinite ones. It shows that the proposed algorithm matches the best known regret rate for ensemble sampling while using a smaller ensemble. 

\begin{assumption}[Boundedness] \label{assm:boundedness}
    For every $t \in [T]$, the arm set $\Xcal_t$ is closed and satisfies $\| \xb\|_2 \le 1$ for all $\xb \in \Xcal_t$. 
    We also assume that there exists $S > 0$ such that $\| \thetab^* \|_2 \le S$.
\end{assumption}

\begin{restatable}[Regret bound for $\algname$]{theorem}{RegretBoundForInfiniteArmSet} \label{thm:regret bound for infinite arm set}
    Suppose Assumption~\ref{assm:boundedness} holds and \(T \ge \max\{2,d\}\).
    Fix any \(\delta \in (0,1/4)\), and choose the parameters in Algorithm~\ref{alg:main algorithm} as \(\lambda\ge80\), \(\gamma=40\), \(\alpha=1\),
    \begin{align*}
        & \beta_t
        =
        \sqrt{d\log\!\left(1+\frac{t}{d\lambda}\right)+2\log\frac{1}{\delta}}
        +
        \sqrt{\lambda}\,S,    \quad \text{and} \quad
        m
        \ge
        C\left(
        d\log\!\left(d+\log\frac{\bar R_\alpha}{\delta}\right)
        +
        \log\frac{\bar R_\alpha}{\delta}
        \right) \, ,
    \end{align*}
    where \(C>0\) is an absolute constant and \(\bar R_\alpha:= 1 +\left\lfloor d\log_{1+\alpha}\!\left(1+\frac{T}{\lambda d}\right)\right\rfloor\). Then, with probability at least \(1-4\delta\), Algorithm~\ref{alg:main algorithm} incurs cumulative regret
    $\regret_T=\tilde{\Ocal}\!\left(d^{3/2}\sqrt{T}\right)$.
\end{restatable}

\begin{remark}[Ensemble size] \label{rem:ensemble size}
    When $\alpha = 1$, it follows that $\log \bar{R}_1 = \Ocal(\log d + \log \log T)$.
    Hence, the sufficient condition on the ensemble size in Theorem~\ref{thm:regret bound for infinite arm set} gives
    \begin{equation*}
        m
        =
        \Theta \left(
        d\log d + d\log\log T + \log\frac{1}{\delta}
        \right)    
    \end{equation*}
    up to universal constants and lower-order logarithmic terms in $d,\lambda,\delta$.
\end{remark}

\paragraph{Discussion of Theorem~\ref{thm:regret bound for infinite arm set}.}
Theorem~\ref{thm:regret bound for infinite arm set} shows that Algorithm~\ref{alg:main algorithm} achieves a regret bound of order $\widetilde{\Ocal}\!\left(d^{3/2}\sqrt{T}\right)$ for dimension-dominated regime, with an ensemble size smaller than that required by prior ES analyses~\citep{janz2024ensemble, janz2026sharp}.
In particular, when $\alpha=1$, Remark~\ref{rem:ensemble size} implies that it suffices to take $m = \Ocal\!\left(d\log d+d\log\log T+\log \frac{1}{\delta}\right)$, up to lower-order logarithmic terms. Thus, our ES algorithm matches the sharpest existing analysis of linear ensemble sampling~\citep{janz2026sharp}, while reducing the sufficient ensemble size from order $d\log (T/\delta)$ to roughly $d\log d+d\log\log T + \log(1/\delta)$.

This improvement is particularly meaningful in view of the known ensemble-size barrier for ES. \citet{janz2026sharp} showed that, although an ensemble size of order $d\log T$ is sufficient to obtain
$\widetilde{\Ocal}(d^{3/2}\sqrt{T})$ regret, there exist instances on which ES with only $m=\Theta(d)$ ensemble members suffers linear regret. 
Our result therefore moves the sufficient ensemble size significantly closer to this lower-bound regime, while preserving the same regret rate as the best known ES bound for general arm sets.

\begin{remark}[Anytime implementation without reset] \label{rem:no-reset-doubling}
    Our theoretical choice of the ensemble size depends on the horizon \(T\).
    When \(T\) is unknown, we can use a doubling schedule~\citep{besson2018doubling} and increase the ensemble
    size whenever the current horizon estimate doubles. Unlike prior ES methods~\citep{janz2024ensemble, sun2025provable, janz2026sharp}, this does not require restarting the algorithm or
    discarding past observations. Because our algorithm already refreshes ensemble
    perturbations according to the current Gram matrix, we can simply start a new
    epoch and initialize the enlarged ensemble using the current Gram matrix. The
    regularized Gram matrix and the ridge estimator are kept unchanged, so all past
    data continue to be used. Thus, our epoch-refresh mechanism yields an anytime
    implementation without data reset, preserving the regret guarantees up to
    logarithmic factors and providing a more data-efficient procedure in practice.
\end{remark}

\subsection{Regret Analysis}

In this section, we present the main ingredients of the regret analysis for Theorem~\ref{thm:regret bound for infinite arm set}. The proof follows the standard randomized-optimism route for linear bandits~\citep{abeille2017Linear, kveton2020perturbed, janz2024ensemble, lee2024improved, janz2026sharp}, but the key technical difficulty is to show that a sufficiently large fraction of ensemble members is optimistic uniformly over time and over all directions~\citep{janz2026sharp}.
Our epoch-refresh construction makes this possible with a smaller ensemble size by localizing the perturbation analysis to epochs.

\subsubsection{Deterministic Upper Bound on the Number of Epochs}
The number of epochs generated by Algorithm~\ref{alg:main algorithm} is random, since epoch refreshes depend on the realized Gram matrices, which in turn are determined by the algorithm's randomization and the observed data.
We first derive a deterministic upper bound on the total number of epochs.
Let $R_\alpha$ denote the total number of epochs generated by Algorithm~\ref{alg:main algorithm}. 
Since each completed epoch is triggered by a multiplicative growth of the regularized Gram matrix, a determinant-growth argument gives
\begin{equation} \label{eq:epoch-num-upper-bound}
    R_{\alpha}
    \le
    1+
    \frac{
        d\log\left(1+\frac{T}{\lambda d}\right)
    }{
        \log(1+\alpha)
    }
    \le 
    1+
    \left\lfloor
    \frac{
        d\log\left(1+\frac{T}{\lambda d}\right)
    }{
        \log(1+\alpha)
    }
    \right\rfloor =: \bar{R}_\alpha \, .
\end{equation}
This bound is deterministic and holds for every possible realization of the algorithm. The proof of Eq.~\eqref{eq:epoch-num-upper-bound} is provided in Lemma~\ref{lem:num-epochs}.

\subsubsection{Self-Normalized Exceedance Frequencies Inside One Epoch}
The central object in the analysis is the self-normalized exceedance frequency
\begin{equation*}
    E_{t,m}(\ub,c)
    :=
    \frac{1}{m}
    \sum_{j=1}^m
    \ind\left\{
        \frac{\ips{\ub}{\sbb_{t-1}^j}}
             {\|\ub\|_{\Vb_{t-1}}}
        \ge c
    \right\},
    \qquad
    \ub\in \mathbb S^{d-1} \, .
\end{equation*}
This quantity measures the fraction of ensemble perturbations that are large enough in direction $\ub$ at round $t$.
This exceedance frequency is also the central quantity in the analysis of~\citet{janz2026sharp}. 
Their proof controls a fixed collection of $m$ persistent ensemble members over the entire horizon by reducing the problem to a time-uniform exceedance event for the corresponding $m$ independent Brownian motions.
At a high level, for each fixed direction $\ub$, the projected perturbation process $\Mb_t := (\ips{\sbb^j_{t-1}}{\ub})_{j=1}^m$ can be represented through Brownian motions whose time parameter is the self-normalized clock $\|\ub\|_{\Vb_{t-1}}^2$. Since their ensemble is not refreshed, this clock must be controlled over the full range induced by the Gram-matrix growth from the beginning to the end of the horizon. After further making this control uniform over directions, their analysis requires an ensemble size of order $d\log T$.

The main advantage of epoch refreshes is that the regularized Gram matrix varies only within a controlled multiplicative range during each epoch.
Indeed, if round $t$ belongs to the $r$-th epoch starting at $\tau_r$, then the refresh rule guarantees 
$    \Vb_{\tau_r}
    \preceq
    \Vb_{t-1}
    \preceq
    (1+\alpha)\Vb_{\tau_r}$.
Thus, for any fixed direction $\ub$, the Brownian clock associated with the projected perturbation process varies only over the local interval
$\left[ \|\ub\|_{\Vb_{\tau_r}}^2, (1+\alpha)\|\ub\|_{\Vb_{\tau_r}}^2 \right]$.
Consequently, when applying the Brownian exceedance theorem (Theorem 5.4 in~\citep{janz2026sharp}), the relevant clock ratio is controlled by $1+\alpha$, rather than by the global Gram-matrix growth over the entire horizon.

\begin{restatable}[Fixed-direction exceedance inside one epoch]{lemma}{FixedDriectionEpoch}\label{lem:fixed-direction-epoch}
    Fix an epoch $r$ and a direction $\ub \in \usd$.
    For $\alpha > 0$, let $\gamma := 40, \, \kappa_\alpha := \left\lceil 250\log (1 + \alpha)\right\rceil$.
    Then for every $\delta'\in(0,1)$, if $m \ge 40 \log \frac{\kappa_\alpha}{\delta'}$,
    then with probability at least $1-\delta'$, \(\forall t\in\{\tau_r+1,\dots,\tau_{r+1}\}\), we have
    $E_{t,m}(\ub,2/\gamma) \ge \frac{1}{10}$.
\end{restatable}
Note that the epoch-refresh construction allows us to represent the projected perturbations by Brownian motions whose clocks vary only over the local multiplicative range of that epoch. 
This makes it possible to invoke the Brownian exceedance theorem with a uniformly bounded clock ratio, rather than the horizon-dependent ratio arising in the global persistent-ensemble analysis~\citep{janz2026sharp}.

\subsubsection{Self-Normalized Exceedance Frequencies Across Epochs}
The fixed-direction bound is not sufficient for bounded arm sets that may be infinite, because the optimistic direction can depend on both the round and the unknown optimal arm. We therefore extend the fixed-direction guarantee uniformly over all directions on the unit sphere.
The first ingredient is a local Lipschitz property of the normalized map induced by the Gram-matrix change within an epoch. 
For an epoch $r$ and a round $t$ in that epoch, define
\begin{equation*}
    \Bb_{r,t}
    :=
    \Vb_{t-1}^{1/2}\Vb_{\tau_r}^{-1/2},
    \qquad
    \phib_{r,t}(\wb)
    :=
    \frac{\Bb_{r,t}\wb}{\|\Bb_{r,t}\wb\|_2},
    \quad
    \wb\in\mathbb S^{d-1} \, .
\end{equation*}
Since $\Vb_{t-1}\preceq (1+\alpha)\Vb_{\tau_r}$ inside the epoch, this map is Lipschitz (Lemma~\ref{lem:local-lipschitz}):
\begin{equation*}
    \|\phib_{r,t}(\wb)-\phib_{r,t}(\wb')\|_2
    \le
    2\sqrt{1+\alpha}\,\|\wb-\wb'\|_2 \, .
\end{equation*}

This allows us to transfer the exceedance guarantee from a nearby net point to an arbitrary direction.
This is a key point where our epoch-refresh analysis differs from prior persistent-ensemble analyses~\citep{janz2024ensemble, janz2026sharp}. 
In the analysis of~\citet{janz2026sharp}, the corresponding normalized map must be controlled over the full sequence of Gram-matrix geometries along the entire horizon. Since the Gram matrix may grow by a factor polynomial in $T$, the Lipschitz constant can scale as large as order $\sqrt{T}$, which forces a much finer sphere net and contributes
a $d\log T$ term to the required ensemble size. In contrast, our refresh rule ensures that, within each epoch, the Gram matrix changes by at most a factor $1+\alpha$. Consequently, the Lipschitz constant of $\phib_{r,t}$ is bounded by
$2\sqrt{1+\alpha}$, which is independent of $T$ for constant $\alpha$. This local Lipschitz control is one of the main reasons why the uniform exceedance argument can be carried out with a smaller ensemble.

The second ingredient is a uniform upper bound on the normalized ensemble perturbations. Define
\begin{equation*}
    \Gamma_\delta
    :=
    \sqrt{d}
    +
    \sqrt{2\log\!\left(\frac{4m\bar R_\alpha}{\delta}\right)}
    +
    \sqrt{
        d\log(1+\alpha)
        +
        2\log\!\left(\frac{4m\bar R_\alpha}{\delta}\right)
    } \, .    
\end{equation*}
Then, with high probability, we have
\begin{equation*}
    \forall r\le R_\alpha,\ \forall j\in[m],\
    \forall t\in\{\tau_r+1,\ldots,\tau_{r+1}\}:
    \qquad
    \left\|
        \Vb_{t-1}^{-1/2}\sbb_{t-1}^j
    \right\|_2
    \le
    \Gamma_\delta \, .
\end{equation*}

This bound controls the loss incurred when passing from a net point to a nearby direction. Combining the local Lipschitz property with this perturbation bound, we choose a sphere net of radius proportional to $(\sqrt{1+\alpha}\gamma\Gamma_\delta)^{-1}$.
Putting these ingredients together, and taking a union bound over all epochs and all points in the sphere net, yields the following uniform exceedance guarantee:
\begin{restatable}[Uniform exceedance count across epochs]{lemma}{UniformExceedanceCountAcrossEpochs} \label{lem:uniform-exceedance-across-epochs}
    Fix $\gamma := 40$, $\kappa_\alpha := \lceil 250\log(1+\alpha)\rceil$, $N_\delta := \left( 6\sqrt{1+\alpha}\,\gamma \Gamma_\delta \right)^d$.
    If
    \begin{equation}     \label{eq:ER-condition}
        m \ge 40 \log\!\left(\frac{2 \kappa_\alpha \bar R_\alpha N_\delta}{\delta}\right),
    \end{equation}
    then with probability at least \(1-\delta\),
    $\displaystyle \min_{t\in[T]}\ \inf_{\ub\in \usd} E_{t,m}(\ub,1/\gamma)\ge \frac{1}{10}$.    
\end{restatable}

Furthermore, the sufficient ensemble-size condition in Eq.~\eqref{eq:ER-condition} can be simplified as follows:
\begin{restatable}[Sufficient ensemble size]{corollary}{SufficientEnsembleSize}\label{cor:explicit-ensemble-size}
    Suppose \(0<\alpha\le 1\). There exists a universal constant \(C>0\) such that the sufficient condition on the ensemble size in Eq.~\eqref{eq:ER-condition} is satisfied by choosing
    \begin{equation*}
        m
        =
        \left\lceil
        C\left[
        d\log\!\left(d+\log\frac{\bar R_\alpha}{\delta}\right)
        +
        \log\frac{\bar R_\alpha}{\delta}
        \right]
        \right\rceil \, .   
    \end{equation*}
\end{restatable}

Combining these ingredients with the standard regret-analysis framework for optimistic randomized algorithms for linear bandits yields the desired regret bound. 
The proof is deferred to Appendix~\ref{appx:regret bound for infinite arm set}.

\subsection{Regret Bound for Finite-Arm-Dominated Regime}
Our second result considers the finite-arm-dominated regime.
In this case, the regret analysis only needs arm-wise concentration over the finite set of arms presented at each round, rather than uniform concentration over all directions.
This makes it possible to replace the dimension-dependent uniform control by an arm-wise concentration argument, leading to the sharper finite-arm regret scaling.

For randomized linear bandit algorithms such as TS~\citep{agrawal2013thompson, abeille2017Linear} and PHE~\citep{kveton2020perturbed}, such a \(\sqrt{\log K}\)-type bound is relatively natural.
In these algorithms, the randomization used for arm selection is either resampled at each round or conditionally independent of the past.
Hence, conditional on the history and the current finite arm set, the perturbed reward estimate of each presented arm admits a sub-Gaussian tail bound. A union bound over the \(K\) arms then yields an arm-wise confidence scale of order \(\sqrt{\log K}\).

However, this argument does not directly apply to ES. In ES, the perturbation vector of each ensemble member is accumulated over time and reused across rounds. Moreover, the selected arms are themselves chosen using past ensemble perturbations. As a result, the perturbation scores at later rounds are not conditionally independent random variables given the past. 
Therefore, a direct Hoeffding-type argument over arms and rounds does not yield the desired arm-wise concentration bound. Our epoch-refresh construction addresses this difficulty by reinitializing the perturbations at epoch boundaries, thereby limiting the interval over which this adaptive dependence accumulates. 

\begin{restatable}[Uniform arm-wise concentration]{lemma}{UniformArmWiseConcentration} \label{lem:uniform arm-wise concentration}
    Suppose Assumption~\ref{assm:boundedness} holds, and suppose that the arm sets \(\{\Xcal_t\}_{t=1}^T\) are chosen by an oblivious adversary with \(|\Xcal_t|=K \ge2\) for all \(t\in[T]\).
    For any \(\alpha>0\) and \(\delta\in(0,1)\), let $L_\alpha:=\log\frac{8 K m T \bar{R}_\alpha}{\delta}$.
    Then, with probability at least \(1-\delta\), we have
    \begin{equation*}
        \max_{t\in[T],\,j\in[m],\,\xb\in\Xcal_t}
        \left|
        \frac{\xb^\top \Vb_{t-1}^{-1}\sbb_{t-1}^j}{\|\xb\|_{\Vb_{t-1}^{-1}}}
        \right|
        \le
        C\left[
        \sqrt{L_\alpha}
        +
        \sqrt{\alpha L_\alpha}
        +
        \alpha(\sqrt d+\sqrt{L_\alpha})
        +
        \alpha\sqrt{\alpha(d+L_\alpha)}
        \right] \, ,            
    \end{equation*}
    where \(C>0\) is a universal constant.
    In particular, if a deterministic quantity $\Lambda_* \ge 1$ satisfies $L_\alpha \le C_{\Lambda} \Lambda_*$ for some universal constant $C_\Lambda > 0$, and if $\alpha = \min \{1, \sqrt{\Lambda_*/d}\}$, then we have 
    \begin{equation*}
        \max_{t\in[T],\,j\in[m],\,\xb\in\Xcal_t}
        \left|
        \frac{\xb^\top \Vb_{t-1}^{-1}\sbb_{t-1}^j}{\|\xb\|_{\Vb_{t-1}^{-1}}}
        \right|
        \le
        \tilde C\sqrt{\Lambda_*} \, ,
    \end{equation*}
    where \(\tilde C>0\) is a universal constant.        
\end{restatable}

\paragraph{Proof sketch of Lemma~\ref{lem:uniform arm-wise concentration}.}
Fix an epoch \(r\) and an ensemble member \(j\in[m]\), and let \(\tau_r\) denote the starting round of the \(r\)-th epoch.
For a round \(t\in\{\tau_r+1,\ldots,\tau_{r+1}\}\), the quantity we need to control is the normalized perturbation score for an arm \(\xb\in\Xcal_t\), defined as
\begin{equation*}
    Z_{r,t}^{j}(\xb)
    :=
    \frac{\xb^\top \Vb_{t-1}^{-1}\sbb_{t-1}^j}
    {\|\xb\|_{\Vb_{t-1}^{-1}}} \, .
\end{equation*}

The ensemble perturbation
\(\sbb^j_{t-1}\) can be written as
\begin{equation*}
    \sbb_{t-1}^j
    = \Vb_{\tau_r}^{1/2} \gb^j_r + \sum_{s=\tau_r + 1}^{t-1} \xb_s \xi^j_s = 
    \Vb_{\tau_r}^{1/2}
    \left(
        \gb_r^j
        +
        \nbb_{r,t-1}^j
    \right),
    \quad \text{where} \quad
    \nbb_{r,t-1}^j
    :=
    \sum_{s=\tau_r+1}^{t-1}
    \Vb_{\tau_r}^{-1/2}\xb_s\xi_s^j \, .
\end{equation*}
Writing $\Ab_{r,t} := \Vb_{\tau_r}^{-1/2} \Vb_{t-1} \Vb_{\tau_r}^{-1/2}$ and $\yb_r(\xb) := \Vb_{\tau_r}^{-1/2} \xb$, the normalized perturbation score can be decomposed as
\begin{equation*}
    Z_{r,t}^{j}(\xb)
    =
    \frac{
        \yb_r(\xb)^\top
        \Ab_{r,t}^{-1}
        \gb_r^j
    }{
        \sqrt{
        \yb_r(\xb)^\top
        \Ab_{r,t}^{-1}
        \yb_r(\xb)}
    }
    +
    \frac{
        \yb_r(\xb)^\top
        \Ab_{r,t}^{-1}
        \nbb_{r,t-1}^j
    }{
        \sqrt{
        \yb_r(\xb)^\top
        \Ab_{r,t}^{-1}
        \yb_r(\xb)}
    } \, .    
\end{equation*}
Moreover, by the epoch-refresh rule, we have $\Ib_d \preceq \Ab_{r,t} \preceq (1+\alpha)\Ib_d$.
Hence, when \(\alpha\) is small, all eigenvalues of \(\Ab_{r,t}\) are close to one. Consequently, \(\Ab_{r,t}^{-1}\) only mildly distorts both the direction \(\yb_r(\xb)\) and the corresponding normalizing factor. More precisely, for any \(\vb\in\mathbb R^d\), 
\begin{equation*}
    \left|
    \frac{
        \yb_r(\xb)^\top \Ab_{r,t}^{-1}\vb
    }{
        \sqrt{\yb_r(\xb)^\top \Ab_{r,t}^{-1}\yb_r(\xb)}
    }
    -
    \left\langle
        \frac{\yb_r(\xb)}{\|\yb_r(\xb)\|_2},
        \vb
    \right\rangle
    \right|
    \lesssim
    \alpha\|\vb\|_2 \, .
\end{equation*}
In words, the self-normalized projection under the current Gram matrix is close to the Euclidean projection after normalization by the epoch-start Gram matrix, and the discrepancy is proportional to the within-epoch Gram-matrix drift \(\alpha\).
Applying this comparison with \(\vb=\gb_r^j\) and
\(\vb=\nbb_{r,t-1}^j\) gives
\begin{equation*}
    |Z_{r,t}^{j}(\xb)|
    \lesssim
    \underbrace{
    \left|
    \left\langle
        \frac{\yb_r(\xb)}{\|\yb_r(\xb)\|_2},
        \gb_r^j
    \right\rangle
    \right|
    }_{(i)}
    +
    \underbrace{
    \left|
    \left\langle
        \frac{\yb_r(\xb)}{\|\yb_r(\xb)\|_2},
        \nbb_{r,t-1}^j
    \right\rangle
    \right|
    }_{(ii)}
    +
    \underbrace{
    \alpha\|\gb_r^j\|_2
    }_{(iii)}
    +
    \underbrace{
    \alpha\|\nbb_{r,t-1}^j\|_2
    }_{(iv)} \, .
\end{equation*}
Thus, the desired finite-arm concentration reduces to controlling four terms: (i) the epoch-start Gaussian projection,
(ii) the scalar within-epoch martingale projection,
(iii) the norm of the epoch-start Gaussian vector, and 
(iv) the norm of the within-epoch martingale vector. 
The first two terms give the main stochastic fluctuations, while the last two are drift terms caused by the change of the Gram matrix within the epoch and are multiplied by the refresh parameter \(\alpha\).

It remains to control these four terms uniformly over all epochs, ensemble members, and arms. The nontrivial terms are (ii) and (iv), since they are accumulated within the epoch and are coupled with the adaptively selected
arms. The scalar martingale projection in (ii) is controlled by an exponential-martingale argument and Ville's inequality. The vector martingale norm in (iv) is controlled by combining the same martingale tail bound with a
sphere-net argument. Combining these estimates proves Lemma~\ref{lem:uniform arm-wise concentration}; the full proof is deferred to Appendix~\ref{appx:proof of finite arm set}. This concentration result, together with the standard optimistic-randomization regret argument, gives the following regret bound.

\begin{restatable}[Regret for finite-arm-dominated regime]{theorem}{RegretForFiniteArmSets}
\label{thm:finite-arm-regret}
    Suppose Assumption~\ref{assm:boundedness} holds, and suppose that the arm sets \(\{\Xcal_t\}_{t=1}^T\) are chosen by an oblivious adversary with \(|\Xcal_t|=K \ge 2\) for all \(t\in[T]\).
    Fix \(\delta\in(0,1/4)\), and choose \(\lambda\ge 80\), \(\gamma=40\), and \(\{\beta_t\}_{t\ge0}\) as in
    Theorem~\ref{thm:regret bound for infinite arm set}.
    For \(\alpha_0 = 1/\sqrt{d}\) and \(\bar R_0 := \bar R_{\alpha_0}\), define
    \begin{equation*}
        M_0
        = \left\lceil
        C_m
        \left[
        d\log\!\left(d+\log\frac{\bar R_0}{\delta}\right)
        +
        \log\frac{\bar R_0}{\delta}
        \right]
        \right\rceil \, ,
        \quad 
        \Lambda_0
        :=
        \max\left\{
            1,\,
            C_\Lambda
            \log\!\left(
                \frac{K M_0 T\bar R_0}{\delta}
            \right)
        \right\} \, ,
    \end{equation*}
    for some constants \(C_m, C_\Lambda > 0\).
    Choose \(\alpha := \min \{1, \sqrt{\Lambda_0/d} \}\) and \(m := M_0\). 
    Then, with probability at least \(1 - 4\delta\), Algorithm~\ref{alg:main algorithm} incurs cumulative regret \(\regret_T=\tilde{\Ocal}(d\sqrt{T\log K})\).
\end{restatable}

\section{Numerical Experiments}
\begin{figure}[t!]
    \centering
    \includegraphics[width=\linewidth]{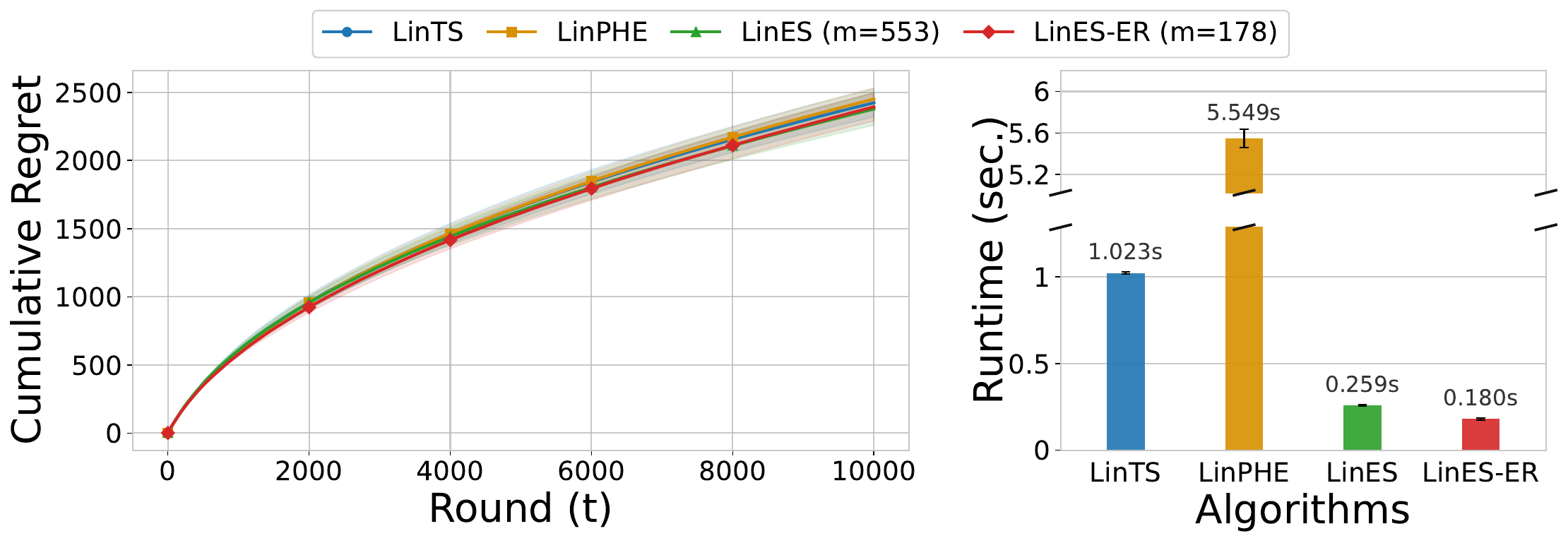}
    \caption{\small
    Cumulative regret and runtime on stochastic linear bandits over \(\Bcal_2^d\) with \(d=30\) and \(T=10{,}000\).
    }
    \label{fig:regret-runtime-d-30}
\end{figure}
We evaluate the proposed ensemble algorithm $\algname$ against \texttt{LinTS}~\citep{abeille2017Linear}, \texttt{LinPHE}~\citep{kveton2020perturbed}, and \texttt{LinES}~\citep{janz2026sharp}.
The experiments are conducted on stochastic linear bandits over the unit ball \(\Bcal_2^d\) with Gaussian noise of variance proxy \(0.3\).
Further experimental details are provided in Appendix~\ref{appx:experiments}.

Figure~\ref{fig:regret-runtime-d-30} reports cumulative regret and runtime for \(d=30\) and \(T=10{,}000\).
For \texttt{LinES}, we use the theoretical ensemble size \(m=\lceil d\log(T/\delta)\rceil\), while for \texttt{LinES-ER}, we use \(m=\lceil d\log d+d\log\log T+\log(1/\delta)\rceil\), with \(\delta=1/T\).
The proposed method, \texttt{LinES-ER}, achieves regret comparable to \texttt{LinES} while using only about one-third as many ensemble members (\(178\) vs. \(553\)).
The runtime comparison shows that this reduction in ensemble size also lowers the computational cost.
In particular, \texttt{LinES-ER} takes \(0.180\) seconds, whereas \texttt{LinES} takes \(0.259\) seconds; thus, \texttt{LinES-ER} reduces the total runtime by approximately \(\frac{0.259-0.180}{0.259}\times 100 \approx 30.5\%\) relative to \texttt{LinES}.

\section{Conclusion}
We proposed an ensemble sampling algorithm for stochastic linear bandits that achieves the same regret order as existing ensemble sampling methods with fewer ensemble members. In particular, the proposed algorithm attains \(\tilde{\Ocal}(d^{3/2}\sqrt T)\) regret for arbitrary bounded arm sets with ensemble size \(\Theta(d\log d+d\log\log T)\), and further achieves the sharper \(\tilde{\Ocal}(d\sqrt{T\log K})\) regret when the arm sets are chosen by an oblivious adversary and each has cardinality \(K\).
The key technical idea is to refresh the ensemble only when the regularized Gram matrix changes substantially, thereby localizing the perturbation analysis to epochs with controlled covariance drift. The algorithm also admits an anytime implementation without resetting past data, making it more data-efficient in practical scenarios.
A promising direction for future work is to characterize the statistical and computational behavior of ensemble sampling in more complex nonlinear reward settings.

\clearpage
\bibliographystyle{plainnat}
\bibliography{references}

@article{hamidi2020frequentist,
  title={On frequentist regret of linear thompson sampling},
  author={Hamidi, Nima and Bayati, Mohsen},
  journal={arXiv preprint arXiv:2006.06790},
  year={2020}
}

@article{osband2016deep,
  title={Deep exploration via bootstrapped DQN},
  author={Osband, Ian and Blundell, Charles and Pritzel, Alexander and Van Roy, Benjamin},
  journal={Advances in neural information processing systems},
  volume={29},
  year={2016}
}

@inproceedings{zhu2023deep,
  title={Deep exploration for recommendation systems},
  author={Zhu, Zheqing and Van Roy, Benjamin},
  booktitle={Proceedings of the 17th ACM Conference on Recommender Systems},
  pages={963--970},
  year={2023}
}

@article{zhou2025stochastic,
  title={Stochastic low-rank tensor bandits for multi-dimensional online decision making},
  author={Zhou, Jie and Hao, Botao and Wen, Zheng and Zhang, Jingfei and Sun, Will Wei},
  journal={Journal of the American Statistical Association},
  volume={120},
  number={549},
  pages={198--211},
  year={2025},
  publisher={Taylor \& Francis}
}

@inproceedings{lu2018efficient,
  title={Efficient online recommendation via low-rank ensemble sampling},
  author={Lu, Xiuyuan and Wen, Zheng and Kveton, Branislav},
  booktitle={Proceedings of the 12th ACM conference on recommender systems},
  pages={460--464},
  year={2018}
}

@InProceedings{janz2026sharp,
  title = 	 {Sharp analysis of linear ensemble sampling},
  author =       {Janz, David and Akhavan, Arya and Szepesv{\'a}ri, Csaba},
  booktitle = 	 {Proceedings of Thirty Ninth Conference on Learning Theory},
  pages = 	 {3716--3750},
  year = 	 {2026},
  volume = 	 {336},
  series = 	 {Proceedings of Machine Learning Research},
  publisher =    {PMLR}
}

@article{osband2018randomized,
	title        = {Randomized prior functions for deep reinforcement learning},
	author       = {Osband, Ian and Aslanides, John and Cassirer, Albin},
	year         = 2018,
	journal      = {Advances in neural information processing systems},
	volume       = 31
}

@article{qin2022analysis,
	title        = {An analysis of ensemble sampling},
	author       = {Qin, Chao and Wen, Zheng and Lu, Xiuyuan and Van Roy, Benjamin},
	year         = 2022,
	journal      = {Advances in Neural Information Processing Systems},
	volume       = 35,
	pages        = {21602--21614}
}

@article{lu2017ensemble,
	title        = {Ensemble sampling},
	author       = {Lu, Xiuyuan and Van Roy, Benjamin},
	year         = 2017,
	journal      = {Advances in neural information processing systems},
	volume       = 30
}

@article{sun2025provable,
	title        = {Provable Anytime Ensemble Sampling Algorithms in Nonlinear Contextual Bandits},
	author       = {Sun, Jiazheng and Wang, Weixin and Xu, Pan},
	year         = 2025,
	journal      = {arXiv preprint arXiv:2510.10730}
}

@inproceedings{janz2024ensemble,
	title        = {Ensemble sampling for linear bandits: small ensembles suffice},
	author       = {David Janz and Alexander Litvak and Csaba Szepesvari},
	year         = 2024,
	booktitle    = {The Thirty-eighth Annual Conference on Neural Information Processing Systems}
}

@inproceedings{lee2024improved,
	title        = {Improved Regret of Linear Ensemble Sampling},
	author       = {Lee, Harin and Oh, Min-hwan},
	year         = 2024,
	booktitle    = {The Thirty-eighth Annual Conference on Neural Information Processing Systems}
}

@inproceedings{kveton2019perturbed,
	title        = {Perturbed-History Exploration in Stochastic Multi-Armed Bandits},
	author       = {Branislav Kveton and Csaba Szepesvari and Mohammad Ghavamzadeh and Craig Boutilier},
	year         = 2019,
	booktitle    = {International Joint Conference on Artificial Intelligence}
}

@article{osband2019deep,
	title        = {Deep Exploration via Randomized Value Functions.},
	author       = {Osband, Ian and Van Roy, Benjamin and Russo, Daniel J and Wen, Zheng and others},
	year         = 2019,
	journal      = {Journal of Machine Learning Research},
	volume       = 20,
	number       = 124,
	pages        = {1--62}
}

@article{abbasi2011improved,
	title        = {Improved algorithms for linear stochastic bandits},
	author       = {Abbasi-Yadkori, Yasin and P{\'a}l, D{\'a}vid and Szepesv{\'a}ri, Csaba},
	year         = 2011,
	journal      = {Advances in neural information processing systems},
	booktitle    = {Advances in Neural Information Processing Systems},
	volume       = 24,
	pages        = {2312--2320}
}

@inproceedings{abeille2017Linear,
	title        = {{Linear Thompson Sampling Revisited}},
	author       = {Abeille, Marc and Lazaric, Alessandro},
	year         = 2017,
	month        = {20--22 Apr},
	booktitle    = {Proceedings of the 20th International Conference on Artificial Intelligence and Statistics},
	publisher    = {PMLR},
	series       = {Proceedings of Machine Learning Research},
	volume       = 54,
	pages        = {176--184},
	editor       = {Singh, Aarti and Zhu, Jerry},
	organization = {PMLR}
}

@inproceedings{chu2011contextual,
	title        = {Contextual bandits with linear payoff functions},
	author       = {Chu, Wei and Li, Lihong and Reyzin, Lev and Schapire, Robert},
	year         = 2011,
	booktitle    = {Proceedings of the Fourteenth International Conference on Artificial Intelligence and Statistics},
	pages        = {208--214},
	organization = {JMLR Workshop and Conference Proceedings}
}

@inproceedings{agrawal2013thompson,
	title        = {Thompson sampling for contextual bandits with linear payoffs},
	author       = {Agrawal, Shipra and Goyal, Navin},
	year         = 2013,
	booktitle    = {International conference on machine learning},
	pages        = {127--135},
	organization = {PMLR}
}

@inproceedings{kveton2020perturbed,
	title        = {Perturbed-History Exploration in Stochastic Linear Bandits},
	author       = {Kveton, Branislav and Szepesv{\'a}ri, Csaba and Ghavamzadeh, Mohammad and Boutilier, Craig},
	year         = 2020,
	booktitle    = {Uncertainty in Artificial Intelligence},
	pages        = {530--540},
	organization = {PMLR}
}

@inproceedings{kveton2020randomized,
	title        = {Randomized exploration in generalized linear bandits},
	author       = {Kveton, Branislav and Zaheer, Manzil and Szepesvari, Csaba and Li, Lihong and Ghavamzadeh, Mohammad and Boutilier, Craig},
	year         = 2020,
	booktitle    = {International Conference on Artificial Intelligence and Statistics},
	pages        = {2066--2076},
	organization = {PMLR}
}

@article{thompson1933likelihood,
	title        = {On the likelihood that one unknown probability exceeds another in view of the evidence of two samples},
	author       = {Thompson, William R},
	year         = 1933,
	journal      = {Biometrika},
	publisher    = {JSTOR},
	volume       = 25,
	number       = {3/4},
	pages        = {285--294}
}

@article{besson2018doubling,
	title        = {What doubling tricks can and can't do for multi-armed bandits},
	author       = {Besson, Lilian and Kaufmann, Emilie},
	year         = 2018,
	journal      = {arXiv preprint arXiv:1803.06971}
}

\clearpage
\appendix
\onecolumn

\renewcommand{\contentsname}{Contents of Appendix}
\addtocontents{toc}{\protect\setcounter{tocdepth}{3}}
{
  \hypersetup{hidelinks}
  \tableofcontents
}

\clearpage

\section{Regret Analysis for Dimension-Dominated Regime} \label{appx:regret bound for infinite arm set}
In this section, we provide the detailed proof of Theorem~\ref{thm:regret bound for infinite arm set}.
We first establish a deterministic upper bound on the number of epochs, as discussed in the main text.
We then prove a self-normalized exceedance guarantee within a single epoch and extend it uniformly across all epochs.
Finally, we combine these ingredients with the standard regret-analysis framework for optimistic randomized algorithms to complete the proof of Theorem~\ref{thm:regret bound for infinite arm set}.

\subsection{Deterministic Upper Bound on the Number of Epochs}
The number of epochs generated by Algorithm~\ref{alg:main algorithm} is random, since epoch refreshes depend on the realized Gram matrices, which in turn are determined by the algorithmic randomization and the observed data. In this section, we derive a deterministic upper bound on the total number of epochs.

Define the epoch start times recursively by
\begin{equation*}
    \tau_0 := 0,
    \qquad
    \tau_{r+1}
    :=
    \min \left\{
    \inf \left\{
    t>\tau_r:\ \Vb_t \npreceq (1 + \alpha)\Vb_{\tau_r}
    \right\}
    \, , T \right\} \, ,
\end{equation*}
where we adopt the convention that $\inf \emptyset = \infty$. 
Then, the $r$-th epoch then consists of the rounds
\begin{equation*}
    \tau_r < t \le \tau_{r+1} \, .
\end{equation*}
The total number of epochs is bounded as follows.

\begin{lemma}[Number of epochs] \label{lem:num-epochs}
    For $\alpha > 0$, let $R_\alpha$ denote the total number of epochs. Under the assumption that $\|\xb\|_2 \le 1$ for all $\xb \in \Xcal$, we have
    \begin{equation*}
        R_{\alpha}
        \le
        1 + \frac{d \log \left(1+\frac{T}{\lambda d}\right)}{\log (1 + \alpha)} \, .
    \end{equation*}
    Consequently, the deterministic quantity
    \begin{equation*}
        \bar{R}_{\alpha}
        :=
        1+\left\lfloor \frac{d \log \left(1+\frac{T}{\lambda d}\right)}{\log (1 + \alpha)} \right\rfloor
    \end{equation*}
    satisfies $R_\alpha \le \bar{R}_\alpha$.
\end{lemma}
\begin{proof}[Proof of Lemma~\ref{lem:num-epochs}]
    If epoch $r$ is completed before the horizon, then by definition,
    \[
    \Vb_{\tau_{r+1}} \npreceq (1+\alpha)\Vb_{\tau_r} \, .
    \]
    Let
    \[
    \Ab_r := \Vb_{\tau_r}^{-1/2}\Vb_{\tau_{r+1}}\Vb_{\tau_r}^{-1/2} \, .
    \]
    Since $\Vb_{\tau_{r+1}} \succeq \Vb_{\tau_r}$, all eigenvalues of $\Ab_r$ are at least $1$. Moreover, since $\Vb_{\tau_{r+1}} \npreceq (1+\alpha)\Vb_{\tau_r}$ by construction, at least one eigenvalue of $\Ab_r$ is strictly larger than $1+\alpha$. Hence,
    \begin{equation}\label{eq:num-epochs-1}
        \detm(\Ab_r) > 1+\alpha
        \quad \Longrightarrow \quad
        \frac{\detm(\Vb_{\tau_{r+1}})}{\detm(\Vb_{\tau_r})} > 1+\alpha \, .
    \end{equation}
    
    Since there are $R_\alpha$ epochs in total, exactly $R_\alpha-1$ epochs are completed before the horizon, namely those indexed by $r=0,1,\dots,R_\alpha-2$. Multiplying the inequalities in \eqref{eq:num-epochs-1} over these indices yields
    \[
    \frac{\detm(\Vb_{\tau_{R_\alpha-1}})}{\detm(\lambda \Ib_d)}
    =
    \prod_{r=0}^{R_\alpha-2}
    \frac{\detm(\Vb_{\tau_{r+1}})}{\detm(\Vb_{\tau_r})}
    >
    (1+\alpha)^{R_\alpha-1} \, .
    \]
    
    Finally, since $\tau_{R_\alpha}=T$ and $\Vb_t$ is increasing in $t$ with respect to the positive semidefinite order, we have
    \[
    \Vb_T \succeq \Vb_{\tau_{R_\alpha-1}},
    \]
    which implies
    \[
    \detm(\Vb_T) \ge \detm(\Vb_{\tau_{R_\alpha-1}}) \, .
    \]
    Therefore,
    \[
    (1+\alpha)^{R_\alpha-1}
    <
    \frac{\detm(\Vb_{\tau_{R_\alpha-1}})}{\detm(\lambda \Ib_d)}
    \le
    \frac{\detm(\Vb_T)}{\detm(\lambda \Ib_d)} \, .
    \]
    This proves that
    \[
    R_\alpha
    \le
    1+\frac{\log\!\bigl(\detm(\Vb_T)/\detm(\lambda \Ib_d)\bigr)}{\log(1+\alpha)} \, .
    \]
    
    Also, the determinant-trace inequality (Lemma~\ref{aux_lem:det-tr inequality}) gives
    \[
    \detm(\Vb_T)\le \left(\lambda+\frac{T}{d}\right)^d,
    \]
    and hence
    \[
    \frac{\detm(\Vb_T)}{\detm(\lambda \Ib_d)}
    \le
    \left(1+\frac{T}{\lambda d}\right)^d \, .
    \]
    Combining the above displays yields the final bound
    \[
    R_\alpha
    \le
    1+\frac{d\log\!\left(1+\frac{T}{\lambda d}\right)}{\log(1+\alpha)} \, .
    \]
    Since $R_\alpha$ is an integer, this further implies
    \[
        R_\alpha
        \le
        1+
        \left\lfloor
            \frac{d\log\!\left(1+\frac{T}{\lambda d}\right)}
            {\log(1+\alpha)}
        \right\rfloor = \bar R_\alpha .
    \]    
\end{proof}

\subsection{Self-Normalized Exceedance Frequencies Inside One Epoch}

The central object in the analysis is the self-normalized exceedance frequency, which measures the fraction of ensemble perturbations that are sufficiently large in a given direction at a given round.
In the setting of~\citet{janz2026sharp}, for each fixed direction, the projected perturbation process can be represented by Brownian motions whose clock is the corresponding self-normalized norm. Since the ensemble is not refreshed, this clock must be controlled over the full Gram-matrix growth along the horizon.
Our epoch-refresh construction localizes this difficulty: within each epoch, the Gram matrix changes by at most a factor \(1+\alpha\), so the associated Brownian clock varies only over a local multiplicative range. Therefore, the Brownian exceedance theorem can be applied with clock ratio \(1+\alpha\), rather than the global Gram-matrix growth. This is the key reason the fixed-direction exceedance guarantee can be obtained with fewer ensemble members.

\FixedDriectionEpoch*

\begin{proof}[Proof of Lemma~\ref{lem:fixed-direction-epoch}]
    Fix an epoch $r$ and condition on the history up to time $\tau_r$.
    By construction,
    \[
    \forall j \in [m]: \qquad
    \sbb_{\tau_r}^j = \Vb_{\tau_r}^{1/2}\gb_r^j,
    \qquad
    \gb_r^j \sim \Ncal(\zero_d,\Ib_d).
    \]
    Hence, for every $t\in\{\tau_r+1,\dots,\tau_{r+1}\}$ and every $j\in[m]$,
    \[
    \ips{\ub}{\sbb_{t-1}^j}
    =
    \ip{\ub}{\Vb_{\tau_r}^{1/2}\gb_r^j}
    +
    \sum_{s=\tau_r+1}^{t-1}\ip{\ub}{\xb_s}\,\xi_s^j.
    \]
    Note that for \(t\in\{\tau_r+1,\dots,\tau_{r+1}\}\), the arm \(\xb_t\) is selected before \(\xi_t^j\) is sampled. Hence, the coefficient \(\ips{\ub}{\xb_t}\) is predictable with respect to the filtration immediately before the perturbation noise \(\xi_t^j\) is drawn. Therefore, the projected process is a Gaussian martingale transform with predictable scalar coefficients.    

    On the other hand, by the epoch-refresh construction, all ensemble perturbations are reinitialized at time \(\tau_r\), so within epoch \(r\) they evolve from a common Gaussian restart point. This differs from the Brownian representation in~\citet{janz2026sharp}, where the process is tracked globally over the entire horizon starting from time \(0\). Here, by conditioning on the history up to \(\tau_r\), we obtain an \emph{epoch-wise} Brownian representation whose clock starts from the local covariance matrix \(\Vb_{\tau_r}\).
    We formalize this observation in the following lemma, which gives an epoch-wise Brownian representation for a diagonal martingale transform of Gaussian noise. Its proof is deferred to Appendix~\ref{appx:proof of epoch-brownian-representation}.

    \begin{restatable}[Epoch-wise scalar-clock Brownian representation]{lemma}{EpochBrownianRepresentation} \label{lem:epoch-brownian-representation}
        Fix an epoch \(r\) and a direction \(\ub\in \usd\). Let
        \[
        \ell_r := \tau_{r+1}-\tau_r \, .
        \]
        For \(k=0,1,\dots,\ell_r-1\), define the \(m\)-dimensional vector
        \[
        \Mb_k
        :=
        \bigl(
        \ips{\ub}{\sbb_{\tau_r+k}^1},\dots,\ips{\ub}{\sbb_{\tau_r+k}^m}
        \bigr)\in\RR^m \, .
        \]
        Then \((\Mb_k)_{k=0}^{\ell_r-1}\) is a diagonal martingale transform of a standard
        \(m\)-dimensional Gaussian noise sequence with scalar coefficients
        \[
        D_0 := \norm{\ub}_{\Vb_{\tau_r}},
        \qquad
        D_k := \ip{\ub}{\xb_{\tau_r+k}},
        \qquad k=1,\dots,\ell_r-1 \, .
        \]
        Consequently, on an extension of the probability space, there exist independent
        standard Brownian motions \(W^1,\dots,W^m\) such that for every
        \(k=0,\dots,\ell_r-1\) and every \(j\in[m]\),
        \[
        \Mb_k^j
        =
        W^j(A_k^2) \, ,
        \qquad
        A_k^2
        :=
        D_0^2+\sum_{s=1}^k D_s^2
        =
        \norm{\ub}_{\Vb_{\tau_r+k}}^2 \, .
        \]
        Equivalently, for every \(t\in\{\tau_r+1,\dots,\tau_{r+1}\}\) and every \(j\in[m]\),
        \[
        \ips{\ub}{\sbb_{t-1}^j}
        =
        W^j\!\bigl(\norm{\ub}_{\Vb_{t-1}}^2\bigr) \, .
        \]        
    \end{restatable}
    
    Applying Lemma~\ref{lem:epoch-brownian-representation}, on an extension of the probability space, there exist independent standard Brownian motions
    \(W^1,\dots,W^m\) such that for all
    \(t\in\{\tau_r+1,\dots,\tau_{r+1}\}\) and all \(j\in[m]\),
    \[
    \ips{\ub}{\sbb_{t-1}^j}
    =
    W^j\!\bigl(\norm{\ub}_{\Vb_{t-1}}^2\bigr).
    \]
    Therefore,
    \[
    E_{t,m}(\ub,2/\gamma)
    =
    \frac1m\sum_{j=1}^m
    \ind \!\left\{
    \frac{W^j(\norm{\ub}_{\Vb_{t-1}}^2)}{\norm{\ub}_{\Vb_{t-1}}}
    \ge \frac{2}{\gamma}
    \right\}.
    \]
    
    Now, for every \(t\in\{\tau_r+1,\dots,\tau_{r+1}\}\), we have
    \[
    \Vb_{\tau_r}\preceq \Vb_{t-1}\preceq (1+\alpha)\Vb_{\tau_r} \, ,
    \]
    which implies
    \[
    \norm{\ub}_{\Vb_{t-1}}^2
    \in
    \bigl[\norm{\ub}_{\Vb_{\tau_r}}^2,\ (1+\alpha)\norm{\ub}_{\Vb_{\tau_r}}^2\bigr] \, .
    \]
    Thus, the Brownian clock ratio is at most \(1+\alpha\).
    
    We now invoke the time-uniform exceedance count result for Brownian motions along a fixed direction due to~\citet{janz2026sharp}.
    
    \begin{lemma}[Theorem 5.4 in~\citet{janz2026sharp}] \label{thm:akhavan-brownian}
        Let \(W^1,\dots,W^m\) be independent standard Brownian motions.
        Fix \(0<\tau\le \tau'<\infty\), \(c>0\), and \(0<p<p_0(c):=\frac14(1-\Phi(c))\).
        Then there exists \(h\in(0,1]\) such that, for any \(\delta\in(0,1)\), if
        \[
        m \ge \frac{4}{p}\log\!\left(\frac{\lceil \log(\tau'/\tau)/h\rceil}{\delta}\right),
        \]
        then with probability at least \(1-\delta\),
        \[
        \inf_{t\in[\tau,\tau']}
        \frac1m\sum_{j=1}^m
        \ind \!\left\{\frac{W^j(t)}{\sqrt{t}}\ge c\right\}
        \ge p.
        \]
        Moreover, if \(c\le 1/20\) and \(p\le 1/10\), then \(h=1/250\) is admissible.
    \end{lemma}
    
    Applying Lemma~\ref{thm:akhavan-brownian} on the interval
    \[
    [\tau,\tau']
    =
    \bigl[\norm{\ub}_{\Vb_{\tau_r}}^2,\ (1+\alpha)\norm{\ub}_{\Vb_{\tau_r}}^2\bigr],
    \]
    with \(p=1/10\), \(c=1/20\), and the admissible choice \(h=1/250\), we obtain the desired conclusion provided that
    \[
    m \ge \frac{4}{p}\log\!\left(\frac{\lceil \log(\tau'/\tau)/h\rceil}{\delta'}\right)
    =
    40 \log\!\left(\frac{\lceil 250 \log(\tau'/\tau)\rceil}{\delta'}\right).
    \]
    Since \(\tau'/\tau\le 1+\alpha\), it suffices to require
    \[
    m \ge 40\log \frac{\kappa_\alpha}{\delta'} \, ,
    \qquad
    \kappa_\alpha:=\left\lceil 250\log(1+\alpha)\right\rceil.
    \]
    This concludes the proof.
\end{proof}

\subsection{Self-Normalized Exceedance Frequencies Across Epochs}
The fixed-direction bound is not sufficient for general bounded arm sets, because the optimistic direction can depend on both the round and the unknown optimal arm. We therefore extend the fixed-direction guarantee uniformly over all directions on the unit sphere.
The first ingredient is a local Lipschitz property of the normalized map induced by the Gram-matrix change within an epoch. 

\begin{lemma}[Local Lipschitz map within one epoch] \label{lem:local-lipschitz}
    Fix an epoch \(r\) and a round \(t\in\{\tau_r+1,\dots,\tau_{r+1}\}\).
    Let
    \[
    \Bb_{r,t} := \Vb_{t-1}^{1/2}\Vb_{\tau_r}^{-1/2}.
    \]
    Define a map \(\phib_{r,t} : \usd \to \usd\) by
    \[
    \phib_{r,t}(\wb) := \frac{\Bb_{r,t}\wb}{\norm{\Bb_{r,t}\wb}_2},
    \qquad \wb\in \usd.
    \]
    Then, for any \(\wb,\wb' \in \usd\),
    \[
    \norm{\phib_{r,t}(\wb)-\phib_{r,t}(\wb')}_2
    \le
    2\sqrt{1+\alpha}\,\norm{\wb-\wb'}_2 \, .
    \]
\end{lemma}

\begin{proof}[Proof of Lemma~\ref{lem:local-lipschitz}]
    Since \(t\) lies in epoch \(r\), by the definition of the epoch refresh rule we have
    \[
    \Vb_{\tau_r}\preceq \Vb_{t-1}\preceq (1+\alpha)\Vb_{\tau_r}.
    \]
    Multiplying on the left and right by \(\Vb_{\tau_r}^{-1/2}\), we obtain
    \[
    \Ib_d
    \preceq
    \Vb_{\tau_r}^{-1/2}\Vb_{t-1}\Vb_{\tau_r}^{-1/2}
    =
    \Bb_{r,t}^{\top}\Bb_{r,t}
    \preceq
    (1+\alpha)\Ib_d.
    \]
    Therefore,
    \[
    \norm{\Bb_{r,t}}_{\op}\le \sqrt{1+\alpha},
    \qquad
    \inf_{\norm{\wb}_2=1}\norm{\Bb_{r,t}\wb}_2 \ge 1.
    \]
    
    Now let
    \[
    \ab:=\Bb_{r,t}\wb,
    \qquad
    \bb:=\Bb_{r,t}\wb',
    \]
    where \(\wb,\wb'\in\usd\). Since \(\norm{\ab}_2,\norm{\bb}_2\ge 1\), we may use the elementary inequality
    \[
     \left\| \frac{\ab}{\norm{\ab}_2}-\frac{\bb}{\norm{\bb}_2} \right\|_2
    \le
    \frac{2\norm{\ab-\bb}_2}{\min\{\norm{\ab}_2,\norm{\bb}_2\}} \, .
    \]
    Applying this with \(\ab=\Bb_{r,t}\wb\) and \(\bb=\Bb_{r,t}\wb'\), and using \(\min\{\norm{\ab}_2,\norm{\bb}_2\}\ge 1\), we obtain
    \[
    \norm{\phib_{r,t}(\wb)-\phib_{r,t}(\wb')}_2
    \le
    2\norm{\Bb_{r,t}(\wb-\wb')}_2
    \le
    2\norm{\Bb_{r,t}}_{\op}\norm{\wb-\wb'}_2
    \le
    2\sqrt{1+\alpha}\,\norm{\wb-\wb'}_2.
    \]
\end{proof}

To make the net argument quantitative, we also need a uniform bound on the normalized ensemble perturbations over all epochs. This guarantees that transferring exceedance from a nearby net point to an arbitrary direction incurs only a controlled loss, uniformly over time and across epochs.

\begin{lemma}[Uniform bound on normalized perturbations] \label{lem:uniform-state-bound}
    For a given \(\delta \in (0,1)\), define
    \[
    \Gamma_\delta
    :=
    \underbrace{
    \sqrt{d} + \sqrt{2\log\!\left(\frac{4m\bar{R}_{\alpha}}{\delta}\right)}
    }_{\text{refreshed Gaussian term}}
    +
    \underbrace{
    \sqrt{d\log(1+\alpha) + 2\log\!\left(\frac{4m\bar{R}_{\alpha}}{\delta}\right)}
    }_{\text{within-epoch self-normalized term}}.
    \]
    With probability at least \(1-\delta\), we have
    \[
    \forall r\le R_\alpha,\ \forall j\in[m],\ \forall t\in\{\tau_r+1,\dots,\tau_{r+1}\}:
    \qquad
    \norm{\Vb_{t-1}^{-1/2}\sbb_{t-1}^j}_2 \le \Gamma_\delta.
    \]
\end{lemma}
\begin{proof}[Proof of Lemma~\ref{lem:uniform-state-bound}]
    
    Fix an epoch \(r\) and an ensemble index \(j\in[m]\).
    For each \(t\in\{\tau_r+1,\dots,\tau_{r+1}\}\), the construction of the algorithm gives
    \[
    \sbb_{t-1}^j
    =
    \Vb_{\tau_r}^{1/2}\gb_r^j
    +
    \sum_{s=\tau_r+1}^{t-1} \xb_s \xi_s^j.
    \]
    Hence,
    \[
    \Vb_{t-1}^{-1/2}\sbb_{t-1}^j
    =
    \underbrace{\Vb_{t-1}^{-1/2}\Vb_{\tau_r}^{1/2} \gb_r^j}_{\text{refreshed Gaussian term}}
    +
    \underbrace{\Vb_{t-1}^{-1/2}\sum_{s=\tau_r+1}^{t-1} \xb_s \xi_s^j}_{\text{within-epoch martingale term}} \, .
    \]
    We bound these two terms separately.
    
    \medskip
    \noindent
    \textbf{Step 1: The refreshed Gaussian term.}
    Since \(\Vb_{t-1}\succeq \Vb_{\tau_r}\), we have
    \[
    \norm{\Vb_{t-1}^{-1/2}\Vb_{\tau_r}^{1/2}}_{\op}\le 1.
    \]
    Therefore,
    \[
    \norm{\Vb_{t-1}^{-1/2}\Vb_{\tau_r}^{1/2}\gb_r^j}_2
    \le
    \norm{\gb_r^j}_2.
    \]
    Moreover, by the concentration inequality for the Euclidean norm of a standard Gaussian vector (Lemma~\ref{aux_lem:gaussian-norm-concentration}),
    \[
    \PP\!\left(\norm{\gb_r^j}_2 > \sqrt{d} + \sqrt{2x}\right)\le e^{-x}.
    \]
    Taking
    \[
    x=\log\!\left(\frac{4m\bar R_\alpha}{\delta}\right),
    \]
    and applying a union bound over all pairs \((r,j)\), we conclude that, with probability at least \(1-\delta/2\),
    \[
    \forall r\le \bar R_\alpha,\ \forall j\in[m]:
    \qquad
    \norm{\gb_r^j}_2
    \le
    \sqrt{d}+\sqrt{2\log\!\left(\frac{4m\bar R_\alpha}{\delta}\right)}.
    \]
    
    \medskip
    \noindent
    \textbf{Step 2: The within-epoch martingale term.}
    
    For a fixed epoch \(r\) and ensemble index \(j\in[m]\), define
    \[
    \Mb_t^{(r,j)}
    :=
    \sum_{s=\tau_r+1}^{t-1}\xb_s \xi_s^j,
    \qquad
    t\in\{\tau_r+1,\dots,\tau_{r+1}\}.
    \]
    Conditional on the history up to time \(\tau_r\), namely \(\mathcal F_{\tau_r}\), the process \(\{\Mb_t^{(r,j)}\}_{t=\tau_r+1}^{\tau_{r+1}}\) is a vector-valued martingale transform with conditionally Gaussian increments and predictable coefficients. More precisely, for each \(s\in\{\tau_r+1,\dots,\tau_{r+1}-1\}\), the increment
    \[
    \xb_s \xi_s^j
    \]
    has conditional mean zero given the past, the scalar noise \(\xi_s^j\) is conditionally standard Gaussian, and the coefficient vector \(\xb_s\) is predictable.
    
    We may therefore apply the standard self-normalized inequality for vector-valued martingales with conditionally sub-Gaussian increments (Lemma~\ref{aux_lem:self-normalized inequality}), with initial matrix \(\Vb_{\tau_r}\). This implies that for every \(\eta\in(0,1)\),
    \begin{equation} \label{eq:uniform-state-bound-1}
        \PP\!\left(
        \exists\, t\in\{\tau_r+1,\dots,\tau_{r+1}\} :
        \norm{\Vb_{t-1}^{-1/2}\Mb_t^{(r,j)}}_2
        >
        \sqrt{
        2\log(1/\eta)
        +
        \log\!\frac{\detm(\Vb_{t-1})}{\detm(\Vb_{\tau_r})}
        }
        \ \middle|\ \mathcal F_{\tau_r}
        \right)
        \le \eta.    
    \end{equation}
    
    Next, we control the determinant ratio uniformly within the epoch. Since \(t\) lies in epoch \(r\), the epoch refresh rule implies
    \[
    \Vb_{t-1}\preceq (1+\alpha)\Vb_{\tau_r}.
    \]
    Hence,
    \[
    \frac{\detm(\Vb_{t-1})}{\detm(\Vb_{\tau_r})}
    \le
    (1+\alpha)^d.
    \]
    Substituting this bound into the Eq.~\eqref{eq:uniform-state-bound-1}, we obtain
    \[
    \PP\!\left(
    \sup_{t\in\{\tau_r+1,\dots,\tau_{r+1}\}}
    \norm{\Vb_{t-1}^{-1/2}\Mb_t^{(r,j)}}_2
    >
    \sqrt{d\log(1+\alpha)+2\log(1/\eta)}
    \right)
    \le \eta.
    \]
    
    Finally, taking
    \[
    \eta=\frac{\delta}{2m\bar R_\alpha},
    \]
    and applying a union bound over all epochs \(r\le \bar R_\alpha\) and all ensemble indices \(j\in[m]\), we conclude that, with probability at least \(1-\delta/2\),
    $\forall r \le \bar R_\alpha$, $\forall j \in [m]$, and $\forall t\in\{\tau_r+1,\dots,\tau_{r+1}\}$, it holds:
    \begin{equation*}
        \norm{\Vb_{t-1}^{-1/2}\Mb_t^{(r,j)}}_2
        \le
        \sqrt{d\log(1+\alpha)+2\log\!\left(\frac{2m\bar R_\alpha}{\delta}\right)} \, .
    \end{equation*}
    Intersecting the events from Steps 1 and 2 and applying the triangle inequality, we conclude that, with probability at least \(1-\delta\),
    \[
    \norm{\Vb_{t-1}^{-1/2}\sbb_{t-1}^j}_2
    \le
    \Gamma_\delta
    \]
    for all \(r\le \bar R_\alpha\), all \(j\in[m]\), and all \(t\in\{\tau_r+1,\dots,\tau_{r+1}\}\), where
    \[
    \Gamma_\delta
    =
    \left(
    \sqrt{d}+\sqrt{2\log\!\left(\frac{4m\bar R_\alpha}{\delta}\right)}
    \right)
    +
    \left(
    \sqrt{d\log(1+\alpha)+2\log\!\left(\frac{4m\bar R_\alpha}{\delta}\right)}
    \right).
    \]
    This concludes the proof.
\end{proof}

We are now ready to combine the preceding ingredients. The fixed-direction epochwise exceedance bound provides control at a single net point within a single epoch. The local Lipschitz lemma shows that nearby directions induce nearby normalized geometries, while the uniform perturbation bound controls the loss incurred when passing from a net point to an arbitrary direction. Taking a union bound over the sphere net and over all epochs then yields a uniform exceedance lower bound simultaneously for all times and all directions.
\UniformExceedanceCountAcrossEpochs*

\begin{proof}[Proof of Lemma~\ref{lem:uniform-exceedance-across-epochs}]
    Fix an epoch \(r\).
    For each nonzero \(\ub\in \RR^d\), define
    \[
    \wb(\ub) := \frac{\Vb_{\tau_r}^{1/2}\ub}{\norm{\ub}_{\Vb_{\tau_r}}}\in \usd.
    \]
    Then for every \(t\in\{\tau_r+1,\dots,\tau_{r+1}\}\) and every \(j\in[m]\), since
    \[
    \frac{\Vb_{t-1}^{1/2}\ub}{\norm{\ub}_{\Vb_{t-1}}}
    =
    \frac{\Vb_{t-1}^{1/2}\Vb_{\tau_r}^{-1/2}\wb(\ub)}
    {\norm{\Vb_{t-1}^{1/2}\Vb_{\tau_r}^{-1/2}\wb(\ub)}_2}
    =
    \frac{\Bb_{r,t}\wb(\ub)}
    {\norm{\Bb_{r,t}\wb(\ub)}_2}
    =
    \phib_{r,t}(\wb(\ub)),
    \]
    we have
    \begin{equation}
    \label{eq:epoch-reparam}
    \frac{\ips{\ub}{\sbb_{t-1}^j}}{\norm{\ub}_{\Vb_{t-1}}}
    =
    \frac{\ip{\Vb_{t-1}^{1/2}\ub}{\Vb_{t-1}^{-1/2}\sbb_{t-1}^j}}{\norm{\ub}_{\Vb_{t-1}}}
    =
    \ip{
    \phib_{r,t}(\wb(\ub))
    }{
    \Vb_{t-1}^{-1/2}\sbb_{t-1}^j
    }.
    \end{equation}
    
    Now let \(\mathcal N_r\) be a deterministic \(\varepsilon\)-net of \(\usd\) with \(\varepsilon := 1/(2\sqrt{1+\alpha}\,\gamma \Gamma_\delta)\). By the standard covering bound, we may choose \(\mathcal N_r\) so that
    \[
    |\mathcal N_r|
    \le
    \left(\frac{3}{\varepsilon}\right)^d
    =
    \left( 6\sqrt{1+\alpha}\,\gamma \Gamma_\delta \right)^d
    =: N_\delta \, .
    \]
    
    For each fixed \(\wb'\in \mathcal N_r\), choose a nonzero vector \(\ub(\wb')\) such that \(\wb(\ub(\wb'))=\wb'\). Such a vector exists because \(\Vb_{\tau_r}\) is positive definite, and hence \(\Vb_{\tau_r}^{1/2}\) is invertible. Therefore, the map \(\ub \mapsto \Vb_{\tau_r}^{1/2}\ub/\norm{\ub}_{\Vb_{\tau_r}}\) is surjective onto \(\usd\).
    
    Applying Lemma~\ref{lem:fixed-direction-epoch} with \(\delta_{r,\wb'} := \delta/(2\bar R_\alpha N_\delta)\), and using condition~\eqref{eq:ER-condition} to ensure that
    \begin{equation*}
        m \ge 40 \log \frac{\kappa_\alpha}{\delta_{r,\wb'}} = 40 \log \frac{2 \kappa_\alpha \bar R_\alpha N_\delta }{\delta} \, ,
    \end{equation*}
    we conclude that, with probability at least \(1-\delta_{r,\wb'}\),
    \[
    \forall t\in\{\tau_r+1,\dots,\tau_{r+1}\}:
    \qquad
    E_{t,m}(\ub(\wb'),2/\gamma)\ge \frac{1}{10}.
    \]
    
    Define the event
    \[
    \mathcal E_{\mathrm{net}}
    :=
    \left\{
    \forall r\le \bar R_\alpha,\ \forall \wb'\in\mathcal N_r,\ \forall t\in\{\tau_r+1,\dots,\tau_{r+1}\}:
    \quad
    E_{t,m}(\ub(\wb'),2/\gamma)\ge \frac{1}{10}
    \right\}.
    \]
    Taking a union bound over all epochs \(r\le \bar R_\alpha\) and all net points \(\wb'\in\mathcal N_r\), we obtain \(\PP(\mathcal E_{\mathrm{net}})\ge 1-\delta/2\).
    
    Next, define
    \[
    \mathcal E_{\mathrm{state}}
    :=
    \left\{
    \forall r\le \bar R_\alpha,\ \forall j\in[m],\ \forall t\in\{\tau_r+1,\dots,\tau_{r+1}\}:
    \quad
    \norm{\Vb_{t-1}^{-1/2}\sbb_{t-1}^j}_2\le \Gamma_\delta
    \right\}.
    \]
    By Lemma~\ref{lem:uniform-state-bound}, we have \(\PP(\mathcal E_{\mathrm{state}})\ge 1-\delta/2\). Hence, by the union bound, \(\PP(\mathcal E_{\mathrm{net}}\cap \mathcal E_{\mathrm{state}})\ge 1-\delta\).
    
    We now work on the event \(\mathcal E_{\mathrm{net}}\cap \mathcal E_{\mathrm{state}}\).
    Fix an epoch \(r\), a round \(t\in\{\tau_r+1,\dots,\tau_{r+1}\}\), an ensemble index \(j\in[m]\), and \(\wb,\wb'\in\usd\). By Lemma~\ref{lem:local-lipschitz},
    \[
    \norm{\phib_{r,t}(\wb)-\phib_{r,t}(\wb')}_2
    \le 2\sqrt{1+\alpha}\,\norm{\wb-\wb'}_2.
    \]
    Since \(\norm{\Vb_{t-1}^{-1/2}\sbb_{t-1}^j}_2\le \Gamma_\delta\) on \(\mathcal E_{\mathrm{state}}\), we obtain
    \begin{align*}
    \left|
    \ip{\phib_{r,t}(\wb)}{\Vb_{t-1}^{-1/2}\sbb_{t-1}^j}
    -
    \ip{\phib_{r,t}(\wb')}{\Vb_{t-1}^{-1/2}\sbb_{t-1}^j}
    \right|
    &=
    \left|
    \ip{\phib_{r,t}(\wb)-\phib_{r,t}(\wb')}{\Vb_{t-1}^{-1/2}\sbb_{t-1}^j}
    \right| \\
    &\le
    \norm{\phib_{r,t}(\wb)-\phib_{r,t}(\wb')}_2\,
    \norm{\Vb_{t-1}^{-1/2}\sbb_{t-1}^j}_2 \\
    &\le
    2\sqrt{1+\alpha}\,\Gamma_\delta \norm{\wb-\wb'}_2.
    \end{align*}
    
    Now fix any \(t\in[T]\) and any \(\ub\in\usd\), and let \(r\) be the epoch containing round \(t\). Choose \(\wb'\in\mathcal N_r\) such that \(\norm{\wb(\ub)-\wb'}_2\le \varepsilon\). Then, for every \(j\in[m]\),
    \[
    \left|
    \ip{\phib_{r,t}(\wb(\ub))}{\Vb_{t-1}^{-1/2}\sbb_{t-1}^j}
    -
    \ip{\phib_{r,t}(\wb')}{\Vb_{t-1}^{-1/2}\sbb_{t-1}^j}
    \right|
    \le
    2\sqrt{1+\alpha}\,\Gamma_\delta \varepsilon
    =
    \frac{1}{\gamma}.
    \]
    Hence, if
    \[
    \ip{\phib_{r,t}(\wb')}{\Vb_{t-1}^{-1/2}\sbb_{t-1}^j}\ge \frac{2}{\gamma},
    \]
    then necessarily
    \[
    \ip{\phib_{r,t}(\wb(\ub))}{\Vb_{t-1}^{-1/2}\sbb_{t-1}^j}\ge \frac{1}{\gamma}.
    \]
    Using~\eqref{eq:epoch-reparam}, this is equivalent to
    \[
    \frac{\ips{\ub(\wb')}{\sbb_{t-1}^j}}{\norm{\ub(\wb')}_{\Vb_{t-1}}}\ge \frac{2}{\gamma}
    \quad\Longrightarrow\quad
    \frac{\ips{\ub}{\sbb_{t-1}^j}}{\norm{\ub}_{\Vb_{t-1}}}\ge \frac{1}{\gamma}.
    \]
    Therefore,
    \[
    \ind\!\left\{
    \frac{\ips{\ub(\wb')}{\sbb_{t-1}^j}}{\norm{\ub(\wb')}_{\Vb_{t-1}}}\ge \frac{2}{\gamma}
    \right\}
    \le
    \ind\!\left\{
    \frac{\ips{\ub}{\sbb_{t-1}^j}}{\norm{\ub}_{\Vb_{t-1}}}\ge \frac{1}{\gamma}
    \right\}.
    \]
    Summing over \(j\in[m]\) and dividing by \(m\), we obtain
    \[
    E_{t,m}(\ub,1/\gamma)\ge E_{t,m}(\ub(\wb'),2/\gamma).
    \]
    Since \(\mathcal E_{\mathrm{net}}\) guarantees that \(E_{t,m}(\ub(\wb'),2/\gamma)\ge 1/10\), it follows that
    \[
    E_{t,m}(\ub,1/\gamma)\ge \frac{1}{10}.
    \]
    As \(t\in[T]\) and \(\ub\in\usd\) were arbitrary, we conclude that
    \[
    \forall t\in[T],\ \forall \ub\in\usd:
    \qquad
    E_{t,m}(\ub,1/\gamma)\ge \frac{1}{10}.
    \]
    Since this holds on \(\mathcal E_{\mathrm{net}}\cap \mathcal E_{\mathrm{state}}\), whose probability is at least \(1-\delta\), this concludes the proof.    
\end{proof}

Now, we show that the sufficient condition in Eq.~\eqref{eq:ER-condition} can be simplified as follows:
\SufficientEnsembleSize*
\begin{proof}[Proof of Corollary~\ref{cor:explicit-ensemble-size}]
    Let
    \[
    z:=\log\frac{\bar R_\alpha}{\delta}.
    \]
    Since \(0<\alpha\le1\), we have \(\sqrt{1+\alpha}\le \sqrt2\), \(\log(1+\alpha)\le \log2\), which implies
    \[
    \kappa_\alpha=\left\lceil250\log(1+\alpha)\right\rceil
    \le
    \left\lceil250\log2\right\rceil.
    \]
    Thus all explicit \(\alpha\)-dependent prefactors in \(\kappa_\alpha\) and \(N_\delta\) are bounded by universal constants.
    
    Recall that
    \[
    N_\delta
    =
    \left(6\sqrt{1+\alpha}\,\gamma\Gamma_\delta\right)^d.
    \]
    Moreover, by the definition of \(\Gamma_\delta\),
    \[
    \Gamma_\delta
    =
    \sqrt d
    +
    \sqrt{2\log\!\left(\frac{4m\bar R_\alpha}{\delta}\right)}
    +
    \sqrt{d\log(1+\alpha)+2\log\!\left(\frac{4m\bar R_\alpha}{\delta}\right)}.
    \]
    Using \(0<\alpha\le1\), we get
    \[
    \Gamma_\delta
    \le
    C_1\sqrt{d+\log m+z}
    \]
    for a universal constant \(C_1>0\). Therefore,
    \begin{equation*} 
        \log N_\delta
        =
        d\log\!\left(6\sqrt{1+\alpha}\,\gamma\Gamma_\delta\right)
        \le
        C_2 d\log\!\left(d+\log m+z\right)
    \end{equation*}
    for another universal constant \(C_2>0\).
    
    Substituting this bound into Eq.~\eqref{eq:ER-condition}, we see that it is enough to ensure
    \begin{equation} \label{eq:m-cond-eq-1}
        m
        \ge
        C_3\left[
        z+d\log\!\left(d+\log m+z\right)
        \right]
    \end{equation}
    for a universal constant \(C_3>0\).
    Now define
    \begin{equation*}
        M:=d\log(d+z)+z.
    \end{equation*}
    We claim that choosing \(m=\lceil C_4 M\rceil\) with a sufficiently large universal constant \(C_4>0\) satisfies the inequality in Eq.~\eqref{eq:m-cond-eq-1}.
    Suppose that $m=\lceil C_4 M\rceil$.
    Indeed, since
    \begin{equation*}
        M =
        d\log(d+z)+z
        \le
        d(d+z)+z
        \le
       (d+z)^2 \, .
    \end{equation*}
    Then, we obtain
    \begin{equation*}
        \log m
        \le
        C_5+\log M
        \le
        C_6\log(d+z)
    \end{equation*}
    for universal constants \(C_5,C_6>0\). Hence
    \begin{equation*}
        d+\log m+z
        \le
        d+z+ C_6 \log(d+z) 
        \le 
        C_7(d+z)
    \end{equation*}
    for a universal constant \(C_7>0\), and therefore
    \begin{equation*}
        d\log\!\left(d+\log m+z\right)
        \le
        d\log\!\left(C_7 (d+z)\right)
        \le
        C_8 d\log(d+z)
    \end{equation*}
    for a universal constant \(C_8>0\). It follows that
    \begin{equation*}
        C_3\left[
        z+d\log\!\left(d+\log m+z\right)
        \right]
        \le
        C_3\left[
        z + C_8 d \log(d+z)
        \right]
        \le
        C_{9}\left[d\log(d+z)+z\right]
        =
        C_{9}M \, .
    \end{equation*}
    Choosing \(C_4\ge C_{9}\) proves that \(m=\lceil C_4M\rceil\) implies \eqref{eq:ER-condition}.
\end{proof}

\subsection{Proof of Theorem~\ref{thm:regret bound for infinite arm set}}
In the previous section, we established that our epoch-refresh technique yields time-uniform exceedance control with a substantially smaller ensemble size. Accordingly, Theorem~\ref{thm:regret bound for infinite arm set} follows by replacing the uniform exceedance control component in the proof of~\citet{janz2026sharp} with our Lemma~\ref{lem:uniform-exceedance-across-epochs}. For completeness, we provide the full proof below.

\RegretBoundForInfiniteArmSet*

\begin{proof}[Proof of Theorem~\ref{thm:regret bound for infinite arm set}]
    We begin by recalling several standard lemmas used in the regret analysis of linear bandit algorithms. The first is the well-known concentration result for the ridge estimator.    
    \begin{lemma}[Theorem 2 of~\citet{abbasi2011improved}] \label{aux_lem:confidence ellipsoid}
        With probability at least $1 - \delta$, for all $t \ge 0$, $\thetab^*$ lies in the set
        \begin{equation*}
            \left\{ \thetab \in \RR^d : \| \hat{\thetab}_t - \thetab \|_{\Vb_t} \le \sqrt{d \log \left(1 + \frac{t}{d \lambda}\right) + 2 \log \frac{1}{\delta}} + \sqrt{\lambda}S \right\} \, .  
        \end{equation*}
    \end{lemma}
    
    Next, we recall the general regret analysis framework of~\citet{janz2024ensemble} for linear bandit algorithms that select arms greedily with respect to an estimated parameter.
    Although the original result is stated for a fixed arm set, its proof only requires that the available arm set at each round be measurable with respect to the history before the arm is selected. Therefore, the same argument applies to time-varying arm sets under this measurability condition. 
    
    \begin{lemma}[Adapted from Theorem 2 of~\citet{janz2024ensemble}] \label{aux_lem:master regret}
        Fix $\delta \in (0,1]$, $\lambda \ge 1$, and $T \in \NN$.
        Let $\{\Fcal_t\}_{t\ge 0}$ be a filtration, and let $\{\Acal_t\}_{t\ge 0}$ be a filtration on the same probability space such that
        \begin{equation*}
            \Fcal_{t-1} \subset \Acal_{t-1} \subset \Fcal_t
            \qquad \text{for all } t \ge 1 .
        \end{equation*}
        Assume that, for each $t \ge 1$, the sigma-field $\Fcal_t$ makes $\Xcal_1, \xb_1, y_1, \ldots, \Xcal_t, \xb_t, y_t$ as well as all learner randomizations used up to round $t$ measurable.
    
        Let $\{\Xcal_t\}_{t \ge 1}$ be a sequence of closed subsets of the unit ball $\Bcal_2^d$, each $\Acal_{t-1}$-measurable.
        Let $\{\thetab_t\}_{t\ge 1}$ be an $\Fcal_t$-adapted $\RR^d$-valued sequence, and suppose that
        \begin{equation*}
            \forall t \ge 1,
            \qquad
            \xb_t \in \argmax_{\xb \in \Xcal_t} \ips{\xb}{\thetab_t} \, .
        \end{equation*}
        For each $t \ge 1$, let
        \begin{equation*}
            \Vb_{t-1}
            :=
            \lambda \Ib_d + \sum_{s=1}^{t-1} \xb_s \xb_s^\top ,
        \end{equation*}
        and let $\hat{\thetab}_{t-1}$ be the usual ridge regression estimate of $\thetab^*$, which we assume to be $\Fcal_{t-1}$-measurable.
    
        Let $\{b_{t-1}\}_{t\ge 1}$ be a nonnegative $\Fcal_t$-predictable sequence such that:
        \begin{enumerate}
            \item the event $\mathcal E=\{\forall t\le T,\ \|\thetab_t-\hat{\thetab}_{t-1}\|_{\Vb_{t-1}}\le b_{t-1}\}$
            satisfies \(\mathbb P(\mathcal E^\complement)\le \delta\); and
            \item the event $\mathcal E_*=\{\forall t\le T,\ \|\thetab^*-\hat{\thetab}_{t-1}\|_{\Vb_{t-1}}\le b_{t-1}\}$
            satisfies \(\mathbb P(\mathcal E_*^\complement)\le \delta\) \, .
        \end{enumerate}
    
        For each $t \ge 1$, let
        \begin{equation*}
            \xb_t^*
            \in
            \argmax_{\xb \in \Xcal_t} \ips{\xb}{\thetab^*}
        \end{equation*}
        be an $\Acal_{t-1}$-measurable choice, and define the $\Acal_{t-1}$-conditional probability of optimism by
        \begin{equation*}
            p_{t-1}
            :=
            \PP\!\left(
                \ips{\xb_t}{\thetab_t}
                \ge
                \ips{\xb_t^*}{\thetab^*}
                \,\middle|\,
                \Acal_{t-1}
            \right) \, .
        \end{equation*}
        
        Then, on an event $\bar{\Ecal} \subset \Ecal \cap \Ecal_*$ with $\PP(\bar{\Ecal}) \ge 1 - 3\delta$,
        for all $\tau \in [T]$, the $\tau$-step regret
        \begin{equation*}
            \regret_\tau
            :=
            \sum_{t=1}^\tau \ips{\xb_t^* - \xb_t}{\thetab^*}
        \end{equation*}
        satisfies
        \begin{equation*}
            \regret_\tau
            \le
            2\max_{i\in[\tau]}\frac{b_{i-1}}{p_{i-1}}
            \left(
                2\sqrt{2d\tau\log\!\left(1+\frac{\tau}{d\lambda}\right)}
                +
                \sqrt{
                    2(4\tau/\lambda+1)
                    \log\!\left(
                        \frac{\sqrt{4\tau/\lambda+1}}{\delta}
                    \right)
                }
            \right) \, .
        \end{equation*}
    \end{lemma}
    Note that Lemma~\ref{aux_lem:master regret} follows by a direct extension of Theorem~2 of~\citet{janz2024ensemble}. Indeed, suppose that each $\Xcal_t \subset \Bcal_2^d$ is nonempty, compact, and measurable before the arm selection step at round $t$. Then the fixed support function
    $J(\thetab)=\max_{\xb\in\Xcal}\ips{\xb}{\thetab}$ used in~\citep{janz2024ensemble} can be replaced by the time-dependent support function $J_t(\thetab)=\max_{\xb\in\Xcal_t}\ips{\xb}{\thetab}$, and the optimistic region can be replaced by $\Theta^{\mathrm{OPT}}_t = \{\thetab : J_t(\thetab)\ge J_t(\thetab^*)\}$.
    The proof is otherwise unchanged, since it only uses, round by round, the fact that $\xb_t \in \argmax_{\xb\in\Xcal_t}\ips{\xb}{\thetab_t}$, which implies $\xb_t\in\partial J_t(\thetab_t)$, together with the confidence events $\thetab_t,\thetab^*\in\Theta_{t-1}$ and the elliptical potential bound for the selected arms. Hence, the same regret bound holds for the time-varying regret 
    $\displaystyle \regret_\tau = \sum_{s=1}^{\tau} \left( \max_{\xb\in\Xcal_s}\ips{\xb}{\thetab^*}
            -
            \ips{\xb_s}{\thetab^*}
        \right)
    $,
    with $p_{t-1}$ defined using $\Theta^{\mathrm{OPT}}_t$ instead of a fixed $\Theta^{\mathrm{OPT}}$.

    In the remaining part of the proof, we first set up the notation and verify the conditions required to apply Lemma~\ref{aux_lem:master regret}.
    For \(t\in[T]\), define
    \begin{equation*}
        b_{t-1} := \gamma \Gamma_\delta \beta_{t-1} \, .    
    \end{equation*}
    Also, we define
    \begin{equation*}
        \mathcal E
        :=
        \left\{
        \forall t\le T:\ \norm{\thetab_t-\hat{\thetab}_{t-1}}_{\Vb_{t-1}}\le b_{t-1}
        \right\} \, .    
    \end{equation*}

    Since
    \[
    \thetab_t=\hat{\thetab}_{t-1}+\gamma\beta_{t-1}\Vb_{t-1}^{-1}\sbb_{t-1}^{J_t},
    \]
    we have
    \[
    \norm{\thetab_t-\hat{\thetab}_{t-1}}_{\Vb_{t-1}}
    =
    \gamma\beta_{t-1}\norm{\Vb_{t-1}^{-1/2}\sbb_{t-1}^{J_t}}_2.
    \]
    By Lemma~\ref{lem:uniform-state-bound}, with probability at least \(1-\delta\),
    \[
    \forall t\le T,\ \forall j\in[m]:
    \qquad
    \norm{\Vb_{t-1}^{-1/2}\sbb_{t-1}^j}_2\le \Gamma_\delta.
    \]
    Hence \(\PP(\mathcal E)\ge 1-\delta\).
    
    Next, let
    \[
    \mathcal E_*
    :=
    \left\{
    \forall t\le T:\ \norm{\thetab^*-\hat{\thetab}_{t-1}}_{\Vb_{t-1}}\le b_{t-1}
    \right\}.
    \]
    Since $\Gamma_\delta \ge \sqrt{d}$, we have $\gamma \Gamma_{\delta} \ge 1$, we have $b_{t-1} \ge \beta_{t-1}$.
    By Lemma~\ref{aux_lem:confidence ellipsoid} and our choice of \(\beta_t\), we have $\PP(\mathcal E^*)\ge 1-\delta$.
    
    We now define filtrations adapted to the epoch-refresh construction. For each epoch-start time \(\tau_r\), let \(\gb_{\tau_r}^1,\ldots,\gb_{\tau_r}^m\) denote the refreshed Gaussian vectors sampled at time \(\tau_r\). 
    Define
    \begin{equation} \label{eq:filtration F}
        \Fcal_t :=
        \sigma( \{\gb_{\tau_r}^1,\ldots,\gb_{\tau_r}^m\} : \tau_r \le t)
        \vee
        \sigma( \{\Xcal_s, \xb_s,y_s,J_s\}: s \le t)
        \vee
        \sigma(\{\xi_s^1,\ldots,\xi_s^m\} : s < t \bigr),    
    \end{equation}
    and
    \begin{equation} \label{eq:filtration A}
        \Acal_t := \sigma( \{\sbb_t^1, \ldots, \sbb_t^m \}) \vee \sigma(\Xcal_{t+1}) \vee \Fcal_t
    \end{equation}
    where \(\vee\) denotes the join of \(\sigma\)-algebras. By construction, \(\hat{\thetab}_t\), \(\Vb_t\), and \(\beta_t\) are \(\Fcal_t\)-measurable, while \(\{\sbb_t^j\}_{j=1}^m\) are \(\Acal_t\)-measurable. Moreover, conditional on \(\Acal_{t-1}\), the index \(J_t\) is uniform on \([m]\).
        
    For each \(t\in[T]\), define the \(\Acal_{t-1}\)-conditional probability of optimism by
    \[
    p_{t-1}
    :=
    \PP\!\left(
    \ip{\xb_t}{\thetab_t}\ge \ip{\xb_t^*}{\thetab^*}
    \ \middle|\ \Acal_{t-1}
    \right),
    \]
    where
    \[
    \xb_t^* \in \argmax_{\xb\in\Xcal_t}\ip{\xb}{\thetab^*}.
    \]
    
    Fix \(t\in[T]\). If \(\xb_t^*=\zero_d\), then trivially
    \[
    \ip{\xb_t}{\thetab_t}
    =
    \max_{\xb\in\Xcal_t}\ip{\xb}{\thetab_t}
    \ge 0
    =
    \ip{\xb^*_t}{\thetab^*},
    \]
    and hence \(p_{t-1}=1\).
    
    Assume henceforth that \(\xb_t^*\neq \zero_d\), and define
    \[
    \ub_t^*:=\Vb_{t-1}^{-1}\xb_t^*.
    \]
    Then \(\norm{\ub_t^*}_{\Vb_{t-1}}=\norm{\xb_t^*}_{\Vb_{t-1}^{-1}}\). 
    On the other hand, by Lemma~\ref{aux_lem:confidence ellipsoid}, we have
    \[
    \ips{\thetab^*-\hat{\thetab}_{t-1}}{\xb^*_t}
    \le
    \norm{\thetab^*-\hat{\thetab}_{t-1}}_{\Vb_{t-1}}
    \norm{\xb^*_t}_{\Vb_{t-1}^{-1}}
    \le
    \beta_{t-1} \norm{\xb^*_t}_{\Vb_{t-1}^{-1}}.
    \]
    Therefore, if
    \[
    \ips{\thetab_t-\hat{\thetab}_{t-1}}{\xb^*_t}
    \ge
    \beta_{t-1}\norm{\xb^*_t}_{\Vb_{t-1}^{-1}},
    \]
    then
    \begin{align*}
        \ips{\thetab_t}{\xb^*_t}
        & = 
        \ips{\hat{\thetab}_{t-1}}{\xb^*_t}
        +
        \ips{\thetab_t-\hat{\thetab}_{t-1}}{\xb^*_t}
        \\
        & \ge
        \ips{\hat{\thetab}_{t-1}}{\xb^*_t}
        + 
        \beta_{t-1}\norm{\xb^*_t}_{\Vb_{t-1}^{-1}}
        \\
        & \ge
        \ips{\hat{\thetab}_{t-1}}{\xb^*_t}
        + 
        \ips{\thetab^*-\hat{\thetab}_{t-1}}{\xb^*_t} 
        = 
        \ips{\thetab^*}{\xb^*_t} \, . 
    \end{align*}
    Since \(\xb_t\in\argmax_{\xb\in\Xcal_t}\ip{\xb}{\thetab_t}\), this implies
    \[
    \ips{\xb_t}{\thetab_t}\ge \ips{\xb^*_t}{\thetab^*}.
    \]
    
    Now,
    \[
    \ips{\thetab_t-\hat{\thetab}_{t-1}}{\xb^*_t}
    =
    \gamma\beta_{t-1}\ip{\Vb_{t-1}^{-1}\sbb_{t-1}^{J_t}}{\xb^*_t}
    =
    \gamma\beta_{t-1}\ip{\sbb_{t-1}^{J_t}}{\Vb_{t-1}^{-1}\xb^*_t}
    =
    \gamma\beta_{t-1}\ips{\sbb_{t-1}^{J_t}}{\ub^*_t} \, .
    \]
    Hence the sufficient condition
    \[
        \frac{\ips{\sbb_{t-1}^{J_t}}{\ub^*_t}}{\norm{\ub_t^*}_{\Vb_{t-1}}}
        \ge \frac{1}{\gamma}
        \]
        implies
        \[
        \ips{\thetab_t-\hat{\thetab}_{t-1}}{\xb^*_t}
        \ge
        \beta_{t-1}\norm{\ub^*_t}_{\Vb_{t-1}}
        =
        \beta_{t-1}\norm{\xb^*_t}_{\Vb_{t-1}^{-1}},
    \]
    and therefore implies optimism:
    \[
    \frac{\ips{\sbb_{t-1}^{J_t}}{\ub^*_t}}{\norm{\ub^*_t}_{\Vb_{t-1}}}
    \ge \frac{1}{\gamma}
    \quad\Longrightarrow\quad
    \ips{\xb_t}{\thetab_t}\ge \ips{\xb^*_t}{\thetab^*} \, .
    \]

    Since \(\xb_t^*\neq \zero_d\), we have \(\ub_t^*\neq \zero_d\). Define
    $
        \bar{\ub}_t^*
        :=
        \frac{\ub_t^*}{\|\ub_t^*\|_2}
        \in \mathbb S^{d-1} \, .
    $
    Moreover, since the self-normalized ratio is invariant under positive rescaling of
    \(\ub_t^*\), we have
    \[
        \frac{\ips{\sbb_{t-1}^{J_t}}{\ub_t^*}}
        {\|\ub_t^*\|_{\Vb_{t-1}}}
        =
        \frac{\ips{\sbb_{t-1}^{J_t}}{\bar{\ub}_t^*}}
        {\|\bar{\ub}_t^*\|_{\Vb_{t-1}}} \, .
    \]
    Therefore, using the conditional uniformity of \(J_t\) given \(\Acal_{t-1}\),
    \begin{align*}
        p_{t-1}
        &:=
        \PP\left(
            \ips{\xb_t}{\thetab_t}
            \ge
            \ips{\xb_t^*}{\thetab^*}
            \,\middle|\,\Acal_{t-1}
        \right)
        \\
        &\ge
        \PP\left(
            \ips{\sbb_{t-1}^{J_t}}{\ub_t^*}/\|\ub_t^*\|_{\Vb_{t-1}}
            \ge
            1/\gamma
            \,\middle|\,\Acal_{t-1}
        \right)
        \\
        & =
        \frac{1}{m}
        \sum_{j=1}^m
        \ind\left\{
            \frac{\ips{\sbb_{t-1}^{j}}{\bar{\ub}_t^*}}
            {\|\bar{\ub}_t^*\|_{\Vb_{t-1}}}
            \ge
            \frac{1}{\gamma}
        \right\}
        \\
        &=
        E_{t,m}\left(\bar{\ub}_t^*,1/\gamma\right) \, .
    \end{align*}

    By Lemma~\ref{lem:uniform-exceedance-across-epochs} and Corollary~\ref{cor:explicit-ensemble-size}, our choice of the ensemble-size condition gives, with probability at least \(1-\delta\), 
    \begin{equation*}
            \displaystyle \min_{t\in[T]}\ \inf_{\ub\in \usd} E_{t,m}(\ub,1/\gamma)\ge \frac{1}{10} \, ,
    \end{equation*}
    which implies    
    \[
    \forall t\in[T]:
    \qquad
    p_{t-1}\ge \frac{1}{10} \, .
    \]
    
    We have therefore verified the key ingredients for Lemma~\ref{aux_lem:master regret}:
    the sampled parameter lies in a \(b_{t-1}\)-ellipsoid around \(\hat{\thetab}_{t-1}\)
    on an event of probability at least \(1-\delta\), the true parameter lies in a
    \(b_{t-1}\)-ellipsoid around \(\hat{\thetab}_{t-1}\) on an event of probability at least
    \(1-\delta\), and, by Lemma~\ref{lem:uniform-exceedance-across-epochs} together with
    Corollary~\ref{cor:explicit-ensemble-size}, the conditional optimism probabilities satisfy
    \(p_{t-1}\ge 1/10\) uniformly over \(t\) on an event of probability at least \(1-\delta\).
    Applying Lemma~\ref{aux_lem:master regret} and intersecting with the uniform optimism event, the following final regret bound holds with probability at least \(1-4\delta\):
    \begin{align*}
        \regret_T
        &\le
        2\max_{t\in[T]}\frac{b_{t-1}}{p_{t-1}}
        \left(
        2\sqrt{2dT\log\!\left(1+\frac{T}{d\lambda}\right)}
        +
        \sqrt{
        2(4T/\lambda+1)
        \log\!\left(\frac{\sqrt{4T/\lambda+1}}{\delta}\right)
        }
        \right) \\
        &\le
        \tilde{\Ocal}\!\left(b_T\sqrt{dT}\right).
    \end{align*}
    Finally, since \( b_T=\gamma\Gamma_\delta\beta_T\), \(\Gamma_\delta=\tilde{\Ocal}(\sqrt d)\), and \(\beta_T=\tilde{\Ocal}(\sqrt d)\), we conclude that
    \[
    \regret_T=\tilde{\Ocal}(d^{3/2}\sqrt T).
    \]
\end{proof}

\section{Regret Analysis for Finite-Arm-Dominated Regime} \label{appx:proof of finite arm set}
In this section, we provide the regret analysis of our algorithm for finite-arm-dominated regime.
First, we revisit the regret analyses of previous randomized linear bandit algorithms in the finite-arm-dominated regime. Let $\tilde{\thetab}_t$ denote the perturbed estimator used for exploration at round $t$.
In TS~\citep{abeille2017Linear}, $\tilde{\thetab}_t$ is given by
\begin{equation*}
    \tilde{\thetab}_t
    =
    \hat{\thetab}_{t-1}
    +
    \beta_{t-1}\Vb_{t-1}^{-1/2}\gb_t,
    \qquad
    \gb_t \sim \Ncal(\zero_d, \Ib_d).
\end{equation*}
Then, for each candidate arm $\xb \in \Xcal_t$ at round $t$, we have
\begin{equation*}
    \xb^\top (\tilde{\thetab}_t - \hat{\thetab}_{t-1})
    =
    \beta_{t-1}\xb^\top \Vb_{t-1}^{-1/2} \gb_t
    \sim
    \Ncal\!\left(0, \beta_{t-1}^2 \| \xb \|_{\Vb_{t-1}^{-1}}^2 \right).
\end{equation*}
Therefore, for each fixed arm $\xb$, the quantity
$
|\xb^\top (\tilde{\thetab}_t - \hat{\thetab}_{t-1})|
$
can be controlled by a Gaussian tail bound at the scale
$
\beta_{t-1}\|\xb\|_{\Vb_{t-1}^{-1}}.
$
On the other hand, in PHE~\citep{kveton2020perturbed}, $\tilde{\thetab}_t$ is given by
\begin{equation*}
    \tilde{\thetab}_t
    :=
    \Vb_{t-1}^{-1}
    \left[
        \sum_{s=1}^{t-1} \xb_s (y_s + Z_{t,s})
    \right],
    \qquad
    \{Z_{t,s}\}_{s=1}^{t-1}
    \overset{\textit{i.i.d.}}{\sim}
    \Ncal(0,1),
\end{equation*}
so that
\begin{equation*}
    \xb^\top (\tilde{\thetab}_t - \hat{\thetab}_{t-1})
    =
    \xb^\top \Vb_{t-1}^{-1}
    \left[
        \sum_{s=1}^{t-1} \xb_s Z_{t,s}
    \right].
\end{equation*}
Since the perturbations $\{Z_{t,s}\}_{s=1}^{t-1}$ are sampled freshly at round $t$, conditional on the history up to round $t-1$, the above quantity is again a centered Gaussian random variable with variance of order $\|\xb\|_{\Vb_{t-1}^{-1}}^2$. Hence, for each candidate arm $\xb \in \Xcal_t$, one obtains an arm-wise concentration bound at the same self-normalized scale.

The key point is that, in both TS and PHE, the randomized score of each arm admits a clean \emph{arm-wise} concentration bound conditional on the past. More precisely, for every $\xb \in \Xcal_t$,
\[
    |\xb^\top (\tilde{\thetab}_t - \hat{\thetab}_{t-1})|
    \lesssim
    \beta_{t-1}\|\xb\|_{\Vb_{t-1}^{-1}}
\]
holds with high probability for a fixed arm, and thus, by taking a union bound over the $K$ candidate arms, one obtains
\[
    \max_{\xb \in \Xcal_t}
    |\xb^\top (\tilde{\thetab}_t - \hat{\thetab}_{t-1})|
    \lesssim
    \beta_{t-1}\sqrt{\log K}\,
    \max_{\xb \in \Xcal_t}\|\xb\|_{\Vb_{t-1}^{-1}}.
\]
This is precisely the mechanism by which the finite-arm assumption yields a $\sqrt{\log K}$ factor instead of a $\sqrt d$ factor in the regret analysis. In other words, the $\log K$ improvement becomes available because the perturbation can be controlled separately for each of the $K$ arms, and the only price for uniformity over the current arm set is a union bound over $K$ events.

By contrast, this argument does not directly apply to ensemble sampling. In ensemble sampling, \textit{the perturbation used at round $t$ is not freshly generated from scratch, but is accumulated along the realized arm sequence.} More specifically, for each ensemble member $j$, the perturbation takes the form
\begin{equation*}
    \sbb_{t-1}^j
    =
    \sqrt{\lambda}\,\zetab^j
    +
    \sum_{s=1}^{t-1} \xb_s \xi_s^j \, ,
\end{equation*}
and the perturbed estimator is given by
\begin{equation*}
    \tilde{\thetab}_t^j
    =
    \hat{\thetab}_{t-1}
    +
    \gamma\beta_{t-1}\Vb_{t-1}^{-1}\sbb_{t-1}^j \, .
\end{equation*}
The key difficulty is that the arm sequence $\{\xb_s\}_{s=1}^{t-1}$ is itself selected using the ensemble, and therefore depends on the same perturbations that define $\sbb_{t-1}^j$. As a result, after conditioning on the history up to round $t-1$, the quantity
\begin{equation*}
    \xb^\top (\tilde{\thetab}_t^j - \hat{\thetab}_{t-1})
    =
    \gamma \beta_{t-1}\xb^\top \Vb_{t-1}^{-1}\sbb_{t-1}^j
\end{equation*}
can no longer be viewed as a centered Gaussian random variable with variance proportional to $\|\xb\|_{\Vb_{t-1}^{-1}}^2$. In other words, the perturbation and the current Gram matrix are coupled through the past trajectory, and thus the simple arm-wise conditional concentration argument used in TS and PHE breaks down.

This distinction is crucial. In TS and PHE, the perturbation at round $t$ remains conditionally well-behaved for each arm after fixing the past, so a union bound over the $K$ candidate arms yields a $\sqrt{\log K}$ factor. In ensemble sampling, however, the perturbation is persistent and evolves together with the action sequence that it helps generate. Therefore, the projected perturbation along a given arm does not admit the same clean arm-wise tail bound after conditioning on the past. Consequently, one cannot directly obtain
\[
    \max_{\xb\in\Xcal}
    |\xb^\top(\tilde{\thetab}_t^j-\hat{\thetab}_{t-1})|
    \lesssim
    \beta_{t-1}\sqrt{\log K}\,
    \max_{\xb\in\Xcal}\|\xb\|_{\Vb_{t-1}^{-1}}
\]
by the same argument as in TS or PHE.

Instead, previous analyses of ensemble sampling control the perturbation in a more global manner, for example through bounds on self-normalized directional processes or on the vector norm $\|\Vb_{t-1}^{-1/2}\sbb_{t-1}^j\|$. Such bounds are sufficient to establish optimism and sublinear regret, but they naturally lead to a $\sqrt d$-scale complexity rather than a $\sqrt{\log K}$ one. This is the main reason why the standard finite-arm $\log K$ improvement available in TS and PHE does not automatically carry over to ensemble sampling.

However, our epoch-refresh construction makes an arm-wise analysis possible. Within each epoch, the refresh rule ensures that the Gram matrix remains close to its epoch-start value, so the self-normalized perturbation score of each fixed arm changes only mildly during the epoch. This allows us to control the perturbations directly over the finite set of \(K\) arms, rather than uniformly over all directions. Lemma~\ref{lem:uniform arm-wise concentration} formalizes this idea by showing that the ensemble perturbations admit a uniform arm-wise concentration bound at essentially the \(\sqrt{\log K}\) scale, which is the key ingredient for the regret analysis for the finite-arm-dominated regime.

\subsection{Proof of Lemma~\ref{lem:uniform arm-wise concentration}}
\UniformArmWiseConcentration*
\begin{proof}[Proof of Lemma~\ref{lem:uniform arm-wise concentration}]
    Fix an epoch \(r\) and a round \(t\in\{\tau_r+1,\dots,\tau_{r+1}\}\) played in epoch \(r\). Recall that by construction,
    \[
        \sbb_{t-1}^j
        =
        \Vb_{\tau_r}^{1/2}\gb_r^j
        +
        \sum_{s=\tau_r+1}^{t-1}\xb_s\xi_s^j \, .
    \]
    Fix a candidate arm \(\xb \in \Xcal_t\) at round \(t\). The case \(\xb=\zero_d\) is trivial, so we assume \(\xb\ne\zero_d\). We define
    \[
        \Ab_{r,t}
        :=
        \Vb_{\tau_r}^{-1/2}\Vb_{t-1}\Vb_{\tau_r}^{-1/2},
        \qquad
        \yb_{r,t}(\xb)
        :=
        \Vb_{\tau_r}^{-1/2}\xb,
        \qquad 
        \nbb_{r,t-1}^j
        :=
        \Vb_{\tau_r}^{-1/2}
        \sum_{s=\tau_r+1}^{t-1}\xb_s\xi_s^j.
    \]
    Then we have
    \[
        \Vb_{\tau_r}^{-1/2}\sbb_{t-1}^j
        =
        \gb_r^j+\nbb_{r,t-1}^j \, .
    \]
    Moreover, since \(\Vb_{t-1}^{-1} = \Vb_{\tau_r}^{-1/2}\Ab_{r,t}^{-1}\Vb_{\tau_r}^{-1/2}\), for \(\xb \neq \zero_d\), we have
    \[
        Z^{j}_{r, t}(\xb)
        :=
        \frac{
        \xb^\top \Vb_{t-1}^{-1}\sbb_{t-1}^j
        }{
        \|\xb\|_{\Vb_{t-1}^{-1}}
        }
        =
        \frac{
        \yb_{r,t}(\xb)^\top \Ab_{r,t}^{-1}(\gb_r^j+\nbb_{r,t-1}^j)
        }{
        \sqrt{\yb_{r,t}(\xb)^\top \Ab_{r,t}^{-1}\yb_{r,t}(\xb)}
        } \, .
    \]
    In the following proof, we establish a uniform bound on \(|Z_{r,t}^{j}(\xb)|\) over all \(t\in[T]\), \(j\in[m]\), and \(\xb \in \Xcal_t\).
    
    Note that since \(t\) is played in epoch \(r\), the refresh rule guarantees
    \[
    \Ib_d\preceq \Ab_{r,t}\preceq (1+\alpha)\Ib_d \, .
    \]
    
    Next, we introduce a lemma which quantifies how little the self-normalized projection changes when the Gram matrix drifts only mildly within an epoch. In particular, when \(\Ab\) remains within a multiplicative factor \(1+\alpha\) of the identity, the normalized linear functional induced by \(\Ab^{-1}\) stays uniformly close to the corresponding Euclidean projection.
    A proof of Lemma~\ref{lem:matrix-drift-perturbation} is provided in Appendix~\ref{appx:proof of lem:matrix-drift-perturbation}. 

    \begin{restatable}[Stability of normalized projections under small matrix drift]{lemma}{MatrixDriftPerturbation}
    \label{lem:matrix-drift-perturbation}
        Let \(\alpha \ge0\), and suppose
        \[
        \Ib_d\preceq \Ab\preceq (1+\alpha)\Ib_d \, .
        \]
        Then, for every \(\yb\neq \zero_d\) and every \(\vb\in\RR^d\),
        \[
        \left|
        \frac{\yb^\top \Ab^{-1}\vb}{\sqrt{\yb^\top \Ab^{-1}\yb}}
        -
        \ip{\frac{\yb}{\norm{\yb}_2}}{\vb}
        \right|
        \le
        2\alpha \norm{\vb}_2 \, .
        \]   
    \end{restatable}
    
    Applying Lemma~\ref{lem:matrix-drift-perturbation} with
    \[
        \yb=\yb_{r,t}(\xb),
        \qquad
        \Ab=\Ab_{r,t},
        \qquad
        \vb=\gb_r^j+\nbb_{r,t-1}^j,
    \]
    we obtain
    \begin{equation} \label{eq:Z-decomposition}
        |Z_{r,t}^{j}(\xb)|
        \le
        \left|
        \ip{\frac{\yb_{r,t}(\xb)}{\norm{\yb_{r,t}(\xb)}_2}}{\gb_r^j}
        \right|
        +
        \left|
        \ip{\frac{\yb_{r,t}(\xb)}{\norm{\yb_{r,t}(\xb)}_2}}{\nbb_{r,t-1}^j}
        \right| 
        +
        2\alpha\norm{\gb_r^j}_2
        +
        2\alpha\norm{\nbb_{r,t-1}^j}_2 \, .    
    \end{equation}
    We now control the four terms on the right-hand side of Eq.\eqref{eq:Z-decomposition}.

    \paragraph{Term 1: epoch-start Gaussian projection.}
    Define
    \[
        \ub_{r,t}(\xb)
        :=
        \frac{\yb_{r,t}(\xb)}{\norm{\yb_{r,t}(\xb)}_2}.
    \]
    We first control the epoch-start Gaussian projection. 
    By the oblivious-arm-set assumption, the sequence \(\{\Xcal_t\}_{t=1}^T\) is generated independently of the algorithm's internal randomization.
    Moreover, the epoch-start Gaussian perturbation \(\gb_r^j\) is sampled independently of the history available at the beginning of epoch \(r\). Hence, for each nonzero \(\xb\in\Xcal_t\), the direction \(\ub_{r,t}(\xb)\) is independent of \(\gb_r^j\).
    Since \(\|\ub_{r,t}(\xb)\|_2=1\), it follows that 
    \begin{equation*}
        \ips{\ub_{r,t}(\xb)}{\gb_r^j}\sim \Ncal(0,1) \, .    
    \end{equation*}
    Thus, for every \(z>0\),
    \[
        \PP\!\left(
        \left|\ips{\ub_{r,t}(\xb)}{\gb_{r}^j}\right|\ge z
        \right)
        \le
        2\exp(-z^2/2) \, .
    \]
    Taking a union bound over epochs \(r\le R_\alpha\), rounds \(t\in[T]\),
    ensemble members \(j\in[m]\), and arms \(\xb\in\Xcal_t\), and using
    \(|\Xcal_t|\le K\) and \(R_\alpha\le \bar R_\alpha\), we obtain
    \[
        \PP\!\left(
        \max_{r\le R_\alpha,\;t\in[T],\;j\in[m],\;\xb\in\Xcal_t}
        \left|\ips{\ub_{r,t}(\xb)}{\gb_r^j}\right|
        \ge z
        \right)
        \le
        2KmT\bar R_\alpha \exp(-z^2/2) \, .
    \]
    Therefore, taking
    \[
        z
        =
        \sqrt{
            2\log\!\left(\frac{8KmT\bar R_\alpha}{\delta}\right)
        }
    \]
    gives, with probability at least \(1-\delta/4\),
    \begin{equation} \tag{E1} \label{eq:Z-deom-1}
        \max_{t\in[T],\; j\in[m],\; \xb\in\Xcal_t}
        \left|\ips{\ub_{r,t}(\xb)}{\gb_{r(t)}^j}\right|
        \le
        \sqrt{
            2\log\!\left(\frac{8KmT\bar R_\alpha}{\delta}\right) 
        }
        = \sqrt{2 L_\alpha} \, .        
    \end{equation}

    \paragraph{Term 2: scalar within-epoch martingale.}
    Fix an epoch \(r\), an ensemble index \(j\in[m]\), a round
    \(t\in\{\tau_r+1,\ldots,\tau_{r+1}\}\), and a nonzero candidate arm
    \(\xb\in\Xcal_t\). Define
    \begin{equation*}
        \Mcal_{r,\tau_r,t}^{j}(\xb):=0,
        \qquad
        \Mcal_{r,q,t}^{j}(\xb)
        :=
        \sum_{s=\tau_r+1}^{q}
        \ip{\ub_{r,t}(\xb)}
        {\Vb_{\tau_r}^{-1/2}\xb_s}\,\xi_s^j,
        \quad
        q\in\{\tau_r+1,\ldots,t-1\}.
    \end{equation*}
    Then
    \begin{equation*}
        \ips{\ub_{r,t}(\xb)}{\nbb_{r,t-1}^j}
        =
        \Mcal_{r,t-1,t}^{j}(\xb).
    \end{equation*}
    
    We now specify the filtration used in the martingale argument. Let
    \begin{equation*}
        \Gcal_t^{\Xcal}
        :=
        \sigma(\Xcal_1,\ldots,\Xcal_t)
    \end{equation*}
    be the sigma-algebra generated by the sequence of arm sets. 
    By the oblivious-arm-set assumption, \(\Gcal^{\Xcal}_t\) is independent of the algorithm's internal randomization. 
    For \(q\in\{\tau_r,\ldots,t-1\}\), define
    \begin{equation*}
        \Fcal_q^{(r,t)}
        :=
        \Gcal^{\Xcal}_t
        \vee
        \sigma(\gb_\ell^1,\ldots,\gb_\ell^m:\ell\le r)
        \vee
        \sigma(\xb_s,y_s,J_s:s\le q)
        \vee
        \sigma(\xi_s^1,\ldots,\xi_s^m:s<q) \, .
    \end{equation*}
    This filtration contains the realized arm set \(\Xcal_t\), the selected arm
    \(\xb_q\), and all perturbation noises generated before round \(q\), but not
    \(\xi_q^j\). Hence, by the oblivious-arm-set assumption and the independence of
    the perturbation noises, \(\xi_q^j\sim\Ncal(0,1)\) is independent of
    \(\Fcal_q^{(r,t)}\). 
    
    On the other hand, since $\xb_q$ is $\Fcal_q^{(r,t)}$-measurable, therefore the scalar $a_{r,q,t}(\xb) := \ip{\ub_{r,t}(\xb)}{\Vb_{\tau_r}^{-1/2}\xb_q}$ is also $\Fcal_q^{(r,t)}$-measurable. 
    
    Define the enlarged filtration
    \[
        \widehat{\Fcal}_{\tau_r}^{(r,t)}
        :=
        \Fcal_{\tau_r}^{(r,t)},
        \qquad
        \widehat{\Fcal}_{q}^{(r,t)}
        :=
        \Fcal_q^{(r,t)}
        \vee
        \sigma(\xi_q^1,\ldots,\xi_q^m),
        \quad
        q\in\{\tau_r+1,\ldots,t-1\}.
    \]
    Then
    \[
        \widehat{\Fcal}_{q-1}^{(r,t)}
        \subseteq
        \Fcal_q^{(r,t)}
        \subseteq
        \widehat{\Fcal}_q^{(r,t)}.
    \]
    Furthermore, we have
    \begin{align*}
        \Mcal_{r,q,t}^{j}(\xb) 
        & = 
        \sum_{s=\tau_r+1}^{q}
        \ip{\ub_{r,t}(\xb)}
        {\Vb_{\tau_r}^{-1/2}\xb_s}\,\xi_s^j
        \\
        & =
        \sum_{s=\tau_r+1}^{q-1}
        \ip{\ub_{r,t}(\xb)}
        {\Vb_{\tau_r}^{-1/2}\xb_s}\,\xi_s^j
        + \ip{\ub_{r,t}(\xb)}
        {\Vb_{\tau_r}^{-1/2}\xb_q}\,\xi_q^j
        \\
        &
        = \Mcal_{r,q-1,t}^{j}(\xb)
        + a_{r,q, t}(\xb)\xi_q^j \, .
    \end{align*}
    Since \(\xi_q^j\) is independent of \(\Fcal_q^{(r,t)}\) and has mean zero,
    \begin{equation*}
        \EE\!\left[
        \Mcal_{r,q,t}^{j}(\xb)
        \mid
        \Fcal_q^{(r,t)}
        \right]
        =
        \Mcal_{r,q-1,t}^{j}(\xb)
        + a_{r,q, t}(\xb) \EE\left[ \xi_q^j \mid \Fcal_q^{(r,t)}
        \right]
        = \Mcal_{r,q-1,t}^{j}(\xb) \, .
    \end{equation*}
    By the tower property,
    \[
        \EE\!\left[
        \Mcal_{r,q,t}^{j}(\xb)
        \mid
        \widehat{\Fcal}_{q-1}^{(r,t)}
        \right]
        =
        \EE\left[
        \EE\!\left[
        \Mcal_{r,q,t}^{j}(\xb)
        \mid \Fcal_q^{(r,t)}
        \right]\mid
        \widehat{\Fcal}_{q-1}^{(r,t)}
        \right]
        =
        \Mcal_{r,q-1,t}^{j}(\xb).
    \]
    Thus,
    \(\{\Mcal_{r,q,t}^{j}(\xb)\}_{q=\tau_r}^{t-1}\) is a martingale with conditionally Gaussian increments with respect to
    \(\{\widehat{\Fcal}_q^{(r,t)}\}_{q=\tau_r}^{t-1}\).
    We also define the variance process
    \[
        A_{r,\tau_r, t}(\xb):=0,
        \qquad
        A_{r,q,t}(\xb)
        :=
        \sum_{s=\tau_r+1}^{q}
        \left(a_{r,s,t}(\xb)\right)^2,
        \quad
        q\in\{\tau_r+1,\ldots,t-1\}.
    \]
    Since \(t\) belongs to epoch \(r\), for every \(q\le t-1\),
    \[
        \Vb_q\preceq (1+\alpha)\Vb_{\tau_r}.
    \]
    Hence for any $q \in \{ \tau_r+1, \ldots, t-1\}$, we obtain
    \begin{align*}
        A_{r,q,t}(\xb)
        &=
        \sum_{s=\tau_r+1}^{q}
        \left(
        \ip{\ub_{r,t}(\xb)}
        {\Vb_{\tau_r}^{-1/2}\xb_s}
        \right)^2  \\
        &=
        \ub_{r,t}(\xb)^\top
        \Vb_{\tau_r}^{-1/2}
        (\Vb_q-\Vb_{\tau_r})
        \Vb_{\tau_r}^{-1/2}
        \ub_{r,t}(\xb)
        \le \alpha \, . \numberthis \label{eq:variance-process-upper-bdd}
    \end{align*}
    Now, for a fixed $\lambda_0 > 0$, we define the exponential process
    \[
        \mathcal E_q(\lambda_0)
        :=
        \exp\!\left(
        \lambda_0\Mcal_{r,q,t}^{j}(\xb)
        -
        \frac{\lambda_0^2}{2}
        A_{r,q,t}(\xb)
        \right),
        \qquad
        q\in\{\tau_r,\ldots,t-1\} \, .
    \]
    We claim that \(\{\mathcal E_q(\lambda_0)\}_{q=\tau_r}^{t-1}\) is a nonnegative martingale with respect to $\{\widehat{\Fcal}_q^{(r,t)}\}_{q=\tau_r}^{t-1}$.
    Indeed, $\Mcal^j_{r,q,t}(\xb)$ is $\hat{\Fcal}_q^{(r,t)}$-measurable by construction.
    Moreover, since $\Mcal^j_{r,q-1,t}(\xb)$ and $a_{r,q,t}(\xb)$ are $\Fcal^{(r,t)}_q$-measurable and $\xi^j_q \sim \Ncal(0,1)$ is independent of $\Fcal^{(r,t)}_q$, we have
    \begin{equation*}
        \EE\left[ \Ecal_q(\lambda_0) \mid \Fcal^{(r,t)}_q \right]
        = \Ecal_{q-1}(\lambda_0) \EE \left[ \exp \left( \lambda_0 a_{r,q,t}(\xb) \xi^j_q - \frac{\lambda_0^2}{2}(a_{r,q,t}(\xb))^2 \right) \mid \Fcal^{(r,t)}_q \right]
        = \Ecal_{q-1}(\lambda_0)
    \end{equation*}
    where the last equality follows from the Gaussian moment generating function. 
    Then, the tower property again yields
    \begin{equation*}
        \EE\left[ \Ecal_q(\lambda_0) \mid \hat{\Fcal}^{(r,t)}_{q-1} \right]
        = \EE \left[
            \EE\left[ 
                \Ecal_q(\lambda_0) \mid \Fcal^{(r,t)}_q
                \right]
                \mid \hat{\Fcal}^{(r,t)}_{q-1}
            \right]
        = \Ecal_{q-1}(\lambda_0) \, .
    \end{equation*}
    Thus, \(\{\mathcal E_q(\lambda_0)\}_{q=\tau_r}^{t-1}\) is a nonnegative martingale.
    Now, for any $z \ge 0$, we define an event
    \begin{equation*}
        E_z := \left\{ \max_{q \in \{\tau_r+1, \ldots, t-1\}} \Mcal_{r,q,t}^{j} (\xb) \ge z \right\} \, .
    \end{equation*}
    On $E_z$, there exists $q^* \in \{\tau_r + 1, \ldots, t-1\}$ such that $\Mcal_{r,q^*,t}^j \ge z$. 
    By Eq.~\eqref{eq:variance-process-upper-bdd}, we have 
    \begin{equation*}
        A_{r, q^*, t}(\xb) \le \alpha \, .
    \end{equation*}
    Therefore, 
    \begin{equation*}
        \mathcal E_{q^*}(\lambda_0)
        =
        \exp\!\left(
        \lambda_0\Mcal_{r,q^*,t}^{j}(\xb)
        -
        \frac{\lambda_0^2}{2}
        A_{r,q^*,t}(\xb)
        \right)
        \ge \exp \left( \lambda_0 z - \frac{\lambda_0^2}{2} \alpha \right) \, .
    \end{equation*}
    Hence, we have the following relationship:
    \begin{equation*}
        E_z 
        \subset
        \left\{ \max_{q \in \{\tau_r, \ldots, t-1 \}} \Ecal_q (\lambda_0)
            \ge \exp 
                \left(
                \lambda_0 z - \frac{\lambda_0^2}{2} \alpha
                \right)
        \right\} \, .
    \end{equation*}
    By Ville's inequality, we have
    \begin{equation} \label{eq:ville-ineq-result}
        \PP(E_z)
        \le
        \PP \left(
            \max_{q \in \{\tau_r, \ldots, t-1\}} \Ecal_q (\lambda_0)
            \ge \exp 
                \left(
                \lambda_0 z - \frac{\lambda_0^2}{2} \alpha
                \right)
            \right)
            \le 
            \exp\left( -\lambda_0 z + \frac{\lambda_0^2}{2} \alpha\right) \, .
    \end{equation}
    Note that Eq.~\eqref{eq:ville-ineq-result} holds for any $\lambda_0 > 0$, for specific choicse $\lambda_0 = z/\alpha$, we have
    \begin{equation*}
        \PP(E_z) \le \exp\left( - \frac{z^2}{2 \alpha} \right) \, .
    \end{equation*}
    Applying the same argument to \(-\Mcal_{r,q,t}^{j}(\xb)\), we obtain
    \begin{equation} \label{eq:ville-ineq-result-negative}
        \PP\!\left(
        \left|    
        \ip{\ub_{r,t}(\xb)}{\nbb_{r,t-1}^j}
        \right|
        \ge z
        \right)
        \le 
        \PP\!\left(
        \max_{q \in \{\tau_r, \ldots, t-1\}}
        \left|    
        \ip{\ub_{r,t}(\xb)}{\nbb_{r,q}^j}
        \right|
        \ge z
        \right)
        \le
        2\exp\!\left(-\frac{z^2}{2\alpha}\right) \, .
    \end{equation}
    
    Finally, taking a union bound over epochs \(r\le R_\alpha\), rounds
    \(t\in[T]\), ensemble members \(j\in[m]\), and arms \(\xb\in\Xcal_t\), and
    using \(|\Xcal_t|\le K\) and \(R_\alpha\le \bar R_\alpha\), we get
    \[
        \PP\!\left(
        \max_{\substack{r\le R_\alpha,\;t\in[T],\\ j\in[m],\;\xb\in\Xcal_t}}
        \left|
        \ip{\ub_{r,t}(\xb)}{\nbb_{r,t-1}^j}
        \right|
        \ge z
        \right)
        \le
        2KmT\bar R_\alpha
        \exp\!\left(-\frac{z^2}{2\alpha}\right).
    \]
    Thus, taking \(z=\sqrt{2\alpha L_\alpha}\) yields, with probability at least \(1-\delta/4\),
    \[
        \max_{t\in[T],\,j\in[m],\,\xb\in\Xcal_t}
        \left|
        \ip{\ub_{r,t}(\xb)}{\nbb_{r,t-1}^j}
        \right|
        \le
        \sqrt{2\alpha L_\alpha} \, .
    \tag{E2}
    \]

    \paragraph{Term 3: epoch-start Gaussian norm.}
    Fix an epoch \(r\le \bar R_\alpha\) and an ensemble index \(j\in[m]\).
    Since \(\gb_r^j\sim \Ncal(\zero_d,\Ib_d)\), the standard concentration
    inequality for the Euclidean norm of a Gaussian vector
    (Lemma~\ref{aux_lem:gaussian-norm-concentration}) gives, for every \(u>0\),
    \[
        \PP\!\left(
        \norm{\gb_r^j}_2\ge \sqrt d+\sqrt{2u}
        \right)
        \le \exp(-u) \, .
    \]
    Taking a union bound over all epochs \(r\le \bar R_\alpha\) and ensemble members
    \(j\in[m]\), we obtain
    \[
        \PP\!\left(
        \max_{r\le \bar R_\alpha,\; j\in[m]}
        \norm{\gb_r^j}_2
        \ge
        \sqrt d+\sqrt{2u}
        \right)
        \le
        m\bar R_\alpha \exp(-u) \, .
    \]
    Choosing
    \[
        u:=\log\frac{4m\bar R_\alpha}{\delta}
    \]
    gives \(m\bar R_\alpha \exp(-u)=\delta/4\). Hence, with probability at least
    \(1-\delta/4\),
    \begin{equation} \tag{E3} \label{eq:Z-deom-3}
        \max_{r\le \bar R_\alpha,\; j\in[m]}
        \norm{\gb_r^j}_2
        \le
        \sqrt d+\sqrt{2\log\!\left(\frac{4m\bar R_\alpha}{\delta}\right)}
        \le
        \sqrt{d} + \sqrt{2 L_\alpha} \, .
    \end{equation}
    
    \paragraph{Term 4: vector within-epoch martingale.}
    Fix an epoch \(r\) and an ensemble index \(j\in[m]\).
    Define
    \[
        \ab_{r,q}:=\Vb_{\tau_r}^{-1/2}\xb_q,
        \qquad
        q\in\{\tau_r+1,\dots,\tau_{r+1}-1\}.
    \]
    Then, for every \(t\in\{\tau_r+1,\dots,\tau_{r+1}\}\),
    \[
        \nbb_{r,t-1}^j
        =
        \Vb_{\tau_r}^{-1/2}
        \sum_{s=\tau_r+1}^{t-1} \xb_s \xi_s^j
        =
        \sum_{s=\tau_r+1}^{t-1}\ab_{r,s}\xi_s^j .
    \]
    Moreover,
    \[
        \sum_{s=\tau_r+1}^{t-1}\ab_{r,s}\ab_{r,s}^{\top}
        =
        \Vb_{\tau_r}^{-1/2}
        (\Vb_{t-1}-\Vb_{\tau_r})
        \Vb_{\tau_r}^{-1/2}
        \preceq
        \alpha \Ib_d ,
    \]
    because \(t\) belongs to epoch \(r\) and hence
    \(\Vb_{t-1}\preceq (1+\alpha)\Vb_{\tau_r}\). Consequently, for every
    \(\ub\in\usd\),
    \[
        \sum_{s=\tau_r+1}^{t-1}
        \left(\ips{\ub}{\ab_{r,s}}\right)^2
        \le
        \alpha .
    \]
    
    Applying the scalar martingale bound from Term 2 in Eq.~\eqref{eq:ville-ineq-result} to the process
    \[
        \sum_{s=\tau_r+1}^{t-1}
        \ips{\ub}{\ab_{r,s}}\xi_s^j
        =
        \ips{\ub}{\nbb_{r,t-1}^j},
    \]
    we obtain, for every fixed \(\ub\in\usd\) and every \(z\ge0\),
    \begin{equation}
    \label{eq:tail-bound-vector-projection-martingale}
        \PP\!\left(
        \max_{t\in\{\tau_r+1,\dots,\tau_{r+1}\}}
        \left|\ips{\ub}{\nbb_{r,t-1}^j}\right|
        \ge z
        \right)
        \le
        2\exp\!\left(-\frac{z^2}{2\alpha}\right).
    \end{equation}
    
    Let \(\Ncal_{1/2}\subset\usd\) be a \(1/2\)-net satisfying
    \(|\Ncal_{1/2}|\le 5^d\). We use the standard net argument: for every
    \(\vb\in\RR^d\),
    \[
        \norm{\vb}_2
        \le
        2\max_{\ub\in\Ncal_{1/2}}|\ips{\ub}{\vb}|.
    \]
    Indeed, the claim is trivial if \(\vb=\zero_d\). Otherwise, let
    \(\wb:=\vb/\norm{\vb}_2\). By the definition of the net, there exists
    \(\ub^*\in\Ncal_{1/2}\) such that \(\norm{\ub^*-\wb}_2\le 1/2\). Then
    \[
        |\ips{\ub^*}{\vb}|
        \ge
        |\ips{\wb}{\vb}|-|\ips{\ub^*-\wb}{\vb}|
        \ge
        \norm{\vb}_2-\frac12\norm{\vb}_2
        =
        \frac12\norm{\vb}_2,
    \]
    which proves the claim.
    
    Applying this claim with \(\vb=\nbb_{r,t-1}^j\), we have
    \[
        \norm{\nbb_{r,t-1}^j}_2
        \le
        2\max_{\ub\in\Ncal_{1/2}}
        \left|\ips{\ub}{\nbb_{r,t-1}^j}\right|.
    \]
    Therefore,
    \[
        \left\{
        \max_{t\in\{\tau_r+1,\dots,\tau_{r+1}\}}
        \norm{\nbb_{r,t-1}^j}_2
        \ge 2z
        \right\}
        \subseteq
        \bigcup_{\ub\in\Ncal_{1/2}}
        \left\{
        \max_{t\in\{\tau_r+1,\dots,\tau_{r+1}\}}
        \left|\ips{\ub}{\nbb_{r,t-1}^j}\right|
        \ge z
        \right\}.
    \]
    Using Eq.~\eqref{eq:tail-bound-vector-projection-martingale} and taking a union
    bound over \(\Ncal_{1/2}\), we get
    \[
        \PP\!\left(
        \max_{t\in\{\tau_r+1,\dots,\tau_{r+1}\}}
        \norm{\nbb_{r,t-1}^j}_2
        \ge 2z
        \right)
        \le
        2|\Ncal_{1/2}|
        \exp\!\left(-\frac{z^2}{2\alpha}\right)
        \le
        2\cdot 5^d
        \exp\!\left(-\frac{z^2}{2\alpha}\right).
    \]
    Taking another union bound over all \(r\le \bar R_\alpha\) and \(j\in[m]\), we
    obtain
    \[
        \PP\!\left(
        \max_{\substack{r\le \bar R_\alpha,\; j\in[m],\\
        t\in\{\tau_r+1,\dots,\tau_{r+1}\}}}
        \norm{\nbb_{r,t-1}^j}_2
        \ge 2z
        \right)
        \le
        2m\bar R_\alpha 5^d
        \exp\!\left(-\frac{z^2}{2\alpha}\right).
    \]
    Choosing
    \[
        z^2
        =
        2\alpha
        \left(
        d\log 5+\log\frac{8m\bar R_\alpha}{\delta}
        \right),
    \]
    we conclude that, with probability at least \(1-\delta/4\),
    \[
        \max_{\substack{r\le \bar R_\alpha,\; j\in[m],\\
        t\in\{\tau_r+1,\dots,\tau_{r+1}\}}}
        \norm{\nbb_{r,t-1}^j}_2
        \le
        2\sqrt{
        2\alpha\left(
        d\log 5+\log\frac{8m\bar R_\alpha}{\delta}
        \right)
        }.
    \]
    Finally, since \(K,T\ge1\) and
    \(L_\alpha=\log\frac{8KmT\bar R_\alpha}{\delta}\), the logarithmic factor above
    is bounded by \(d\log 5+L_\alpha\). Hence, for a universal constant \(C>0\),
    \begin{equation}
    \tag{E4}
    \label{eq:Z-deom-4}
        \max_{\substack{r\le \bar R_\alpha,\; j\in[m],\\
        t\in\{\tau_r+1,\dots,\tau_{r+1}\}}}
        \norm{\nbb_{r,t-1}^j}_2
        \le
        C\sqrt{\alpha(d+L_\alpha)} .
    \end{equation}
    
    \paragraph{Combining the four events.}
    By a union bound, the events in Eqs.~\eqref{eq:Z-deom-1}--\eqref{eq:Z-deom-4}
    hold simultaneously with probability at least \(1-\delta\).
    On this event, Eq.~\eqref{eq:Z-decomposition} implies that, for all
    \(t\in[T]\), \(j\in[m]\), and nonzero \(\xb\in\Xcal_t\),
    \begin{equation}
    \label{eq:arm-wise-bound-general}
        \left|
        \frac{\xb^\top \Vb_{t-1}^{-1}\sbb_{t-1}^j}
        {\|\xb\|_{\Vb_{t-1}^{-1}}}
        \right|
        \le
        C\left[
        \sqrt{L_\alpha}
        +
        \sqrt{\alpha L_\alpha}
        +
        \alpha(\sqrt d+\sqrt{L_\alpha})
        +
        \alpha\sqrt{\alpha(d+L_\alpha)}
        \right],
    \end{equation}
    where \(C>0\) is a universal constant. This proves the first part of the theorem.
    
    For the second part, suppose that a deterministic quantity \(\Lambda_*\ge 1\)
    satisfies \(L_\alpha\le C_\Lambda \Lambda_*\) for some universal constant
    \(C_\Lambda>0\), and choose
    $
        \alpha=\min\{1,\sqrt{\Lambda_*/d}\} \, .
    $
    We show that every term on the right-hand side of
    Eq.~\eqref{eq:arm-wise-bound-general} is bounded by a universal constant times
    \(\sqrt{\Lambda_*}\).
    
    First, since \(L_\alpha\le C_\Lambda\Lambda_*\), we have
    \[
        \sqrt{L_\alpha}
        \le
        \sqrt{C_\Lambda}\sqrt{\Lambda_*}.
    \]
    Also, since \(\alpha\le 1\),
    \[
        \sqrt{\alpha L_\alpha}
        \le
        \sqrt{L_\alpha}
        \le
        \sqrt{C_\Lambda}\sqrt{\Lambda_*}.
    \]
    Next, we control the drift term \(\alpha(\sqrt d+\sqrt{L_\alpha})\).
    By the definition of \(\alpha\), we have
    \[
        \alpha\sqrt d
        =
        \min\left\{\sqrt d,\sqrt{\Lambda_*}\right\}
        \le
        \sqrt{\Lambda_*}.
    \]
    Moreover,
    \[
        \alpha\sqrt{L_\alpha}
        \le
        \sqrt{L_\alpha}
        \le
        \sqrt{C_\Lambda}\sqrt{\Lambda_*}.
    \]
    Therefore,
    \[
        \alpha(\sqrt d+\sqrt{L_\alpha})
        \le
        (1+\sqrt{C_\Lambda})\sqrt{\Lambda_*}.
    \]
    
    It remains to bound the last term. Using
    \(\sqrt{a+b}\le \sqrt a+\sqrt b\), we get
    \[
        \alpha\sqrt{\alpha(d+L_\alpha)}
        \le
        \alpha\sqrt{\alpha d}
        +
        \alpha\sqrt{\alpha L_\alpha}.
    \]
    Since \(\alpha\le 1\),
    \[
        \alpha\sqrt{\alpha d}
        \le
        \alpha\sqrt d
        \le
        \sqrt{\Lambda_*},
    \]
    and similarly,
    \[
        \alpha\sqrt{\alpha L_\alpha}
        \le
        \sqrt{L_\alpha}
        \le
        \sqrt{C_\Lambda}\sqrt{\Lambda_*}.
    \]
    Thus,
    \[
        \alpha\sqrt{\alpha(d+L_\alpha)}
        \le
        (1+\sqrt{C_\Lambda})\sqrt{\Lambda_*}.
    \]
    
    Combining the above bounds in Eq.~\eqref{eq:arm-wise-bound-general}, and
    absorbing universal constants into a new constant \(\widetilde C>0\), we obtain
    \[
        \max_{t\in[T],\,j\in[m],\,\xb\in\Xcal_t}
        \left|
        \frac{\xb^\top \Vb_{t-1}^{-1}\sbb_{t-1}^j}
        {\|\xb\|_{\Vb_{t-1}^{-1}}}
        \right|
        \le
        \widetilde C\sqrt{\Lambda_*}.
    \]
\end{proof}

\subsection{Proof of Theorem~\ref{thm:finite-arm-regret}}
In this section, we provide a regret bound for Algorithm~\ref{alg:main algorithm} in the finite-arm setting.

\RegretForFiniteArmSets*
\begin{proof}[Proof of Theorem~\ref{thm:finite-arm-regret}]

    We first verify that the parameter choices stated in Theorem~\ref{thm:finite-arm-regret} satisfy the
    technical requirements needed for the proof.
    Since \(\Lambda_0\ge 1\), we have
    \[
        \alpha
        =
        \min\left\{
            1,\sqrt{\frac{\Lambda_0}{d}}
        \right\}
        \ge
        \frac{1}{\sqrt d}
        =
        \alpha_0 .
    \]
    Since the deterministic epoch bound \(\bar R_\alpha\) is nonincreasing in \(\alpha\), we have
    \[
        \bar R_\alpha \le \bar R_{\alpha_0} = \bar R_0.
    \]
    Therefore, by the definition of \(M_0\) and Corollary~\ref{cor:explicit-ensemble-size},
    the choice \(m=M_0\) is sufficient for the uniform exceedance guarantee
    of Lemma~\ref{lem:uniform-exceedance-across-epochs} for the actual value of
    \(\alpha\).
    
    Moreover, since \(m=M_0\) and \(\bar R_\alpha\le \bar R_0\), the logarithmic
    factor in Lemma~\ref{lem:uniform arm-wise concentration} satisfies
    \[
        L_\alpha
        =
        \log\frac{8KmT\bar R_\alpha}{\delta}
        \le
        \log\frac{8KM_0T\bar R_0}{\delta}
        \le
        C\Lambda_0
    \]
    for a universal constant \(C>0\), by the definition of \(\Lambda_0\).
    Since
    $
        \alpha
        =
        \min \{
            1,\sqrt{\Lambda_0/d}
        \},
    $
    the second part of Lemma~\ref{lem:uniform arm-wise concentration}, applied
    with \(\Lambda_*=\Lambda_0\), gives
    \[
        \max_{t\in[T],\,j\in[m],\,\xb\in\Xcal_t}
        \left|
        \frac{\xb^\top\Vb_{t-1}^{-1}\sbb_{t-1}^j}
        {\|\xb\|_{\Vb_{t-1}^{-1}}}
        \right|
        \le
        \widetilde C\sqrt{\Lambda_0}
    \]
    with probability at least \(1-\delta\), where \(\widetilde C>0\) is a universal
    constant.
    
    We now introduce the finite-arm version of the randomized-optimism master
    regret bound.

    \begin{restatable}[Finite-arm regret bound for randomized algorithms]{lemma}{FiniteArmMasterRegret}
        \label{lem:finite-arm-master-regret}
        Suppose that the arm sets \(\{\Xcal_t\}_{t=1}^T\) are chosen by an oblivious
        adversary with \(|\Xcal_t|=K\) for all \(t\in[T]\).
        Under the assumptions of Lemma~\ref{aux_lem:master regret}, the same conclusion
        continues to hold if the event
        \[
            \mathcal E
            =
            \left\{
                \forall t\le T:
                \|\thetab_t-\hat{\thetab}_{t-1}\|_{\Vb_{t-1}}
                \le b_{t-1}
            \right\}
        \]
        is replaced by the arm-wise event
        \[
            \widetilde{\mathcal E}
            =
            \left\{
                \forall t\le T,\ \forall \xb\in\Xcal_t:
                \left|\xb^\top(\thetab_t-\hat{\thetab}_{t-1})\right|
                \le
                b_{t-1}\|\xb\|_{\Vb_{t-1}^{-1}}
            \right\}.
        \]
    \end{restatable}
    
    Define the following events:
    \begin{align*}
        \Ecal_{\rm LS}
        &:=
        \left\{
            \forall t\in[T]:
            \|\hat{\thetab}_{t-1}-\thetab^*\|_{\Vb_{t-1}}
            \le
            \beta_{t-1}
        \right\},
        \\
        \Ecal_{\rm arm}
        &:=
        \left\{
            \forall t\in[T],\ \forall j\in[m],\ \forall \xb\in\Xcal_t:
            \left|
            \xb^\top\Vb_{t-1}^{-1}\sbb_{t-1}^j
            \right|
            \le
            \widetilde C\sqrt{\Lambda_0}
            \|\xb\|_{\Vb_{t-1}^{-1}}
        \right\},
        \\
        \Ecal_{\rm exc}
        &:=
        \left\{
            \forall t\in[T]:
            \inf_{\ub\in\mathbb S^{d-1}}
            E_{t,m}(\ub,1/\gamma)
            \ge
            \frac{1}{10}
        \right\}.
    \end{align*}
    By Lemma~\ref{aux_lem:confidence ellipsoid},
    Lemma~\ref{lem:uniform arm-wise concentration}, and
    Lemma~\ref{lem:uniform-exceedance-across-epochs}, respectively,
    \[
        \PP(\Ecal_{\rm LS})\ge 1-\delta,
        \qquad
        \PP(\Ecal_{\rm arm})\ge 1-\delta,
        \qquad
        \PP(\Ecal_{\rm exc})\ge 1-\delta.
    \]
    
    For each \(j\in[m]\), define the \(j\)-th ensemble parameter at round \(t\) by
    \[
        \thetab_t^j
        :=
        \hat{\thetab}_{t-1}
        +
        \gamma\beta_{t-1}\Vb_{t-1}^{-1}\sbb_{t-1}^j.
    \]
    On \(\Ecal_{\rm arm}\), for all \(t\in[T]\), \(j\in[m]\), and
    \(\xb\in\Xcal_t\),
    \[
        \left|
        \xb^\top(\thetab_t^j-\hat{\thetab}_{t-1})
        \right|
        =
        \gamma\beta_{t-1}
        \left|
        \xb^\top\Vb_{t-1}^{-1}\sbb_{t-1}^j
        \right|
        \le
        \gamma\widetilde C\sqrt{\Lambda_0}\,
        \beta_{t-1}
        \|\xb\|_{\Vb_{t-1}^{-1}}.
    \]
    Define
    \[
        b_{t-1}
        :=
        \left(1+\gamma\widetilde C\sqrt{\Lambda_0}\right)\beta_{t-1}.
    \]
    Then \(b_{t-1}\ge \beta_{t-1}\). Hence, on \(\Ecal_{\rm LS}\),
    \[
        \|\thetab^*-\hat{\thetab}_{t-1}\|_{\Vb_{t-1}}
        \le
        \beta_{t-1}
        \le
        b_{t-1}
    \]
    for all \(t\in[T]\). Moreover, on \(\Ecal_{\rm arm}\), the arm-wise event
    \(\widetilde{\mathcal E}\) in Lemma~\ref{lem:finite-arm-master-regret} holds
    for the sampled ensemble parameter \(\thetab_t=\thetab_t^{J_t}\), since
    \[
        \left|
        \xb^\top(\thetab_t-\hat{\thetab}_{t-1})
        \right|
        \le
        \gamma\widetilde C\sqrt{\Lambda_0}\,
        \beta_{t-1}
        \|\xb\|_{\Vb_{t-1}^{-1}}
        \le
        b_{t-1}\|\xb\|_{\Vb_{t-1}^{-1}}.
    \]
    
    It remains to verify the optimism probability lower bound. Fix \(t\in[T]\), and
    let
    \[
        \xb_t^*\in\argmax_{\xb\in\Xcal_t}\ips{\xb}{\thetab^*}.
    \]
    If \(\xb_t^*=\zero_d\), then
    \[
        \ips{\xb_t}{\thetab_t}
        =
        \max_{\xb\in\Xcal_t}\ips{\xb}{\thetab_t}
        \ge
        0
        =
        \ips{\xb_t^*}{\thetab^*},
    \]
    and hence \(p_{t-1}=1\).
    
    Assume now that \(\xb_t^*\neq\zero_d\), and define
    \[
        \ub_t^*
        :=
        \Vb_{t-1}^{-1}\xb_t^*,
        \qquad
        \bar{\ub}_t^*
        :=
        \frac{\ub_t^*}{\|\ub_t^*\|_2}
        \in \mathbb S^{d-1}.
    \]
    Since \(\Ecal_{\rm exc}\) holds,
    \[
        E_{t,m}(\bar{\ub}_t^*,1/\gamma)
        \ge
        \frac{1}{10}.
    \]
    Equivalently, at least a \(1/10\)-fraction of ensemble members \(j\in[m]\)
    satisfy
    \[
        \frac{\ips{\sbb_{t-1}^j}{\bar{\ub}_t^*}}
        {\|\bar{\ub}_t^*\|_{\Vb_{t-1}}}
        \ge
        \frac{1}{\gamma}.
    \]
    By positive rescaling invariance of the self-normalized ratio, this is
    equivalent to
    \[
        \frac{\ips{\sbb_{t-1}^j}{\ub_t^*}}
        {\|\ub_t^*\|_{\Vb_{t-1}}}
        \ge
        \frac{1}{\gamma}.
    \]
    Since \(\ub_t^*=\Vb_{t-1}^{-1}\xb_t^*\), we have
    \[
        \|\ub_t^*\|_{\Vb_{t-1}}
        =
        \|\xb_t^*\|_{\Vb_{t-1}^{-1}},
    \]
    and therefore every such \(j\) satisfies
    \[
        \ips{\xb_t^*}{\thetab_t^j-\hat{\thetab}_{t-1}}
        =
        \gamma\beta_{t-1}
        (\xb_t^*)^\top\Vb_{t-1}^{-1}\sbb_{t-1}^j
        \ge
        \beta_{t-1}
        \|\xb_t^*\|_{\Vb_{t-1}^{-1}}.
    \]
    On the other hand, on \(\Ecal_{\rm LS}\),
    \[
        \ips{\xb_t^*}{\hat{\thetab}_{t-1}-\thetab^*}
        \ge
        -
        \beta_{t-1}
        \|\xb_t^*\|_{\Vb_{t-1}^{-1}}.
    \]
    Combining the two displays gives
    \[
        \ips{\xb_t^*}{\thetab_t^j}
        =
        \ips{\xb_t^*}{\hat{\thetab}_{t-1}}
        +
        \ips{\xb_t^*}{\thetab_t^j-\hat{\thetab}_{t-1}}
        \ge
        \ips{\xb_t^*}{\thetab^*}.
    \]
    Let \(\xb_t(j)\in\argmax_{\xb\in\Xcal_t}\ips{\xb}{\thetab_t^j}\) be the arm
    that would be selected if ensemble member \(j\) were sampled. Then
    \[
        \ips{\xb_t(j)}{\thetab_t^j}
        \ge
        \ips{\xb_t^*}{\thetab_t^j}
        \ge
        \ips{\xb_t^*}{\thetab^*}.
    \]
    Since \(J_t\) is conditionally uniform on \([m]\) given \(\Acal_{t-1}\), it
    follows that, on \(\Ecal_{\rm LS}\cap\Ecal_{\rm exc}\),
    \[
        p_{t-1}
        :=
        \PP\!\left(
            \ips{\xb_t}{\thetab_t}
            \ge
            \ips{\xb_t^*}{\thetab^*}
            \,\middle|\,\Acal_{t-1}
        \right)
        \ge
        \frac{1}{10}
    \]
    for all \(t\in[T]\).
    
    We now apply Lemma~\ref{lem:finite-arm-master-regret}. Its confidence
    conditions are verified on \(\Ecal_{\rm arm}\) and \(\Ecal_{\rm LS}\), and the
    lower bound \(p_{t-1}\ge 1/10\) holds on
    \(\Ecal_{\rm LS}\cap\Ecal_{\rm exc}\). Therefore, applying the finite-arm
    master regret lemma and intersecting with the uniform exceedance event gives
    an event of probability at least \(1-4\delta\) on which
    \[
        \regret_T
        \le
        20\max_{t\in[T]} b_{t-1}
        \left(
            2\sqrt{
                2dT\log\!\left(1+\frac{T}{d\lambda}\right)
            }
            +
            \sqrt{
                2(4T/\lambda+1)
                \log\!\left(
                    \frac{\sqrt{4T/\lambda+1}}{\delta}
                \right)
            }
        \right).
    \]
    
    Finally, by monotonicity of \(\beta_t\),
    \[
        \max_{t\in[T]} b_{t-1}
        \le
        \left(1+\gamma\widetilde C\sqrt{\Lambda_0}\right)\beta_T.
    \]
    Since \(\beta_T=\widetilde O(\sqrt d)\) and
    \[
        1+\gamma\widetilde C\sqrt{\Lambda_0}
        =
        \widetilde O(\sqrt{\Lambda_0}),
    \]
    we obtain
    \[
        \max_{t\in[T]} b_{t-1}
        =
        \widetilde O(\sqrt{d\Lambda_0}).
    \]
    Consequently,
    \[
        \regret_T
        =
        \widetilde O\!\left(
            d\sqrt{T\Lambda_0}
        \right).
    \]
    Since
    \[
        \Lambda_0
        =
        \max\left\{
            1,\,
            C_\Lambda
            \log\!\left(
                \frac{8KM_0T\bar R_0}{\delta}
            \right)
        \right\},
    \]
    and \(M_0\) and \(\bar R_0\) depend only logarithmically on \(T,d\), and
    \(1/\delta\), we have
    \[
        \regret_T
        =
        \widetilde O\!\left(
            d\sqrt{T\log K}
        \right) \, .
    \]
\end{proof}

\section{Technical Lemmas}
\subsection{Proof of Lemma~\ref{lem:epoch-brownian-representation}} \label{appx:proof of epoch-brownian-representation}
\EpochBrownianRepresentation*
\begin{proof}[Proof of Lemma~\ref{lem:epoch-brownian-representation}]
    Let \(\mathcal G_{-1}\) be the \(\sigma\)-field generated by all randomness prior to drawing
    the refreshed Gaussian vectors at the start of epoch \(r\).
    
    For each \(j\in[m]\), define
    \[
    Z_r^j
    :=
    \frac{\ip{\ub}{\Vb_{\tau_r}^{1/2}\gb_r^j}}{\norm{\ub}_{\Vb_{\tau_r}}}.
    \]
    Since \(\gb_r^j \stackrel{\mathrm{i.i.d.}}{\sim} \Ncal(\zero_d,\Ib_d)\) and \(\norm{\ub}_{\Vb_{\tau_r}}>0\),
    the variables \(Z_r^1,\dots,Z_r^m\) are \textit{i.i.d.} \(\Ncal(0,1)\), and are independent of
    \(\mathcal G_{-1}\).
    
    For \(k=1,\dots,\ell_r-1\), let
    \[
    \varepsilon_k := (\xi_{\tau_r+k}^1,\dots,\xi_{\tau_r+k}^m)\in\RR^m,
    \]
    and set
    \[
    \varepsilon_0 := (Z_r^1,\dots,Z_r^m)\in\RR^m.
    \]
    Then \(\varepsilon_0,\varepsilon_1,\dots,\varepsilon_{\ell_r-1}\) are \textit{i.i.d.} \(\Ncal(\zero_m,\Ib_m)\).
    
    Now define
    \[
    \Ub_0 := D_0 \Ib_m,
    \qquad
    \Ub_k := D_k \Ib_m \, ,
    \quad 1 \le k \le \ell_r-1 \, .
    \]
    For \(k\ge 1\), the coefficient
    \[
    D_k = \ip{\ub}{\xb_{\tau_r+k}}
    \]
    is measurable with respect to the randomness up to just before \(\varepsilon_k\) is sampled,
    so \(\Ub_k\) is predictable; \(\Ub_0\) is \(\mathcal G_{-1}\)-measurable.
    
    Moreover, for every \(k=0,\dots,\ell_r-1\),
    \[
    \Mb_k
    =
    \Ub_0\varepsilon_0+\sum_{s=1}^k \Ub_s\varepsilon_s.
    \]
    Indeed, for each \(j\in[m]\),
    \begin{align*}
        \Mb_k^j
        &=
        \ips{\ub}{\sbb_{\tau_r+k}^j} \\
        &=
        \ip{\ub}{\Vb_{\tau_r}^{1/2}\gb_r^j}
        +
        \sum_{s=\tau_r+1}^{\tau_r+k}\ip{\ub}{\xb_s}\,\xi_s^j \\
        &=
        D_0 Z_r^j + \sum_{s=1}^k D_s \xi_{\tau_r+s}^j \, .
    \end{align*}
    Thus \((\Mb_k)\) is a diagonal martingale transform of a standard \(m\)-dimensional Gaussian noise sequence, with diagonal coefficient matrices that are scalar multiples of the identity.
    Applying the diagonal Gaussian embedding theorem of~\citet{janz2026sharp} (Lemma~\ref{aux_lem:akhavan-diagonal-embedding}), we obtain, on an extension of the probability space, independent standard Brownian motions \(W^1,\dots,W^m\) such that
    \begin{equation*}
        \Mb_k^j = W^j(A_{k,j}^2),
        \qquad
        A_{k,j}^2 := \sum_{s=0}^k D_{s,j}^2 \, .        
    \end{equation*}

    Since every coefficient matrix is scalar, the clocks are the same for all \(j\), so
    \[
    A_{k,j}^2 = A_k^2 = D_0^2+\sum_{s=1}^k D_s^2 \, .
    \]
    Finally,
    \[
    A_k^2
    =
    \norm{\ub}_{\Vb_{\tau_r}}^2
    +
    \sum_{s=1}^k \ip{\ub}{\xb_{\tau_r+s}}^2
    =
    \norm{\ub}_{\Vb_{\tau_r+k}}^2,
    \]
    because
    \[
    \Vb_{\tau_r+k}
    =
    \Vb_{\tau_r}
    +
    \sum_{s=1}^k \xb_{\tau_r+s}\xb_{\tau_r+s}^\top \, .
    \]
    Substituting \(k=t-1-\tau_r\) gives
    \[
    \ips{\ub}{\sbb_{t-1}^j}
    =
    W^j\!\bigl(\norm{\ub}_{\Vb_{t-1}}^2\bigr)
    \]
    for all \(t\in\{\tau_r+1,\dots,\tau_{r+1}\}\) and all \(j\in[m]\) \, .
\end{proof}

\subsection{Proof of Lemma~\ref{lem:matrix-drift-perturbation}} \label{appx:proof of lem:matrix-drift-perturbation}
\MatrixDriftPerturbation*
\begin{proof}[Proof of Lemma~\ref{lem:matrix-drift-perturbation}]
    Let
    \[
    \ub := \frac{\yb}{\norm{\yb}_2},
    \qquad
    \Bb := \Ab^{-1} \, .
    \]
    Then we have 
    \[
    \frac{\yb^\top \Ab^{-1}\vb}{\sqrt{\yb^\top \Ab^{-1}\yb}}
    =
    \ip{
    \frac{\Bb\ub}{\sqrt{\ub^\top \Bb\ub}}
    }{\vb} \, .
    \]
    Therefore,
    \[
    \left|
    \frac{\yb^\top \Ab^{-1}\vb}{\sqrt{\yb^\top \Ab^{-1}\yb}}
    -
    \ip{\ub}{\vb}
    \right|
    = 
    \left|
    \ip{
    \frac{\Bb\ub}{\sqrt{\ub^\top \Bb\ub}}
    - \ub}{\vb}
    \right|
    \le
    \left\|
    \frac{\Bb\ub}{\sqrt{\ub^\top \Bb\ub}}-\ub
    \right\|_2
    \norm{\vb}_2 \, ,
    \]
    where the inequality uses Cauchy-Schwarz inequality.
    
    Since $\Ib_d\preceq \Ab\preceq (1+\alpha)\Ib_d$, we have $\frac{1}{1+\alpha}\Ib_d \preceq \Bb \preceq \Ib_d$.
    Hence we get
    \[
    \|\Bb-\Ib_d\|_{\op}
    \le
    1-\frac{1}{1+\alpha}
    =
    \frac{\alpha}{1+\alpha} \, ,
    \qquad
    \ub^\top \Bb\ub\in\left[\frac{1}{1+\alpha},1\right] \, .
    \]
    Thus
    \begin{align*}
        \left\|
        \frac{\Bb\ub}{\sqrt{\ub^\top \Bb\ub}}-\ub
        \right\|_2
        &\le
        \frac{\|(\Bb-\Ib_d)\ub\|_2}{\sqrt{\ub^\top \Bb\ub}}
        +
        \left|
        \frac{1}{\sqrt{\ub^\top \Bb\ub}}-1
        \right| \\
        &\le
        \frac{\alpha}{1+\alpha}\sqrt{1+\alpha}
        +
        \left(\sqrt{1+\alpha}-1\right) \\
        &=
        \frac{\alpha}{\sqrt{1+\alpha}}
        +
        \frac{\alpha}{\sqrt{1+\alpha}+1} \\
        &\le
        \alpha+\alpha
        =
        2\alpha \, ,
    \end{align*}
    which concludes the proof.
\end{proof}

\subsection{Proof of Lemma~\ref{lem:finite-arm-master-regret}}
\FiniteArmMasterRegret*
\begin{proof}[Proof of Lemma~\ref{lem:finite-arm-master-regret}]
    In the proof of Lemma~\ref{aux_lem:master regret} in~\citet{janz2024ensemble}, the only role of the event
    \[
    \Ecal
    =
    \left\{
    \forall t\le T:\ 
    \|\thetab_t-\hat{\thetab}_{t-1}\|_{\Vb_{t-1}}\le b_{t-1}
    \right\}
    \]
    is to imply, by Cauchy--Schwarz, that for any arm \(\xb\in\Xcal_t\),
    \[
    \left|
    \xb^\top(\thetab_t-\hat{\thetab}_{t-1})
    \right|
    \le
    b_{t-1}\|\xb\|_{\Vb_{t-1}^{-1}}.
    \]
    However, in the regret argument this implication is only invoked for two specific arms: the played arm \(\xb_t\) and the optimal arm \(\xb^*_t\) at round $t$. Indeed, on the optimism event
    \[
    \ip{\thetab_t}{\xb_t}\ge \ip{\thetab^*}{\xb^*_t},
    \]
    the instantaneous regret satisfies
    \begin{align*}
    \ip{\xb^*_t-\xb_t}{\thetab^*}
    &\le
    \ip{\xb_t}{\thetab_t-\thetab^*} \\
    &=
    \ips{\xb_t}{\thetab_t-\hat{\thetab}_{t-1}}
    +
    \ips{\xb_t}{\hat{\thetab}_{t-1}-\thetab^*}.
    \end{align*}
    Thus the proof only needs upper bounds on
    \[
    \left|\ips{\xb_t}{\thetab_t-\hat{\thetab}_{t-1}}
    \right|
    \quad\text{and}\quad
    \left|
    \ips{\xb_t}{\hat{\thetab}_{t-1}-\thetab^*}
    \right| \, .
    \]
    Similarly, to show that a sampled parameter is optimistic whenever it is sufficiently large in the optimal-arm direction, the proof only uses the corresponding inequalities at \(\xb^*\):
    \[
    \left|
    \ips{\xb^*_t}{\thetab_t-\hat{\thetab}_{t-1}}
    \right|
    \quad\text{and}\quad
    \left|
    \ips{\xb^*_t}{\hat{\thetab}_{t-1}-\thetab^*}
    \right| \, .
    \]
    
    Therefore, the full confidence-set event \(\Ecal\) is stronger than necessary. It is enough to assume the armwise event
    \[
    \widetilde \Ecal
    =
    \left\{
    \forall t\le T,\ \forall \xb\in\Xcal_t:
    \left|\xb^\top(\thetab_t-\hat{\thetab}_{t-1})\right|
    \le
    b_{t-1}\|\xb\|_{\Vb_{t-1}^{-1}}
    \right\}.
    \]
    Since \(\xb_t\in\Xcal_t\) and \(\xb_t^*\in\Xcal_t\), this replaced event provides exactly the same inequalities used in the proof of Lemma~\ref{aux_lem:master regret}. All subsequent steps, including the optimism-probability argument and the elliptical-potential bound, are unchanged. Hence the same regret bound follows.
\end{proof}

\section{Auxiliary Lemmas}
\begin{lemma}[Lemma 10 in~\citet{abbasi2011improved}] \label{aux_lem:det-tr inequality}
    Suppose $X_1, \ldots, X_t \in \RR^d$ and for any $1 \le s \le t$, $\| X_s \|_2 \le L$,
    Let $\Vb_t = \lambda \Ib_d + \sum_{s=1}^t X_s X_s^\top$ for some $\lambda > 0$.
    Then,
    \begin{equation*}
        \det(\Vb_t) \le (\lambda + t L^2/d)^d \, .
    \end{equation*}
\end{lemma}

\begin{lemma}[Theorem 5.3 in~\citet{janz2026sharp}] \label{aux_lem:akhavan-diagonal-embedding}
    Fix a horizon \(N\in\mathbb N\). Assume
    \[
    M_t=\sum_{s=0}^t \diag(D_s)\,\xi_s,
    \qquad 0\le t\le N-1,
    \]
    is a diagonal martingale transform of a standard \(m\)-dimensional Gaussian noise.
    Define the coordinate clocks
    \[
    A_{t,j}^2:=\sum_{s=0}^t D_{s,j}^2,
    \qquad j\in[m].
    \]
    Then, on an extension of the probability space, there exist independent standard
    Brownian motions \(W^1,\dots,W^m\) such that for all
    \(t\in\{0,\dots,N-1\}\) and all \(j\in[m]\),
    \[
    M_{t,j}=W^j(A_{t,j}^2).
    \]    
\end{lemma}

\begin{lemma}[Lemma 7 in~\citet{lee2024improved}] \label{aux_lem:gaussian-norm-concentration}
    If $Z \sim \Ncal( \zero_d, \Ib_d)$ is a $d$-dimensional multivariate Gaussian vector, then for any $\delta \in (0,1]$, 
    \begin{equation*}
        \PP \left( \| Z \|_2 \ge \sqrt{d} + \sqrt{2 \log \frac{1}{\delta}} \le \delta \, .
        \right)
    \end{equation*}    
\end{lemma}

\begin{lemma}[Theorem 1 in~\citet{abbasi2011improved}] \label{aux_lem:self-normalized inequality}
    Let \(\{F_t\}_{t=0}^\infty\) be a filtration. Let \(\{\eta_t\}_{t=1}^\infty\) be a real-valued stochastic process such that \(\eta_t\) is \(F_t\)-measurable and \(\eta_t\) is conditionally \(R\)-sub-Gaussian for some \(R \ge 0\).
    Let \(\{X_t\}_{t=1}^\infty\) be an \(\mathbb{R}^d\)-valued stochastic process such that \(X_t\) is \(F_{t-1}\)-measurable. Assume that \(V\) is a \(d\times d\) positive definite matrix. For any \(t\ge 0\), define
    \[
    \overline V_t
    =
    V+\sum_{s=1}^t X_sX_s^\top,
    \qquad
    S_t
    =
    \sum_{s=1}^t \eta_s X_s.
    \]
    
    Then, for any \(\delta>0\), with probability at least \(1-\delta\), for all \(t\ge 0\),
    \[
    \|S_t\|_{\overline V_t^{-1}}^2
    \le
    2R^2\log\!\left(
    \frac{\det(\overline V_t)^{1/2}\det(V)^{-1/2}}{\delta}
    \right).
    \]
\end{lemma}

\section{Experimental Details} \label{appx:experiments}
\subsection{Experimental Setup}
We evaluate all methods on synthetic stochastic linear bandit instances over the unit Euclidean ball in \(\mathbb{R}^d\).
For each seed, the unknown parameter \(\thetab^*\) is sampled from a standard Gaussian distribution and normalized to unit norm.
We compare LinES~\citep{janz2026sharp}, $\algname$ (ours) against randomized baselines LinTS~\citep{abeille2017Linear} and LinPHE~\citep{kveton2020perturbed}.
Unless otherwise specified, we use parameter \(\lambda=1\), failure probability \(\delta=10^{-4}\), Gaussian noise level \(\sigma=0.5\), and average results over 10 independent seeds.
The ensemble sizes are set according to the theoretical prescriptions: \(m=\lceil d\log(T/\delta)\rceil\) for \texttt{LinES} and \(m=\lceil d\log d+d\log\log T+\log(1/\delta)\rceil\) for $\algname$.
Runtime is measured as the cumulative wall-clock time spent in arm selection and model updates.
Experiments were run on a MacBook Pro with an Apple M4 Pro chip, 12 CPU cores, and 24 GB memory, running macOS 26.3.1.

For $\algname$, we implement epoch refreshes using an infrequent scheduled check rather than testing the refresh condition at every round. Specifically, the refresh criterion is evaluated only at geometrically spaced checkpoints within each horizon scale, which avoids adding a per-round eigenvalue computation to the runtime. This implementation preserves the intended epoch-reset behavior while substantially reducing overhead; empirically, the resulting scheduled refreshes do not noticeably degrade performance in our synthetic experiments.

\subsection{Additional Experiments}
Figure~\ref{fig:unitball-d} reports cumulative regret and runtime for
\(T=10{,}000\) with \(d\in\{10,20,50\}\). Across all dimensions,
\texttt{LinES-ER} achieves regret comparable to \texttt{LinES} while using
substantially fewer ensemble members. The computational benefit becomes more
visible as the dimension increases. For example, when \(d=50\),
\texttt{LinES-ER} uses \(316\) ensemble members compared with \(922\) for
\texttt{LinES}, and reduces the runtime from \(0.498\) seconds to \(0.316\)
seconds, corresponding to a runtime reduction of about $\frac{0.498-0.316}{0.498}\times 100 \approx 36.5\%$.
These results support the main practical implication of our theory: reducing
the required ensemble size preserves the statistical performance of ensemble
sampling while improving computational efficiency.

\begin{table}[h!]
    \centering
    \caption{
    Comparison of theoretical ensemble sizes and runtime improvement for different dimensions.
    The baseline ensemble size \(m_{\mathrm{base}}\) is computed using the prescription of~\citet{janz2026sharp},
    \(m_{\mathrm{base}}=\lceil d\log(T/\delta)\rceil\), whereas our ensemble size is computed as
    \(m_{\mathrm{ours}}=\lceil d\log d+d\log\log T+\log(1/\delta)\rceil\).
    Runtime improvement is measured relative to \texttt{LinES}.
    }
    \label{tab:runtime-improvement}
    \begin{tabular}{ccccccc}
    \toprule
    \(T\) & \(d\) & \(\delta\) & \(m_{\mathrm{ours}}\) & \(m_{\mathrm{base}}\) & \(m_{\mathrm{base}}/m_{\mathrm{ours}}\) & Runtime improvement \\
    \midrule
    \(10^4\) & \(10\) & \(10^{-4}\) & \(55\)  & \(185\) & \(3.36\) & \(9.65\%\) \\
    \(10^4\) & \(20\) & \(10^{-4}\) & \(114\) & \(369\) & \(3.24\) & \(24.71\%\) \\
    \(10^4\) & \(30\) & \(10^{-4}\) & \(178\) & \(553\) & \(3.10\) & \(30.5\%\) \\
    \(10^4\) & \(50\) & \(10^{-4}\) & \(316\) & \(922\) & \(2.92\) & \(36.55\%\) \\
    \bottomrule
    \end{tabular}
\end{table}

Table~\ref{tab:runtime-improvement} summarizes how the theoretical ensemble sizes translate into runtime gains. 
For all dimensions, \(\algname\) uses substantially fewer ensemble members than \texttt{LinES}; the baseline ensemble size is about \(3.36\), \(3.24\), $3.12$ and \(2.92\) times larger for \(d=10,20, 30, 50\), respectively. 
Despite this reduction, Figure~\ref{fig:unitball-d} shows that \(\algname\) maintains comparable cumulative regret. 
The smaller ensemble size also leads to lower runtime, with the improvement increasing as the dimension grows: from \(9.65\%\) at \(d=10\) to \(36.55\%\) at \(d=50\). 
These results support the computational benefit of the proposed reduced-ensemble design.

\begin{figure}[h!] 
    \centering
    \begin{subfigure}[c]{\linewidth}
        \includegraphics[width=\linewidth]{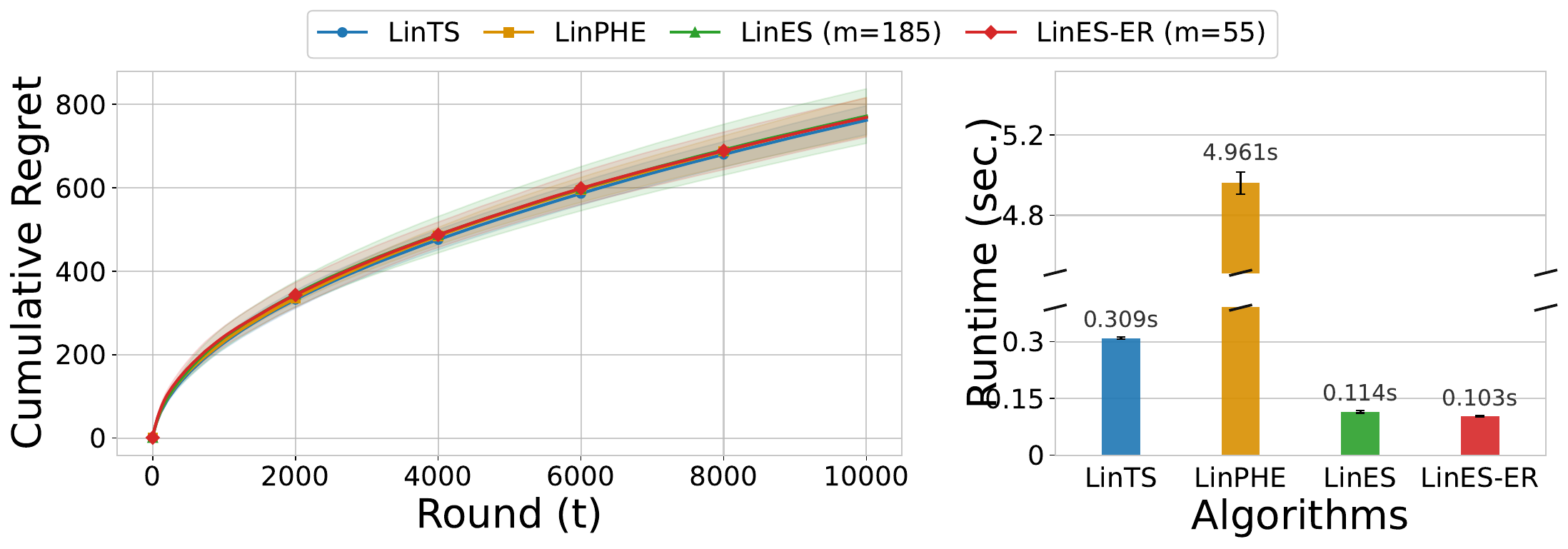}
        \caption{$T=10{,}000$ and $d = 10$}
    \end{subfigure}
    
    \centering
    \begin{subfigure}[c]{\linewidth}
        \includegraphics[width=\linewidth]{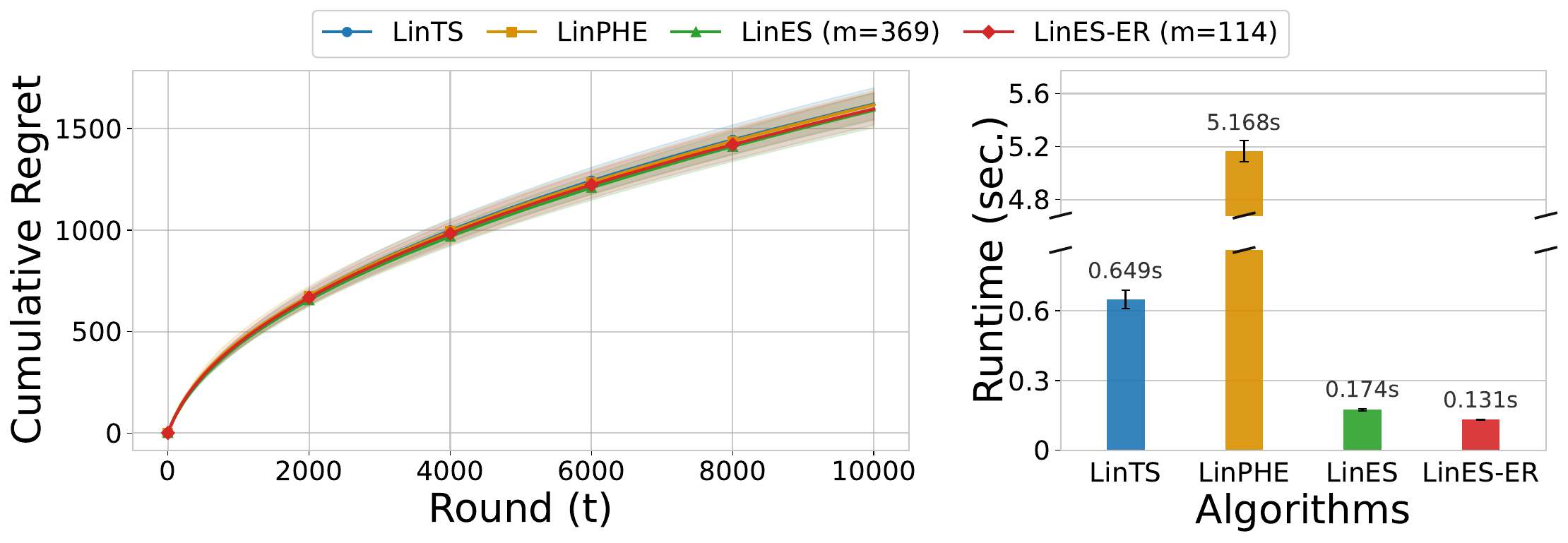}
        \caption{$T=10{,}000$ and $d = 20$}
    \end{subfigure}

    \centering
    \begin{subfigure}[c]{\linewidth}
        \includegraphics[width=\linewidth]{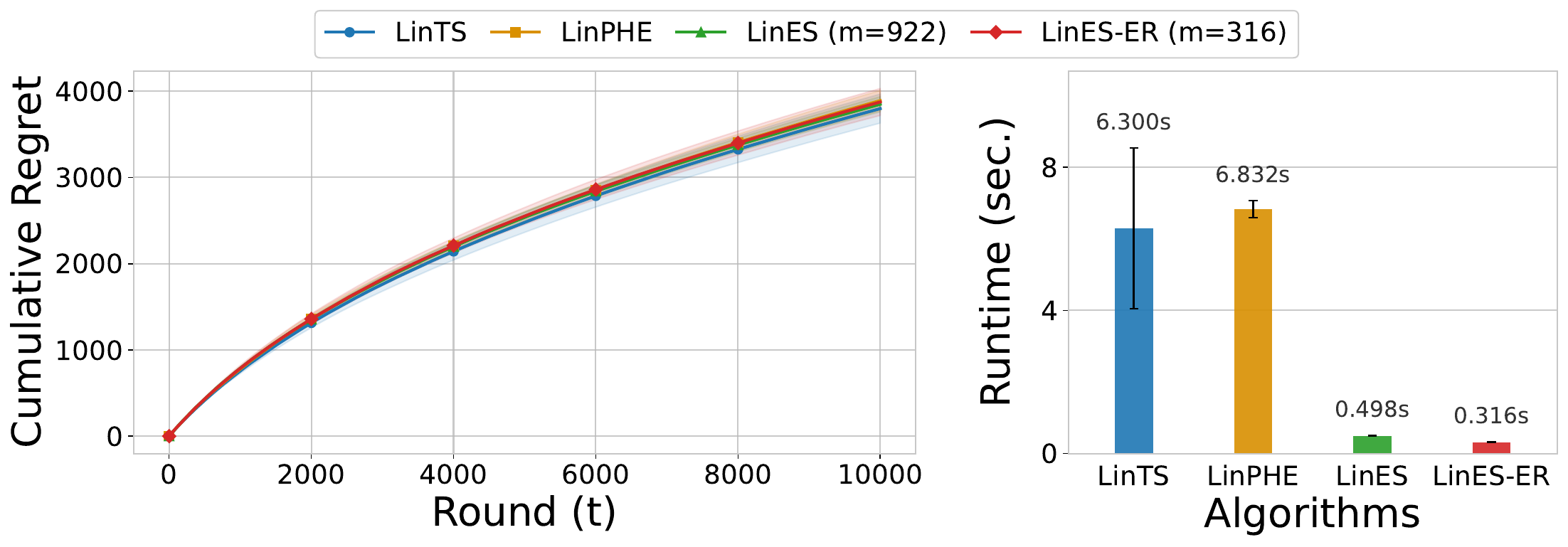}
        \caption{$T=10{,}000$ and $d = 50$}
    \end{subfigure}    

    \caption{\small Cumulative regret and runtime on stochastic linear bandits over \(\Bcal_2^d\)
    with \(T=10{,}000\) and \(d\in\{10,20,50\}\).
    }
    \label{fig:unitball-d}
\end{figure}

\subsection{Anytime Implementation}
\paragraph{Algorithm description.}
We now describe an anytime implementation of $\algname$ for the case where the horizon \(T\) is unknown. 
The main idea is to combine the proposed epoch-refresh mechanism with a doubling schedule~\citep{besson2018doubling}.
Let
$
    H_0,\, H_1,\, H_2,\ldots
$
be a sequence of horizon proxies satisfying \(H_{\ell+1}=2H_\ell\). 
At the beginning of each doubling phase, the algorithm updates the target ensemble size according to the current horizon proxy \(H_\ell\). 
However, unlike standard ensemble-sampling methods~\citep{janz2024ensemble, sun2025provable, janz2026sharp}, this does not require restarting the entire algorithm.
The regularized Gram matrix \(\Vb_t\), the ridge estimator \(\hat{\thetab}_t\), and all previously collected observations are retained. 
Only the ensemble perturbations are reinitialized according to the current Gram matrix.

This is possible because the epoch-refresh construction already represents the ensemble perturbations relative to the current Gram matrix.
Hence, when the ensemble size needs to be increased, the algorithm can simply start a new epoch and initialize the enlarged ensemble using the current Gram matrix \(\Vb_{t-1}\). 
In contrast, prior ensemble-sampling methods without such a refresh mechanism typically need to restart under the doubling trick, since new ensemble members cannot be made consistent with the previous perturbation history without replaying or discarding past data.
The anytime implementation of $\algname$ is described in Algorithm~\ref{alg:anytime-main-algorithm}.

\begin{algorithm}[h!]
    \caption{Anytime Linear Ensemble Sampling with Epoch Refresh}
    \label{alg:anytime-main-algorithm}
    \begin{algorithmic}[1]
    \State \textbf{Inputs:} initial horizon proxy \(H_0\in\NN\), regularization parameter \(\lambda>0\), inflation parameter \(\gamma>0\), confidence radii \(\{\beta_t\}_{t\ge0}\), epoch refresh parameter \(\alpha>0\), ensemble-size rule \(M(\cdot)\)
    \State Set \(H\gets H_0\), \(m\gets M(H)\)
    \State Initialize \(\Vb_0=\lambda\Ib_d\), \(\hat{\thetab}_0=\zero_d\), and set the current epoch start time \(\tau=0\)
    \State Sample \(\gb^j\sim\Ncal(\zero_d,\Ib_d)\) and set $\sbb_0^j=\Vb_0^{1/2}\gb^j$, $\thetab_0^j=\Vb_0^{-1}\sbb_0^j$ for each $j \in [m]$
    \For{\(t=1,2,\ldots\)}
        \If{\(t>H\)}
            \Comment{\textit{Doubling update}}
            \State Set \(H\gets 2H\), \(m\gets M(H)\), and \(\tau\gets t-1\)
            \State For each \(j\in[m]\), sample \(\gb^j\sim\Ncal(\zero_d,\Ib_d)\) and set
            \[
                \sbb_{t-1}^j=\Vb_\tau^{1/2}\gb^j,
                \qquad
                \thetab_{t-1}^j
                =
                \hat{\thetab}_{t-1}
                +
                \gamma\beta_{t-1}\Vb_{t-1}^{-1}\sbb_{t-1}^j
            \]
        \EndIf

        \If{\(\Vb_{t-1}\npreceq (1+\alpha)\Vb_\tau\)}
            \Comment{\textit{Refresh the ensemble perturbations}}
            \State Set \(\tau\gets t-1\)
            \State For each \(j\in[m]\), sample \(\gb^j\sim\Ncal(\zero_d,\Ib_d)\) and set
            \[
                \sbb_{t-1}^j=\Vb_\tau^{1/2}\gb^j,
                \qquad
                \thetab_{t-1}^j
                =
                \hat{\thetab}_{t-1}
                +
                \gamma\beta_{t-1}\Vb_{t-1}^{-1}\sbb_{t-1}^j
            \]
        \EndIf

        \State Observe \(\Xcal_t\), sample \(J_t\sim\unif([m])\), choose
        $
            \xb_t
            \in
            \argmax_{\xb\in\Xcal_t}
            \ips{\xb}{\thetab_{t-1}^{J_t}},
        $
        and observe \(y_t\)

        \State Update
        $
            \Vb_t
            =
            \Vb_{t-1}
            +
            \xb_t\xb_t^\top
        $
        and
        $
            \hat{\thetab}_t
            =
            \Vb_t^{-1}
            \sum_{s=1}^{t}
            \xb_s y_s
        $
        \State For each \(j\in[m]\), sample \(\xi_t^j\sim\Ncal(0,1)\) and update
        \[
            \sbb_t^j
            =
            \sbb_{t-1}^j
            +
            \xb_t\xi_t^j,
            \qquad
            \thetab_t^j
            =
            \hat{\thetab}_t
            +
            \gamma\beta_t
            \Vb_t^{-1}
            \sbb_t^j
        \]
    \EndFor
    \end{algorithmic}
\end{algorithm}

We briefly justify that Algorithm~\ref{alg:anytime-main-algorithm} preserves the
fixed-horizon regret guarantee up to logarithmic factors.  Let
\(H_\ell=2^\ell H_0\), and define
\[
    I_0:=\{1,\ldots,H_0\},\qquad
    I_\ell:=\{H_{\ell-1}+1,\ldots,H_\ell\}\quad(\ell\ge1),
\]
with \(n_\ell:=|I_\ell|\).  Allocate
\[
    \delta_\ell:=\frac{6\delta}{\pi^2(\ell+1)^2},
    \qquad
    \sum_{\ell=0}^\infty \delta_\ell\le\delta .
\]
Fix a phase \(I_\ell\), and condition on the history before this phase. Then
\(\Vb_{a_\ell}\) and \(\hat{\thetab}_{a_\ell}\) are fixed, and the refreshed
perturbations satisfy
\[
    \sbb_{a_\ell}^j=\Vb_{a_\ell}^{1/2}\gb_\ell^j,
    \qquad
    \gb_\ell^j\sim\Ncal(\zero_d,\Ib_d).
\]
Since the Gram matrix and ridge estimator retain all past data, the global
least-squares confidence event remains valid. Conditioning on the phase-start
history merely fixes the initial Gram matrix for the phase-local perturbation
analysis. Hence, the fixed-horizon perturbation and optimism arguments apply
after a time shift, with \(\Vb_{a_\ell}\) replacing \(\lambda\Ib_d\).

Since \(\Vb_{a_\ell}\succeq \lambda\Ib_d\), the number of refresh epochs inside
\(I_\ell\) is bounded by
\[
    1+
    \frac{
        d\log\!\left(1+n_\ell/(\lambda d)\right)
    }{
        \log(1+\alpha)
    } .
\]
Therefore, choosing \(m_\ell=M(H_\ell)\) according to the fixed-horizon rule
with confidence level \(\delta_\ell\), the proof of
Theorem~\ref{thm:regret bound for infinite arm set} gives, with probability at
least \(1-\delta_\ell\),
\[
    \regret_{I_\ell}
    \le
    \widetilde \Ocal \!\left(d^{3/2}\sqrt{n_\ell}\right).
\]
The bound also holds for every prefix of \(I_\ell\).

By a union bound, all phase-wise bounds hold simultaneously with probability at
least \(1-\delta\).  For any \(T\ge1\), let
\(L(T):=\min\{\ell:T\le H_\ell\}\).  Then
\[
    \regret_T
    \le
    \sum_{\ell=0}^{L(T)}
    \regret_{I_\ell\cap[T]}
    \le
    \widetilde \Ocal \!\left(
        d^{3/2}
        \sum_{\ell=0}^{L(T)}\sqrt{n_\ell}
    \right).
\]
Since \(n_0=H_0\) and \(n_\ell=H_{\ell-1}\) for \(\ell\ge1\),
\[
    \sum_{\ell=0}^{L(T)}\sqrt{n_\ell}
    =
    \Ocal (\sqrt T)
\]
for fixed \(H_0\).  Hence
\[
    \regret_T
    =
    \widetilde \Ocal \!\left(d^{3/2}\sqrt T\right)
\]
simultaneously for all \(T\ge1\), up to the standard logarithmic factors from
the doubling schedule~\citep{besson2018doubling}.

\paragraph{Experiments.}
\begin{figure}[h!] 
    \centering
    \begin{subfigure}[c]{\linewidth}
        \centering
        \includegraphics[width=0.9\linewidth]{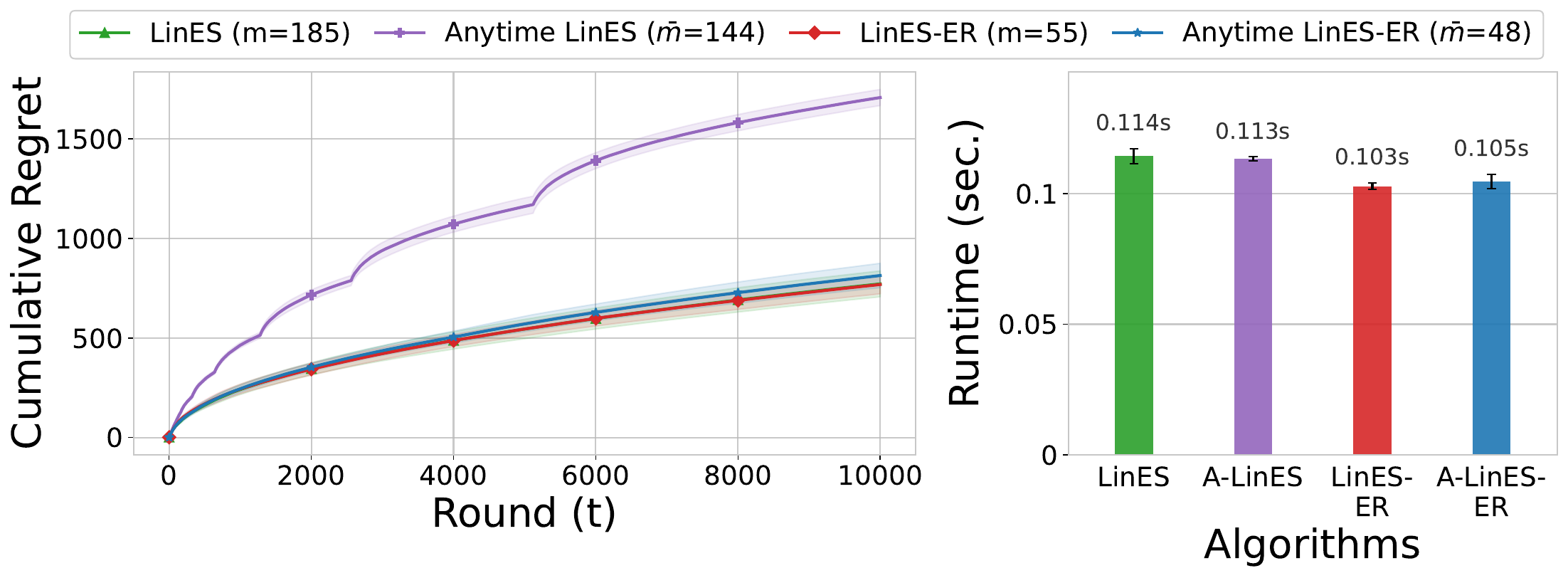}
        \caption{$T=10{,}000$ and $d = 10$}
    \end{subfigure}
    
    \centering
    \begin{subfigure}[c]{\linewidth}
        \centering
        \includegraphics[width=0.9\linewidth]{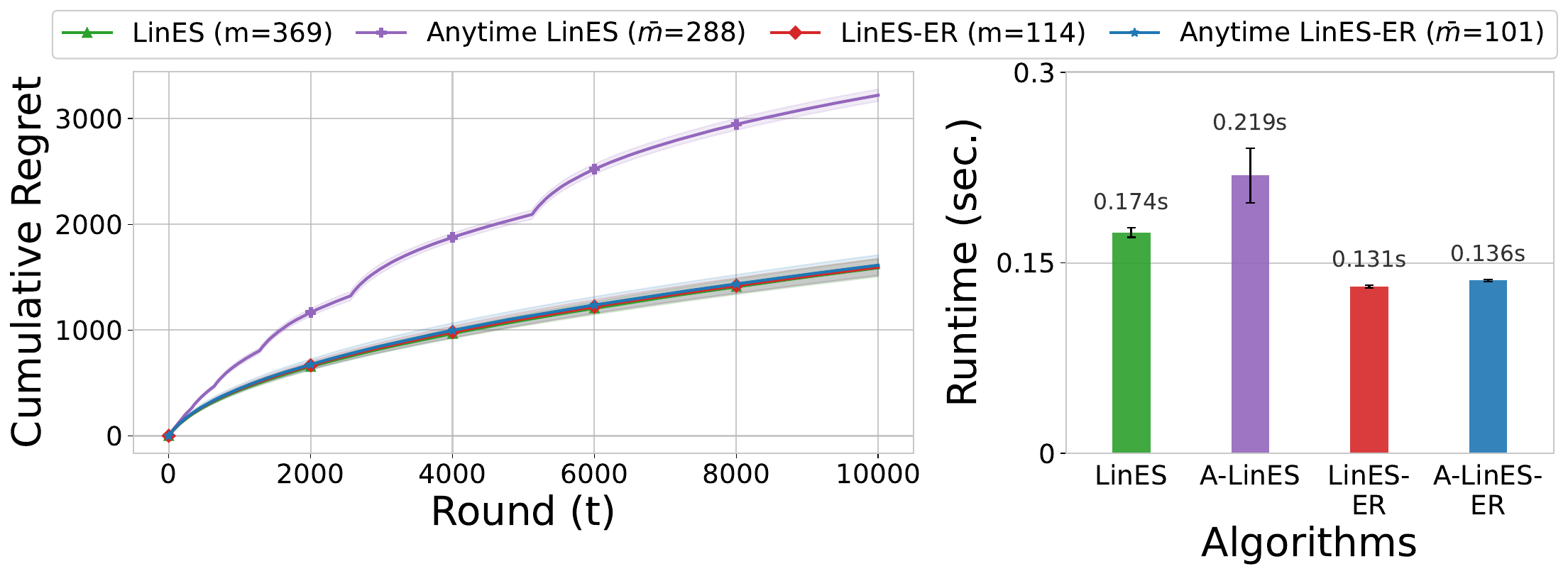}
        \caption{$T=10{,}000$ and $d = 20$}
    \end{subfigure}

    \centering
    \begin{subfigure}[c]{0.9\linewidth}
        \centering
        \includegraphics[width=\linewidth]{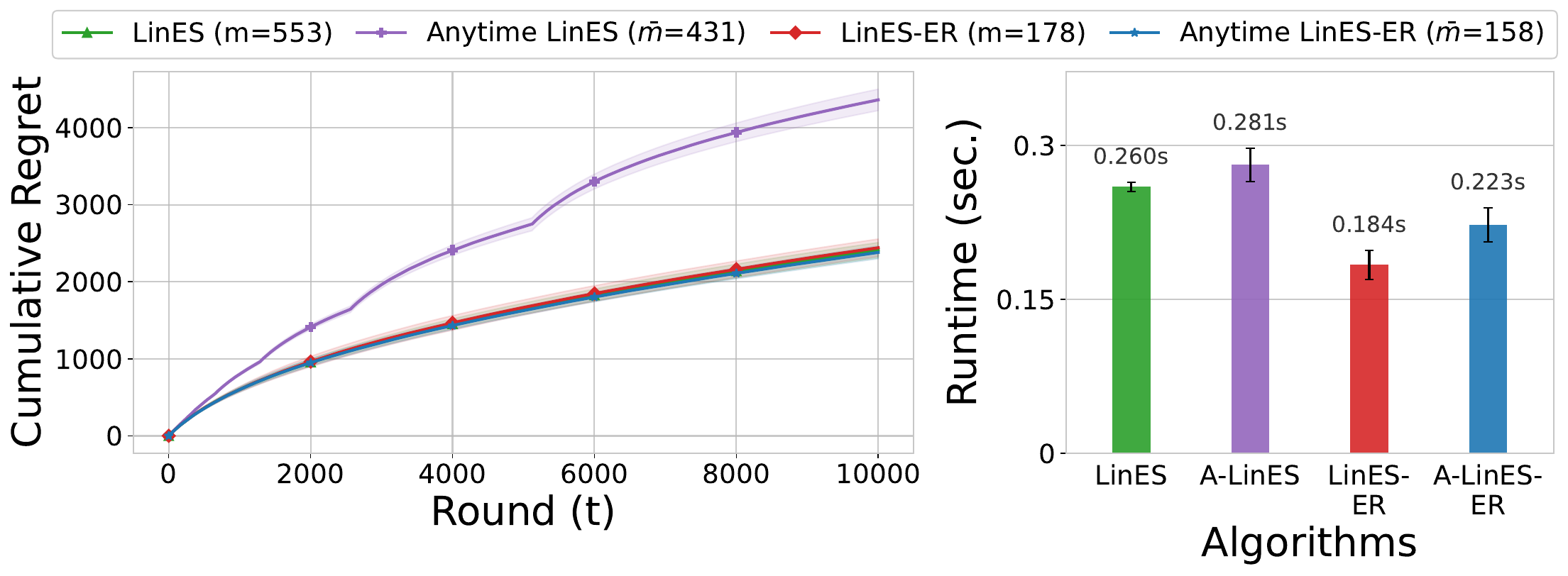}
        \caption{$T=10{,}000$ and $d = 30$}
    \end{subfigure}    

    \centering
    \begin{subfigure}[c]{0.9\linewidth}
        \centering
        \includegraphics[width=\linewidth]{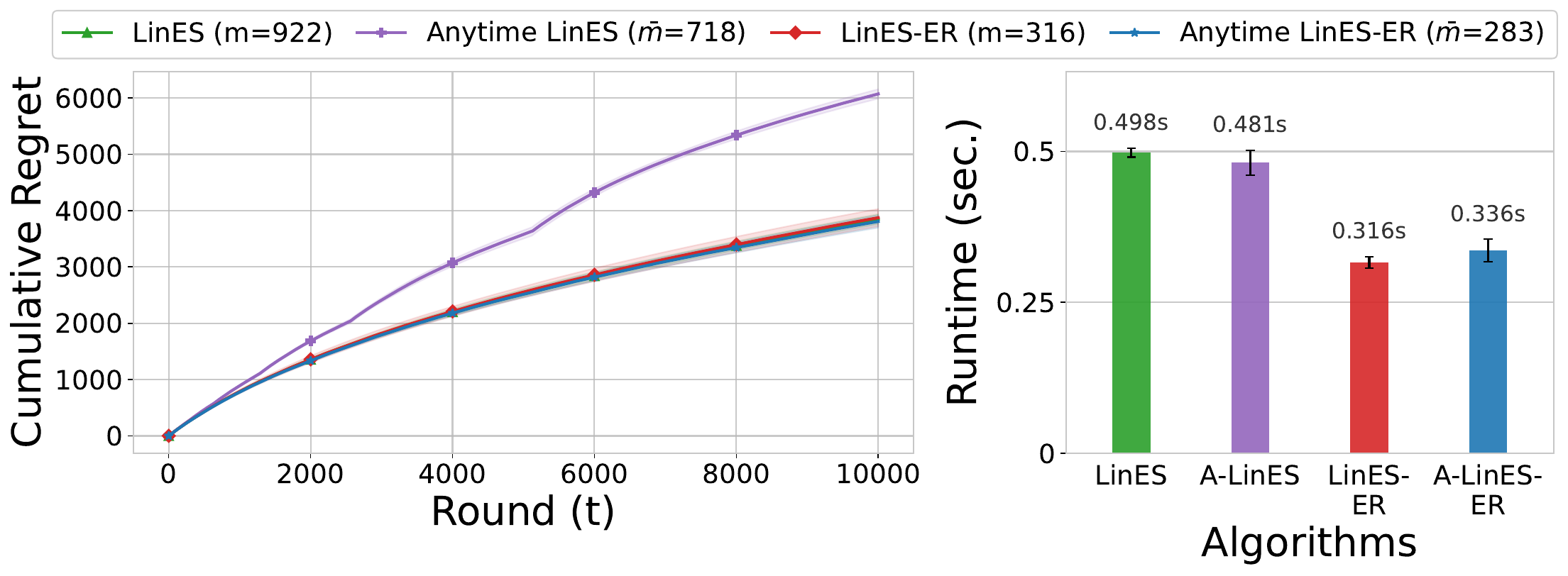}
        \caption{$T=10{,}000$ and $d = 50$}
    \end{subfigure}

    \caption{\small
    Anytime implementation under the doubling schedule for \(T=10{,}000\) and
\(d\in\{10,20,30,50\}\).
Left panels: cumulative regret. Right panels: total runtime.
    }
    \label{fig:unitball-anytime}
\end{figure}

We further evaluate the anytime variants under the doubling schedule.
In the anytime version of \texttt{LinES}, the algorithm is restarted at each doubling epoch and the ensemble size is updated according to the new horizon proxy.
In contrast, the anytime version of \(\algname\) (Algorithm~\ref{alg:anytime-main-algorithm}) retains all past observations and only updates its ensemble representation through the epoch-refresh mechanism.

Figure~\ref{fig:unitball-anytime} shows the cumulative regret and runtime for \(T=10{,}000\) and \(d\in\{10,20,30,50\}\).
Across all dimensions, \texttt{Anytime LinES} suffers a substantial increase in regret due to repeated restarts.
By contrast, \texttt{Anytime LinES-ER} remains close to the fixed-horizon \texttt{LinES-ER} baseline, demonstrating that our no-reset implementation preserves the statistical efficiency of the fixed-horizon algorithm.
Interestingly, \texttt{Anytime LinES-ER} uses fewer ensemble members on average than the fixed-horizon \texttt{LinES-ER}, since its ensemble size is chosen according to the local doubling horizon rather than the final horizon.
Its runtime can nevertheless be slightly larger, because the ensemble is refreshed at every doubling epoch in addition to the usual Gram-matrix-triggered refreshes.
Thus, the anytime version trades a small amount of additional refresh cost for the ability to retain all past observations and avoid full restarts.
Here, \(\bar m\) denotes the average ensemble size used across doubling epochs.

\subsection{Finite-Arm Experiments}
We also evaluate the algorithms in a finite-arm setting, where each instance is generated by sampling \(K\) arms uniformly from the unit sphere in \(\mathbb R^d\).
Figure~\ref{fig:finite-arm} reports cumulative regret for several values of \(K\). Across all tested settings, \texttt{LinES-ER} remains competitive with \texttt{LinES}, \texttt{LinTS}, and \texttt{LinPHE}. This indicates that the proposed epoch-refresh mechanism does not degrade empirical performance in finite-arm problems, despite using a smaller ensemble than \texttt{LinES}. The results are consistent with our finite-arm theory, which shows that arm-wise concentration can yield the sharper \(\widetilde{\Ocal}(d\sqrt{T\log K})\) regret scaling.

\begin{figure}[h!] 
    \centering
    \begin{subfigure}[c]{\linewidth}
        \centering
        \includegraphics[width=\linewidth]{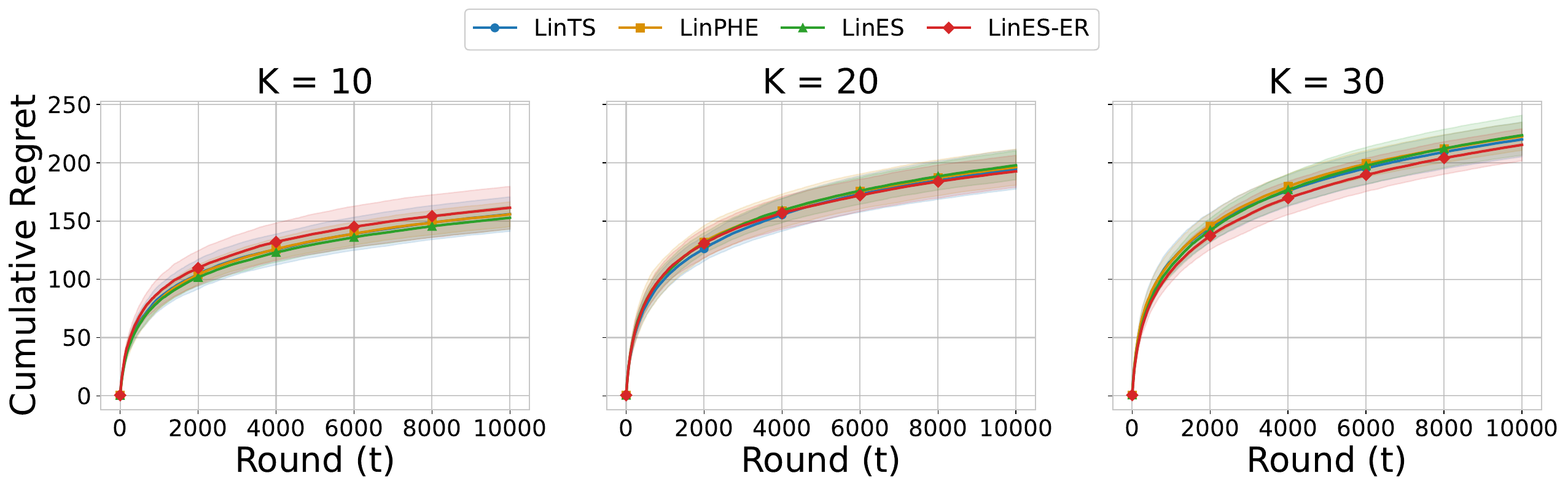}
        \caption{$T=10{,}000$ and $d = 10$}
    \end{subfigure}
    
    \centering
    \begin{subfigure}[c]{\linewidth}
        \centering
        \includegraphics[width=\linewidth]{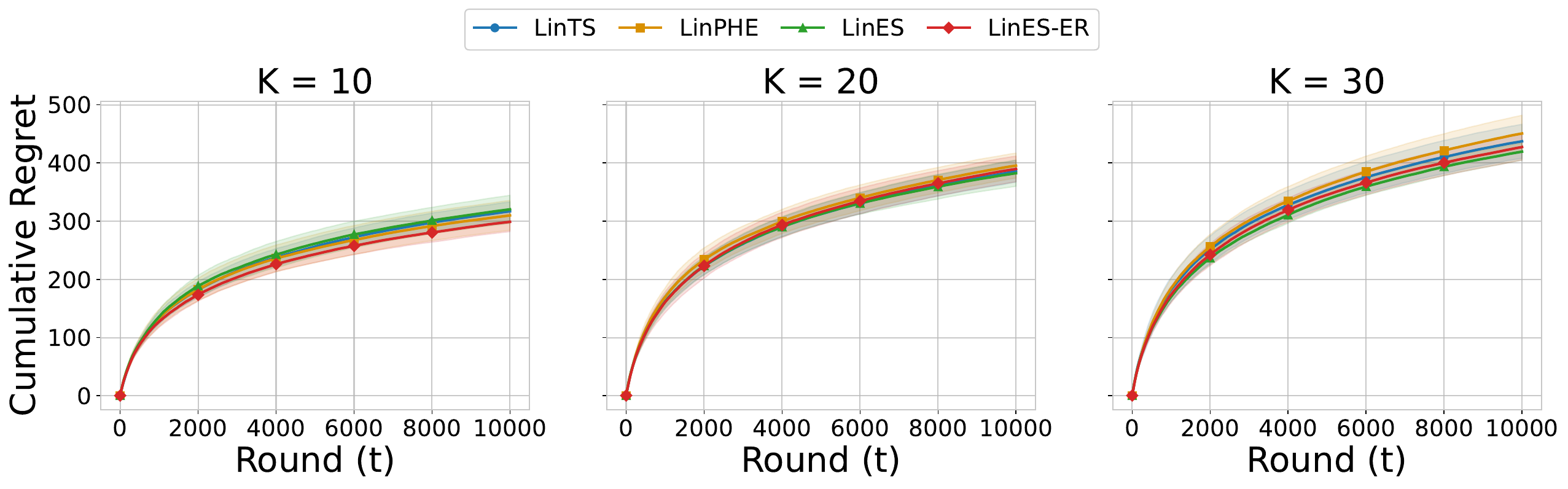}
        \caption{$T=10{,}000$ and $d = 20$}
    \end{subfigure}

    \centering
    \begin{subfigure}[c]{\linewidth}
        \centering
        \includegraphics[width=\linewidth]{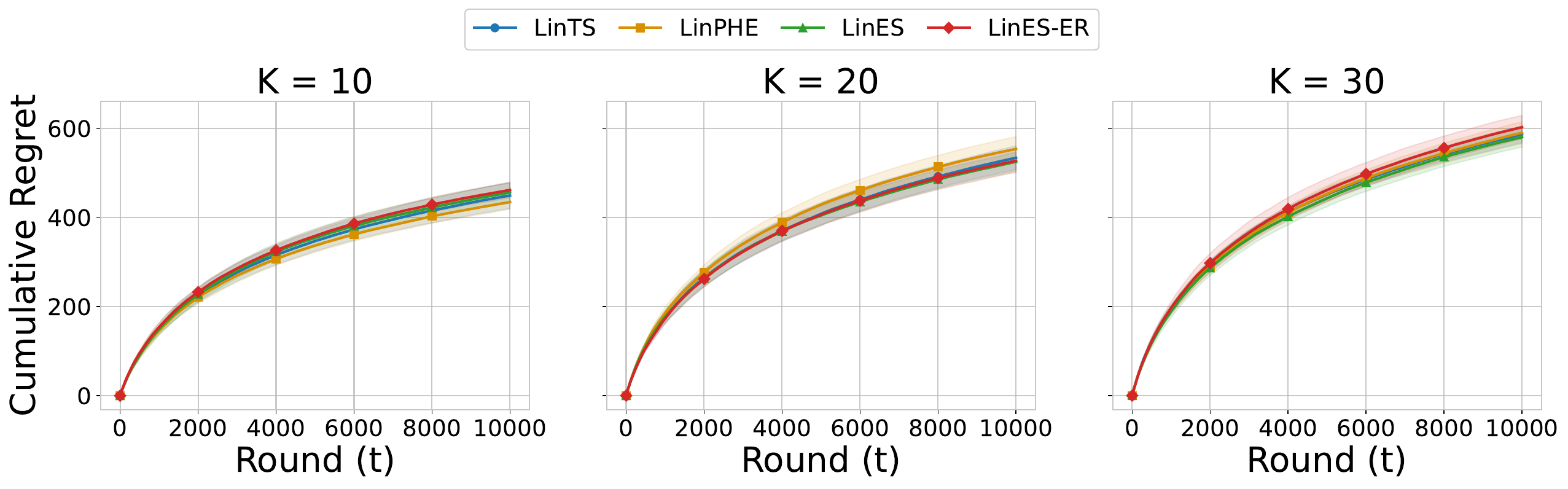}
        \caption{$T=10{,}000$ and $d = 30$}
    \end{subfigure}    

    \caption{\small 
    Finite-arm experiments with randomly generated arms on the unit sphere.
    }
    \label{fig:finite-arm}
\end{figure}

\end{document}